\documentclass{new_tlp}
\usepackage{mathptmx}

\DeclareMathAlphabet{\mathcal}{OMS}{cmsy}{m}{n}

\RequirePackage[latin1]{inputenc}
\usepackage[english]{babel}

\usepackage{graphicx}
\usepackage{latexsym}
\usepackage{amssymb}
\usepackage{amsmath}
\usepackage{enumitem}
\usepackage[colorinlistoftodos, textwidth=3cm]{todonotes} 
\usepackage[spaces,hyphens]{url}
\usepackage{xspace}

\usepackage{breakurl}

\newtheorem{definition} {Definition}
\newtheorem{lemma}      {Lemma}
\newtheorem{proposition}{Proposition}
\newtheorem{corollary}  {Corollary}
\newtheorem{theorem}    {Theorem}
\newtheorem{example}{Example}

\newtheorem{manualtheoreminner}{Theorem}
\newenvironment{manualtheorem}[1]{%
  \renewcommand\themanualtheoreminner{#1}%
  \manualtheoreminner
}{\endmanualtheoreminner}

\newtheorem{manualpropinner}{Proposition}
\newenvironment{manualproposition}[1]{%
  \renewcommand\themanualpropinner{#1}%
  \manualpropinner
}{\endmanualpropinner}

\newtheorem{manuallemmainner}{Lemma}
\newenvironment{manuallemma}[1]{%
  \renewcommand\themanuallemmainner{#1}%
  \manuallemmainner
}{\endmanuallemmainner}

\newtheorem{mcorollaryinner}{Corollary}
\newenvironment{mcorollary}[1]{%
  \renewcommand\themcorollaryinner{#1}%
  \mcorollaryinner
}{\endmcorollaryinner}

\title[Reasoning on $\DLliteR$ with Defeasibility in ASP]{Reasoning on $\DLliteR$ with Defeasibility in ASP}

  \author[L. Bozzato, T. Eiter, L. Serafini]
         {LORIS BOZZATO\\
         Fondazione Bruno Kessler, Via Sommarive 18, 38123 Trento, Italy\\
         \email{bozzato@fbk.eu}
				 \and 
				 THOMAS EITER\\
				 Technische Universit\"{a}t Wien, Favoritenstra\ss e 9-11, A-1040 Vienna, Austria\\
				 \email{eiter@kr.tuwien.ac.at}
				 \and	
				 LUCIANO SERAFINI\\
         Fondazione Bruno Kessler, Via Sommarive 18, 38123 Trento, Italy\\
         \email{serafini@fbk.eu}}

\jdate{July 2021}
\pubyear{2021}
\pagerange{\pageref{firstpage}--\pageref{lastpage}}

\def\isa{\sqsubseteq}

\def\I{\mathcal{I}}

\newcommand{\Pair}[2]{\left\langle#1,#2\right\rangle}

\newcommand{\vc}[1]{\mathbf{#1}}

\newcommand{\dd}{{\vc{d}}}
\newcommand{\ee}{{\vc{e}}}

\newcommand{\SROIQ}{\mathcal{SROIQ}}
\newcommand{\ALC}{\mathcal{ALC}}

\newcommand{\NI}{\mathrm{NI}}
\newcommand{\NR}{\mathrm{NR}}
\newcommand{\NC}{\mathrm{NC}}

\newcommand{\IC}{\mathfrak{I}}

\newcommand{\K}{\mathcal{K}}

\newcommand{\dlmodels}{\mathop{\models_{\mathrm{DL}}}}
\newcommand{\T}{\mathcal{T}}

\newcommand{\stru}[1]{\langle #1 \rangle}

\newcommand{\non}{\neg}
\newcommand{\im}{\!\rightarrow\!}

\newcommand{\subs}{\sqsubseteq}

\newcommand{\Acal}{{\cal A}}

\newcommand{\Ecal}{{\cal E}}

\newcommand{\Kcal}{{\cal K}}
\newcommand{\Ical}{{\cal I}}
\newcommand{\Lcal}{{\cal L}}

\newcommand{\Ncal}{{\cal N}}

\newcommand{\Rcal}{{\cal R}}
\newcommand{\Tcal}{{\cal T}}

\newcommand{\cov}[1]{\preceq}

\newcommand{\CAS}{\mathit{CAS}}

\newcommand{\OVR}{\mathit{OVR}}
\newcommand{\casmap}{\chi}
\newcommand{\mi}[1]{\mathit{#1}}

\newcommand{\nop}[1]{}

\newcommand{\SROIQrl}{\mathcal{SROIQ}\text{-RL}}

\newcommand{\NIs}{\NI_\Sigma}

\newcommand{\KB}{\mathrm{K}} 

\newcommand{\mlc}{\ml{c}}

\newcommand{\DC}{\mathbf{C}}
\newcommand{\DP}{\mathbf{P}}
\newcommand{\DV}{\mathbf{V}}

\newcommand{\subClass}{{\tt subClass}}

\newcommand{\subEx}{{\tt subEx}}
\newcommand{\subRole}{{\tt subRole}}

\newcommand{\supEx}{{\tt supEx}}

\newcommand{\pDis}{{\tt dis}}
\newcommand{\pInv}{{\tt inv}}
\newcommand{\pIrr}{{\tt irr}}

\newcommand{\nom}{{\tt nom}}
\newcommand{\cls}{{\tt cls}}
\newcommand{\rol}{{\tt rol}}
\newcommand{\const}{{\tt const}}

\newcommand{\triple}{{\tt triple}}

\newcommand{\supNot}{{\tt supNot}}

\newcommand{\eval}{\textsl{eval}}

\newcommand{\ml}[1]{\mathsf{#1}} 

\newcommand{\default}{{\mathrm D}}

\newcommand{\ovr}{{\tt ovr}}
\newcommand{\naf}{\mathop{\tt not}}
\newcommand{\grd}{\mathit{grnd}}
\newcommand{\Head}{\mathit{Head}}
\newcommand{\Body}{\mathit{Body}}

\newcommand{\rif}{\leftarrow}

\newcommand{\insta}{{\tt insta}}
\newcommand{\instd}{{\tt instd}}
\newcommand{\triplea}{{\tt triplea}}
\newcommand{\tripled}{{\tt tripled}}

\newcommand{\AllNRel}{{\tt all\_nrel}}
\newcommand{\AllNRelStep}{{\tt all\_nrel\_step}}
\newcommand{\first}{{\tt first}}
\newcommand{\nextp}{{\tt next}}
\newcommand{\lastp}{{\tt last}}

\newcommand{\test}{{\tt test}}

\newcommand{\defsubs}{{\tt def\_subclass}}

\newcommand{\defsubr}{{\tt def\_subr}}

\newcommand{\definv}{{\tt def\_inv}}
\newcommand{\defirr}{{\tt def\_irr}}

\newcommand{\DKB}{\mathcal{K}}

\newcommand{\ptime}{\ensuremath{\mathrm{PTime}}}
\newcommand{\np}{\ensuremath{\mathrm{NP}}}
\newcommand{\conp}{\ensuremath{\mathrm{co\mbox{-}NP}}}
\newcommand{\nlogspace}{\ensuremath{\mathrm{NLogSpace}}}

\def\qed{$\Box$}
\def\endproof{\ifhmode\nobreak\qed\par\fi\medskip}

\renewcommand{\IC}{\Ical}

\setitemize{label=--,leftmargin=*}
\setenumerate{leftmargin=*}
\newcommand{\EndEx}{\mbox{}~\hfill$\Diamond$} 

\newcommand{\comment}[1]{{#1}}
\newcommand{\lbnote}[1]{}
\newcommand{\rewnote}[2]{}

\newcommand{\CKRew}{CKR\emph{ew}\xspace}
\newcommand{\DLliteR}{\textsl{DL-Lite}_\Rcal}
\newcommand{\DLlite}{\textsl{DL-Lite}}
\newcommand{\ELbot}{\Ecal\Lcal_\bot}

\newcommand{\J}{{\cal J}}

\begin{document}

\label{firstpage}

\maketitle

  \begin{abstract}
  Reasoning on defeasible knowledge
  is a topic of interest in the area of description logics, 
  as it is related to the need of representing exceptional instances in knowledge bases.
  %
  In this direction, in our previous works we presented
	a framework for representing
  (contextualized) OWL RL knowledge bases with a notion of
  justified exceptions on defeasible axioms: 
	reasoning in such framework is realized by a translation into ASP programs.
	%
  The resulting reasoning process for OWL RL, however,
  introduces a complex encoding 
  in order to capture reasoning on the negative information
	needed for reasoning on exceptions.
	%
  In this paper, 
  we apply the justified exception approach to knowledge bases in $\DLliteR$, i.e., the language underlying OWL QL.
	We provide a definition for $\DLliteR$ knowledge bases with defeasible axioms and
	study their semantic and computational properties. 
	In particular, we study the effects of exceptions 
	over unnamed individuals.
	The limited form of $\DLliteR$ axioms allows us to formulate a simpler 
	ASP encoding,
	where reasoning on negative information is managed by direct rules.
	The resulting materialization method gives rise to a complete
	reasoning procedure for instance checking in $\DLliteR$ 
	with defeasible axioms.%

  Under consideration in Theory and Practice of Logic Programming (TPLP).%
	\footnote{This paper is an extended and revised version of a conference paper 
	appearing in the proceedings of the \emph{3rd International Joint Conference on Rules and Reasoning 
	(RuleML+RR 2019)}~\cite{DBLP:conf/ruleml/BozzatoES19}.}
  \end{abstract}




\section{Introduction}
\label{sec:intro}

Representing defeasible information
is a topic of interest in the area of description logics (DLs), 
as it is related to the need of accommodating the presence of 
exceptional instances in knowledge bases.
This interest led to different proposals 
for 
non-monotonic features in DLs
based on different 
notions of defeasibility, e.g.,~\cite{BonattiFPS:15,BonattiLW:06,DBLP:conf/jelia/BritzV16,DBLP:conf/jelia/CasiniS10,GiordanoGOP:13,DBLP:journals/ijar/PenselT18}.
%
In this direction, we presented in~\cite{BozzatoES:18} 
an approach to represent
defeasible information in
contextualized DL knowledge bases
by introducing a notion of \emph{justifiable exceptions}
that has been
inspired by ideas in \cite{BuccafurriFL:99}:
general \emph{defeasible axioms} can be overridden
by more specific exceptional instances if 
their application would provably lead to inconsistency.
%
For example\footnote{see~\cite[Example 2]{BozzatoES:18}.}, 
we can express
that \emph{``in general, concerts are expensive''} as a defeasible concept inclusion
$\mi{Concert} \isa \mi{Expensive}$.
However, for a specific instance $\mi{free\_concert}$ of $\mi{Concert}$ representing a free concert 
we may want to ``override'' the defeasible axiom (i.e., disregard its application): 
in our approach this is possible, provided that one can prove a set of assertions
$\{\mi{Concert}(\mi{free\_concert}), \non\mi{Expensive}(\mi{free\_concert})\}$
that justify the exception for this individual.


In our seminal paper~\cite{BozzatoES:18},
we concentrated on reasoning over $\SROIQrl$ based knowledge bases:
$\SROIQrl$ corresponds to the language of 
the OWL RL profile of the \emph{Web Ontology Language (OWL)}~\cite{Motik:09:OWO}
and allows for tractable reasoning.
In particular, this language emerges from the $\SROIQ$ language~\cite{horrocks:2006} 
and dl-programs~\cite{Grosof2003}. 
Remarkably, $\SROIQrl$ can be seen as an intersection of 
DLs and Horn logic programs.

In~\cite{BozzatoES:18} reasoning in $\SROIQrl$ 
knowledge bases is realized by
a translation to datalog (under answer sets semantics), 
which provides a complete \emph{materialization calculus} in the style of~\cite{Krotzsch:10}
for instance checking and conjunctive query (CQ) answering.
%
While the translation 
covers the full $\SROIQrl$ language,
it needs a complex encoding 
to represent reasoning on exceptions.
%
In particular, 
it relies on
proofs by contradiction 
to ensure completeness in presence of negative disjunctive
information.
In fact, negative disjunctive information is not easily
expressible in datalog: for example, from $A \sqcap B \isa C$ and $\non C(a)$
we can derive $(\non A \sqcup \non B )(a)$, which is not directly representable
by datalog rules. Also, a naive use of disjunction in rule heads does
not overcome this problem. For this reason, in~\cite{BozzatoES:18} inference on negative literals
is obtained as an encoding of a ``test'' for contradiction of such literals in the 
deduction rules of the datalog translation.

In this paper, we consider the case of knowledge bases with defeasible axioms
in the description logic $\DLliteR$~\cite{CalvaneseGLLR07}, which 
corresponds to the language underlying the OWL QL
Profile~\cite{Motik:09:OWO}.
%
As in the case of $\SROIQrl$,
also $\DLliteR$ is a Horn logic
and thus can be related to
logic programs. In fact, $\DLliteR$ is a class of existential rules
and falls 
then into the linear fragment~\cite{DBLP:journals/ws/CaliGL12}. 

It is indeed interesting to show the applicability of our defeasible
reasoning approach to the well-known $\DLlite$ family: by adopting $\DLliteR$
as the base logic, we need to take unnamed individuals introduced by
existential quantifiers into account, especially for the justifications
of exceptions. 
Defeasible axioms like $\default(\mi{Concert}
\isa \exists \mi{hasOrganizer})$, which says that concerts have
some organizer, allow for a smooth handling. 
On the other hand, if we have axioms like 
$\exists \mi{hasOrganizer}^- \isa \mi{Organizer}$, which informally
assigns a type, and $\default(\mi{Organizer}\isa \mi{Company})$,
then overridings can happen over unnamed individuals relative to this axiom.
The problem for reasoning with such 
unnamed elements is that they can have different
interpretations in different models of the knowledge base,
while, for determining the applicability of a defeasible axiom
or its overriding, we need to identify the exceptional domain elements.

Moreover,
with respect to the translation to datalog,
we show that due to the restricted form of
its axioms, the $\DLliteR$ language allows us to give a less involved
datalog encoding in which reasoning on negative information
is directly encoded 
in datalog rules; for more background, we refer to the discussion on ``justification safeness'' in~\cite{BozzatoES:18}.

\rewnote{1}{DLs changed from previous work: one motivation is that  translation to Datalog/ASP and proofs are involved. Is it only the reason to study the new DLs? In the studied DL we can encode same problems/properties as in the previous?}
\rewnote{3}{The authors have already presented a similar approach for SROIQ-RL [Bozzato et al. 2018] and $EL_{\bot}$ [Bozzato et al. 2019b]. Apart from some discussion at the beginning of sect. 5, p. 14, it is not very clear what problems they had to face for this re-formulation for DL-lite, and how the proposal differs from their previous proposals.}
\comment{The choice of studying the application of our methods to $\DLliteR$ knowledge bases
is indeed motivated by the interest in covering the OWL QL fragment of
OWL 2, which is relevant from an application perspective.
More importantly, from a formal perspective, $\DLliteR$ allows for the use
of unnamed individuals in inverse roles which are not available in $\ELbot$~\cite{DBLP:conf/birthday/BozzatoES19} and $\SROIQrl$~\cite{BozzatoES:18}: thus, the techniques for managing unnamed individuals, 
especially in exceptions, need to be adapted to the expressivity of $\DLliteR$.
Another reason of our interest in $\DLliteR$ stands in the fact that
it is a standard DL which falls in the fragment where no reasoning on
disjunctive negative information is needed for deductions on exceptions:
this shows a notable example of ``justification safe'' language which, as noted above,
allows us to formulate a simpler version of the datalog encoding for instance checking.}

\smallskip\noindent
The contributions of this paper can be summarized as follows:
\begin{itemize}
\item
  In Section~\ref{sec:dkb} we provide a definition of defeasible DL
  knowledge base (DKB) with justified models
  that draws from the definition of 
	\emph{Contextualized Knowledge Repositories (CKR)}
	\cite{BozHomSer:DL2012,BozzatoSerafini:13,serafini-homola-ckr-jws-2012}
	with defeasible axioms provided in~\cite{BozzatoES:18}. 
	This allows us to concentrate on the defeasible reasoning aspects 
	without considering the aspects related to context representation.
	In the case of $\DLliteR$, 
	we consider the effects of reasoning with unnamed individuals
	and of their admission in the exceptions of defeasible axioms.
  In particular, we consider models
  in which exceptions can only occur on individuals named in the DKB (called \emph{exception-safety}).
\item
  In Section~\ref{sec:properties}, 
	we study the semantic properties of DKB-models. 
	In particular, in the case of exception safe DKBs, we show that their models
	preserve conditions from~\cite{BozzatoES:18} that allow us to concentrate on 
	minimal models that are restricted to the individual names occurring in the 
	knowledge base. These properties are important to verify
	the feasibility of the reasoning method based on the datalog translation
	that we provide in the later sections.
\item 
For exception-safe DKBs based on $\DLliteR$,
we  provide in Section~\ref{sec:translation} a 
translation to datalog (under answer set semantics~\cite{gelf-lifs-91})
that alters the translation in~\cite{BozzatoES:14,BozzatoES:18} 
and prove its correctness
for instance checking.
	Notably, the fact that reasoning on 
	negative disjunctive information is not needed
	allows us to provide a simpler 
	translation without 
	the involving ``test'' environments 
	mechanism of~\cite{BozzatoES:18}.
The datalog translation for $\DLliteR$ DKBs is included in 
the latest version of the \emph{\CKRew (CKR datalog rewriter)} 
prototype~\cite{BozzatoES:18}, which is available online.%
\footnote{\url{http://ckrew.fbk.eu/}.}	
\item
	In Section~\ref{sec:complexity}  
	we provide complexity results for 
	reasoning problems on exception-safe  DKBs based on $\DLliteR$.
        Deciding satisfiability of such a DKB with respect to justified models is
        tractable, while inference of an axiom under cautious (i.e.,
        certainty) semantics is \conp-complete in general.
	Moreover, 
	CQ answering is shown to be $\Pi^p_2$-complete.
				
\item
  In Section~\ref{sec:unnamed} we discuss how reasoning on
  unnamed exceptional instances affects the complexity, and in
  particular how the notion of exception-safety can be generalized in
  a way such that the techniques developed and the results obtained
  can be lifted to this setting. We present for this the class of
  $n$-derivation exception (de) safe programs, which for a few ($n$ bounded by a
  constant) unnamed individuals in exceptions
   stays within the same complexity and for polynomially
   many faces an increase by at most one level in the polynomial hierarchy.
   Furthermore, we discuss how it is possible to extend
   the current datalog translation to manage unnamed elements under de-safety.
\end{itemize}
%
%
\rewnote{1}{This is an extension of the (best) RuleML+RR 2019 paper. Relation to RuleML+RR 2019 paper is quite clear via some comments that are given about the improvements of the current version; however, a more detailed analyze should be added in order to fully understand/appreciate the novel contributions.}
\comment{With respect to the initial conference paper
presented at \emph{RuleML+RR 2019}~\cite{DBLP:conf/ruleml/BozzatoES19}, 
this version of the paper extends the work by
including a study of properties of DKB-models and justifications
over unnamed individuals in Section~\ref{sec:dkb} and Section~\ref{sec:unnamed}, 
where the notion of $n$-de safe DKBs is introduced and 
the extension of results to this setting is discussed.
The current paper also includes a more detailed 
study for semantic properties of DKB-models (Section~\ref{sec:properties})
and complexity of reasoning problems (Section~\ref{sec:complexity}).
With respect to the ASP translation, rules have been slightly simplified;
moreover, the current translation has been implemented in the \CKRew prototype.
Finally, with respect to the conference paper, we provide further details 
and comparisons on related work in Section~\ref{sec:related}.
%
To increase readability and the comprehension of the contributions, 
additional details and 
proofs of the results are reported in the Appendix.}
%


\section{Preliminaries}
\label{sec:prelims}

\noindent
\textbf{Description Logics and $\DLliteR$ language.}
We assume the common definitions of description logics~\cite{dlhb}
and the definition of the logic $\DLliteR$~\cite{CalvaneseGLLR07}:
we summarize in the following the basic definitions used in this work.
For ease of reference, we present in Table~\ref{tab:dl-lite-operators} in the Appendix 
the details of syntax and semantics of $\DLliteR$.

A \emph{DL vocabulary} 
$\Sigma$ consists of the mutually disjoint countably infinite
sets $\NC$ of \emph{atomic concepts},
$\NR$ of \emph{atomic roles}, and 
$\NI$ of \emph{individual constants}.
Intuitively, concepts represent classes of objects
(e.g., $\mi{PhDStudent}$), roles represent binary
relations across objects (e.g., $\mi{hasCourse}$),
and individual names identify specific elements of the domain (e.g., $\mi{bob}$).
%
Complex \emph{concepts} are then recursively defined as the smallest
sets containing all concepts that can be inductively constructed using
the constructors of the considered DL language 
(see, e.g., Table~\ref{tab:dl-lite-operators} for $\DLliteR$).

A $\DLliteR$ \emph{knowledge base} $\K=\stru{\T,\Rcal,\Acal}$ consists of: 
a TBox $\T$ containing \emph{general concept inclusion (GCI)} axioms $C \subs D$ 
where $C,D$ are concepts, of the form:
\[\begin{array}{l}

  C := A \;|\; 
         \exists R
\qquad\qquad
  D := A \;|\; 
         \non C \;|\;
         \exists R
\end{array}\] 
where $A \in \NC$ and $R \in \NR$;\footnote{In the following, 
we will use $C$ to denote a left-side concept and $D$ as a right-side concept.}
an RBox $\Rcal$ containing \emph{role inclusion (RIA)} axioms $S \subs R$, 
reflexivity, irreflexivity, inverse and
role disjointness axioms, where $S,R$ are roles; 
and an ABox $\Acal$ composed of assertions of the forms 
$D(a)$, 
$R(a,b)$, 
with $R \in \NR$ and $a,b \in \NI$.
\begin{example}
\label{ex:prelim-dl}
The TBox $\T$ may include concept inclusion 
expressions such as  $\mi{PhDStudent} \isa \non \exists \mi{hasCourse}$;
the RBox $\Rcal$ 
may contain a role inclusion $\mi{hasAdvisor} \isa \mi{isStudentOf}$;
and finally, the ABox $\Acal$ may contain assertions
$\non \mi{Professor}(\mi{bob}), \mi{hasAdvisor}(\mi{bob}, \mi{alice})$.\footnote{For simplicity, 
in the following examples,
we may represent knowledge bases as set of axioms with implicit separation of TBox, RBox and ABox.}\EndEx
\end{example}
%
A \emph{DL interpretation} is a pair $\I=\stru{\Delta^\I,\cdot^\I}$ where $\Delta^\I$
is a non-empty set called \emph{domain} and $\cdot^\I$ is the \emph{interpretation
function} which assigns denotations for language elements:
$a^\I \in \Delta^\I$, for $a \in \NI$;
$A^\I \subseteq \Delta^\I$, for $A \in \NC$; 
$R^\I \subseteq \Delta^\I\times\Delta^\I$, for $R \in \NR$. 
The interpretation of non-atomic concepts and roles is defined by the evaluation 
of their description logic operators (see Table~\ref{tab:dl-lite-operators} 
and~\cite{CalvaneseGLLR07} for $\DLliteR$).
%
An interpretation $\I$ \emph{satisfies} an axiom 
$\phi$, denoted
$\I\dlmodels\phi$, if it verifies the respective semantic condition, in particular: 
for $\phi = D(a)$, $a^\I \in D^\I$;
for $\phi = R(a,b)$, $\stru{a^\I,b^\I} \in R^\I$;
for $\phi = C \subs D$, $C^\I \subseteq D^\I$ (resp. for RIAs).
$\I$ is a \emph{model} of $\K$, denoted
$\I\dlmodels\K$, if it satisfies all axioms of $\K$.
%
\begin{example}[cont'd]
An interpretation $\I$ satisfies 
$\non \mi{Professor}(\mi{bob})$ if $\mi{bob}^\I \notin \mi{Professor}^\I$,
and $\I$ satisfies $\mi{PhDStudent} \isa \non \exists \mi{hasCourse}$
if, for every element $d$ of $\mi{PhDStudent}^\I$, there does not exist 
some domain element $e$ such that $\stru{d,e} \in \mi{hasCourse}^\I$.\EndEx
\end{example}
Without loss of generality, we adopt the
{\em standard name assumption (SNA)} in the DL context 
(see~\cite{BruijnET:08,DBLP:journals/ai/EiterILST08} 
for more details).
That is, we assume an infinite subset
$\NI_S \subseteq\NI$ of individual constants, called {\em standard
  names} s.t. in every interpretation $\I$ we have (i)
 $\Delta^\I = \NI_S^\I = \{ c^\I \mid c \in \NI_S\}$; (ii) $c^\I
 \neq d^\I$, for every distinct $c,d \in \NI_S$. Thus, we may assume
 that $\Delta^\I= \NI_S$ and $c^\I=c$ for each $c\in \NI_S$. 
  The \emph{unique name assumption (UNA)} 
	corresponds to assuming
  $c\neq d$ for all constants in $\NI\setminus \NI_S$ resp.\ occurring in 
the knowledge base.%
\footnote{Under the SNA, equality between elements can be embedded using a
binary predicate $\approx$ that satisfies the usual congruence
axioms \cite{Fitting-FirsOrdeLogiAuto:96}.}
%
We confine here to 
knowledge bases without reflexivity axioms. The reason 
is that reflexivity
allows one to derive positive properties for any (named and unnamed) individual,
thus complicating the treatment of defeasible axioms.


\smallskip\noindent
\textbf{Datalog programs and Answer Sets.}
We express our rules in
\emph{datalog with negation} 
under
answer sets semantics. 
In fact, we use here two kinds of 
negation\footnote{Strong negation can be easily emulated using fresh
atoms and weak negation resp.\ constraints. While it does not yield higher expressiveness, it is
more convenient for presentation.}: 
strong (``classical'') negation $\non$ and weak \emph{(default) negation}\/ $\naf$
under the interpretation of answer sets semantics~\cite{gelf-lifs-91};
the latter is in particular needed for representing defeasibility.

A \emph{signature} is a tuple $\Pair{\DC}{\DP}$ of  a finite set $\DC$  of \emph{constants}
and a finite set $\DP$  of \emph{predicates}.
We assume a set $\DV$ of \emph{variables}; the elements of $\DC \cup
\DV$ are \emph{terms}.
%
An \emph{atom}
is of the form $p(t_1, \ldots, t_n)$
where $p \in \DP$ and $t_1$, \ldots, $t_n$, are terms.
A \emph{literal} $l$ is either a \emph{positive literal} $p$ or a 
\emph{negative literal} $\non p$, where  $p$ is an atom and $\non$ 
is 
strong negation. Literals of the form $p$, $\non p$ are \emph{complementary}.
We denote with $\neg. l$ the opposite of literal
$l$, i.e., $\neg.p = \non p$ and $\neg.\non p = p$ for an atom $p$.
\rewnote{2}{"$\non.l$" $\rightarrow$ do you mean "$\non l$" ? why using a dot?}
\lbnote{Added reply: The dot notation was used to denote the negation as an operator on (possibly negated) literals, having that $\non.\non p = p$ for an atom $p$.}
A (datalog) rule $r$ is an expression: 
\begin{equation}
\label{rule}
a \leftarrow b_1, \dots, b_k, \naf b_{k+1}, \dots, \naf b_{m}.
\end{equation}
where $a, b_{1}, \dots, b_{m}$ are literals. 
We denote with $\Head(r)$ the head $a$ of rule $r$ and with
$\Body(r) = \{b_1, \dots, b_k,\naf b_{k+1}, \dots,$ $\naf b_{m}\}$ the body of $r$, respectively.
A (datalog) \emph{program} $P$ is a finite set of rules.
%
An atom (rule etc.) is \emph{ground} if 
no variables occur in it. A \emph{ground substitution} $\sigma$ for $\Pair{\DC}{\DP}$
is any function $\sigma \,{:}\, \DV \to \DC$;
the \emph{ground instance} of an atom (rule, etc.) $\chi$ from
$\sigma$, denoted $\chi\sigma$, is obtained by replacing in $\chi$
each occurrence of variable $v \in \DV$ with $\sigma(v)$.
A \emph{fact} $H$ is a ground rule $r$ with empty body.
The \emph{grounding}\/ of a rule $r$, $\grd(r)$, is the set of all
ground instances of $r$, and the \emph{grounding}\/ of a program $P$
is $\grd(P) = \bigcup_{r\in P} \grd(r)$.

Given a program $P$, the \emph{(Herbrand) universe} $U_P$ of $P$ is the set of all
constants occurring in $P$ and the \emph{(Herbrand) base}
$B_P$ of $P$ is the set of all the 
ground literals 
constructable from the predicates in $P$ and the 
constants in $U_P$.
An \emph{interpretation} $I \subseteq B_P$ is any 
satisfiable subset of
$B_P$ (i.e., not containing complementary literals); 
a literal $l$ is \emph{true} in $I$, denoted $I\models l$,
if $l \in I$, and $l$ is \emph{false} in $I$ if $\neg.l$ is true.
%
Given a rule $r \in \grd(P)$,
we say that $\Body(r)$ is true in $I$, denoted $I\models \Body(r)$, if (i) $I\models b$ for each literal 
$b \in \Body(r)$ 
and (ii) $I\not\models b$ for each literal $\naf b\in \Body(r)$.
A rule r is \emph{satisfied} in $I$, denoted $I\models r$, if either 
$I\models \Head(r)$ or $I\not\models \Body(r)$.
An interpretation $I$ 
is a {\em model}\/ of $P$, denoted $I \models P$,
if $I\models r$ for each $r\in \grd(P)$;
moreover, $I$ is 
\emph{minimal}, 
if $I'\not\models P$ for each subset $I'\subset I$.

Given an interpretation $I$ for $P$, the 
\emph{reduct} of $P$ w.r.t. $I$ \cite{gelf-lifs-91}, denoted by $G_I (P)$, is the set of rules obtained from $\grd(P)$ by 
  (i) removing every rule $r$ such that 
 $I\models l$ for some $\naf l\in\Body(r)$; and
  (ii) removing the NAF part from the bodies of the remaining rules.
Then, $I$ is an \emph{answer set} of $P$, if $I$ is a minimal
model of $G_I(P)$; the minimal model is
unique and exists iff $G_I(P)$ has some model. 
Moreover, if $M$ is an answer set for $P$, then $M$ is a minimal model of $P$.
We say that a literal $a \in B_P$ is a \emph{consequence} of $P$ and write
$P \models a$ if every answer set $M$ of $P$ fulfills $M \models a$.


\section{DL Knowledge Base with Justifiable Exceptions}
\label{sec:dkb}


In this paper we concentrate on reasoning over a DL knowledge base enriched
with \emph{defeasible axioms}, whose syntax and interpretation are 
analogous to~\cite{BozzatoES:18}. With respect to the contextual framework
presented in~\cite{BozzatoES:18}, this corresponds to reasoning
inside a single local context: while this simplifies the
presentation of defeasibility aspects and the resulting reasoning method
for the case of $\DLliteR$, it can be generalized to the original case of multiple
local contexts.

\smallskip\noindent
\textbf{Syntax.}
%
Given a DL language $\Lcal_\Sigma$ based on a DL vocabulary 
$\Sigma = \NC \cup \NR \cup \NI$,
a \emph{defeasible axiom} is any expression of the form
$\default(\alpha)$, where $\alpha \in \Lcal_\Sigma$.

We denote with $\Lcal_\Sigma^\default$ the DL language extending
$\Lcal_\Sigma$ with the set of defeasible axioms in $\Lcal_\Sigma$.
On the base of such language,
we provide our definition of knowledge base with defeasible axioms.

\begin{definition}[defeasible knowledge base, DKB]
A \emph{defeasible knowledge base (DKB)} $\Kcal$ on
a vocabulary $\Sigma$ is a DL knowledge base over
$\Lcal^\default_\Sigma$.
\end{definition}
In the following, we tacitly consider DKBs based on $\DLliteR$.

\begin{example}
\label{ex:syntax}
We introduce a simple example showing the definition 
and interpretation of a defeasible existential axiom.
In the organization of a university research department, we want to specify that ``in general''
department members need also to teach at least a course. 
On the other hand, PhD students, while recognized as department members,
are not allowed to hold a course.
We can represent this scenario as a DKB $\K_{dept}$ where:

\begin{center}
  $\begin{array}{rl}
    \K_{dept}: & \left\{\begin{array}{l}
		           \default(\mi{DeptMember} \subs \exists \mi{hasCourse}),
							 \mi{Professor} \subs \mi{DeptMember},\\
							 \mi{PhDStudent} \subs \mi{DeptMember}, 
							 \mi{PhDStudent} \subs \non \exists \mi{hasCourse},\\
		           \mi{Professor}(\mi{alice}),\, \mi{PhDStudent}(\mi{bob})  
							\end{array}\right\}
  \end{array}$	
\end{center}
Intuitively, 
we want to override the
fact that there exists some course assigned to the PhD student $\mi{bob}$.
On the other hand, for the individual $\mi{alice}$ no overriding should happen
and the defeasible axiom can be applied. \EndEx
\end{example}

\smallskip\noindent
\textbf{Semantics.}
We can now define a model based interpretation of DKBs,
in particular by providing a semantic characterization to 
defeasible axioms.

Similarly to the case of $\SROIQrl$ in~\cite{BozzatoES:18}, 
we can express $\DLliteR$ knowledge bases in first-order (FO)
logic, where every axiom $\alpha \in \Lcal_\Sigma$
is translated into an equivalent FO-sentence
$\forall\vec{x}.\phi_\alpha(\vec{x})$ where
$\vec{x}$ contains all free variables of $\phi_\alpha$ depending on
the type of the axiom.
The translation, depending on the axiom types, is natural and can be defined analogously to the
FO-translation presented in~\cite{BozzatoES:18}.\footnote{A FO-translation for $\DLliteR$ axioms is provided in~\ref{sec:fo-translation}.}
In the case of existential axioms of the kind
$\alpha = A \isa \exists R$, the FO-translation
$\phi_\alpha(\vec{x})$ is defined as:

\smallskip

\centerline{$A(x_1) \rightarrow R(x_1, f_R(x_1))$\,;}

\smallskip

\noindent that is, we introduce a Skolem function $f_R(x_1)$
which represents new ``existential'' individuals.
%
%
Formally, for every 
atomic role $R \in \NR$
we define a Skolem function 
$f_R$.
In particular, for a set of individual names $N \subseteq \NI$, we will write $sk(N)$ to denote the 
extension of $N$ with the set of Skolem constants for elements in $N$,
i.e., for each name $a \in N$, $sk(N)$ also contains $f_R(a)$ 
for each $f_R$ as above.
%

After this transformation,
the resulting formulas $\phi_\alpha(\vec{x})$
amount semantically to Horn formulas, since
left-side concepts $C$ can be expressed by an existential
positive FO-formula, and right-side concepts $D$ by a conjunction
of Horn clauses. The following property from~\cite[Section 3.2]{BozzatoES:18} 
is then preserved for $\DLliteR$ knowledge bases.

\begin{lemma}
\label{lem:horn-equiv}
For a DL knowledge base $\K$ on $\Lcal_\Sigma$,
its FO-translation 
$\phi_\K \,{:=}\,\bigwedge_{\alpha \in \K}\!\!
  \forall\vec{x}.\phi_\alpha(\vec{x})$
is semantically equivalent to a conjunction of universal Horn clauses.
\end{lemma}
%
We remark that the introduction of Skolem functions does not 
allow us to work on proper Herbrand models of the original language as in~\cite{BozzatoES:18},
since they introduce new Skolem terms in the language.
As we will see in the following, exceptions on these elements
need further conditions to be defined.

With these considerations on the definition of FO-translation, 
we can now provide our definition of axiom instantiation: 

\begin{definition}[axiom instantiation]
Given an axiom $\alpha \in \Lcal_\Sigma$ with FO-translation
$\forall\vec{x}.\phi_\alpha(\vec{x})$, the instantiation  of $\alpha$
with a
tuple $\ee$ of individuals in $\NI$, 
written $\alpha(\ee)$, is the
specialization of $\alpha$ to $\ee$, i.e., $\phi_\alpha(\ee)$,
depending on the type of $\alpha$.
\end{definition}
Note that, since we are assuming standard names, this basically means that
we can express instantiations (and exceptions) to any element of the domain (identified by a standard name in $\NI_S$).
We next introduce clashing assumptions and clashing sets.


\begin{table}[t]%
\caption{(Minimal) clashing sets for $\DLliteR$ clashing assumptions.}
\label{tab:clashingsets}
\vspace{2ex}
\hrule\mbox{}\\[.5ex]
\small
$\begin{array}{rl}               
\stru{A(a),a} : & \{\non A(a) \} \\
\stru{\non A(a),a} : & \{ A(a) \} \\
\stru{R(a,b),(a,b)} : & \{ \non R(a,b) \} \\

\stru{A \subs B, e} : & \{A(e), \non B(e) \} \\
\stru{A \subs \non B, e} : & \{A(e), B(e) \} \\

\stru{\exists R \subs B, e} : & \{\exists R(e), \non B(e)\} \\
\stru{A \subs \exists R, e} : & \{A(e), \non \exists R(e) \}
\end{array}$
$\begin{array}{rl}               
\stru{R \subs T, (e_1,e_2)}             : & \{R(e_1,e_2), \non T(e_1.e_2)\}\\
\stru{\mathrm{Dis}(R,S), (e_1,e_2)}     : & \{R(e_1,e_2), S(e_1,e_2)\}\\
\stru{\mathrm{Inv}(R,S), (e_1,e_2)}     : & \{R(e_1,e_2), \non S(e_2,e_1)\},\\
                                          &  \{\non R(e_1,e_2), S(e_2,e_1)\}\\
\stru{\mathrm{Irr}(R), e}           : & \{R(e,e)\}\\[1ex]
\end{array}$\\[.5ex]
\hrule\mbox{}
\nop{\comment{TE: changed for $\stru{\exists R \subs B, e}$ in the clashing
set $R(e,f)$ to $\exists R(e)$. Else we might experience difficulties
with rewriting, as the witnessing $f$ might change. And, it is easier
to define $n$-boundedness. Corrected the first clashing set for $Inv$.}}
\end{table}

\begin{definition}[clashing assumptions and sets]
A \emph{clashing assumption} is a pair $\stru{\alpha, \ee}$
s.t.\ $\alpha(\ee)$ is an instantiation 
for an axiom $\alpha \in  \Lcal_\Sigma$. 
A \emph{clashing set} for a clashing assumption $\stru{\alpha,\ee}$
is a satisfiable set $S$ that consists of ABox assertions over
$\Lcal_\Sigma$ and negated ABox assertions of the forms $\neg
C(a)$ and $\neg R(a,b)$
such that
$S \cup \{\alpha(\ee)\}$ is unsatisfiable.
\end{definition}
\rewnote{2}{please add an example of clashing assumption and clashing set after Definition 3}
A clashing assumption $\stru{\alpha, \ee}$
represents
that $\alpha(\ee)$ is not satisfiable, 
while a clashing set $S$ provides an assertional ``justification'' for the assumption of 
local overriding of $\alpha$ on~$\ee$.
In Table~\ref{tab:clashingsets} we show the form of clashing sets for 
axioms of $\DLliteR$.
\comment{For example, in the case of an atomic concept inclusion 
defeasible axiom $\default(A \isa B)$ in a context $\mlc$, 
a clashing assumption $\stru{A \isa B, e}$ states the assumption that $A \isa B$
is not satisfiable for $e$ in $\mlc$; a clashing set $S = \{A(e), \non B(e)\}$
provides a justification for the assumption on the overriding of $A \isa B$ on $e$ in $\mlc$.}
We can then extend the notion of DL interpretation 
with a set of clashing assumptions. 

\begin{definition}[CAS-interpretation]
A \emph{CAS-interpretation} is a structure $\I_{\CAS} = \stru{\I, \casmap}$
where $\I = \stru{\Delta^\I, \cdot^\I}$ is a DL interpretation for $\Sigma$ and 
$\casmap$ is a set of clashing assumptions.
\end{definition}
\rewnote{1}{(Check use of ``CAS-interpretation'' vs. ``CAS interpretation'', model, ecc.)}
\lbnote{changed all occurrences to ``CAS-interpretation'' and ``CAS-model''}
By extending the notion of satisfaction with respect to CAS-interpretations,
we can disregard the application of defeasible axioms to the
exceptional elements in the sets of clashing assumptions. 
For convenience, we 
call two DL interpretations $\I_1$ and $\I_2$
\emph{$\NI$-congruent}, if
$c^{\I_1} = c^{\I_2}$ 
holds for every $c\in \NI$.

\begin{definition}[CAS-model]
\label{def:cas-model}
Given a DKB $\Kcal$, 
a CAS-interpretation $\I_{\CAS} = \stru{\I, \casmap}$
is a CAS-model for $\Kcal$ (denoted $\I_{\CAS} \models
  \Kcal$), if the following holds:
\begin{enumerate}[label=(\roman*)]
  \item
   for every $\alpha \in \Lcal_\Sigma$ in $\Kcal$, $\I \models \alpha$;
  \item
   for every  $\default(\alpha) \in \Kcal$ (where $\alpha \in \Lcal_\Sigma$),
   with $|\vec{x}|$-tuple $\vec{d}$ of elements 
   in $\NIs$ such that $\vec{d} \notin \{ \ee \mid \stru{\alpha,\ee} \in \casmap \}$, 
   we have $\I \models \phi_\alpha(\vec{d})$.
\end{enumerate}
\end{definition}
%
We say that a clashing assumption $\stru{\alpha, \ee} \in \casmap$ is
\emph{justified} for a $\CAS$-model $\I_{\CAS} = \stru{\I, \casmap}$,   
if some clashing set
$S = S_{\stru{\alpha,\ee}}$  exists such that, for every CAS-model
$\I_{\CAS}' = \stru{\I', \casmap}$ of $\Kcal$ 
that is $\NI$-congruent with $\I_{\CAS}$, 
it holds that $\I' \models S_{\stru{\alpha,\ee}}$.
We then consider as DKB-models
only the CAS-models
where all clashing assumptions are justified.

\begin{definition}[justified CAS-model and DKB-model]
A $\CAS$-model $\I_{\CAS} = \stru{\I, \casmap}$ of 
a DKB $\Kcal$ is \emph{justified}, if every $\stru{\alpha, \ee} \in
\casmap$ is justified.
An interpretation $\I$ 
is a \emph{DKB-model} of $\K$ (in
symbols, $\I\models\K$), if $\K$ has some  
justified $\CAS$-model $\I_{\CAS} = \stru{\I, \casmap}$.
\end{definition}
%
\begin{example}
\label{ex:2}
	Reconsidering $\K_{dept}$ in Example~\ref{ex:syntax},
	a CAS-model providing the intended interpretation of 
	defeasible axioms is $\I_{\CAS_\mi{dept}} = \stru{\Ical,\chi_\mi{dept}}$ where
  $\mi{bob}^\Ical \neq \mi{alice}^\Ical$ and $\chi_\mi{dept} = \{\stru{\alpha, \mi{bob}}\}$
	with $\alpha = \mi{DeptMember} \subs \exists \mi{hasCourse}$.
	The fact that this model is justified is verifiable considering that for
	the clashing set $S = \{\mi{DeptMember}(\mi{bob}),$ $ \non \exists \mi{hasCourse}(\mi{bob})\}$
	we have $\I \models S$.
	On the other hand, note that a similar clashing assumption for $\mi{alice}$ is not justifiable:
	it is not possible from the contents of $\K_{dept}$ to derive a clashing set $S'$
	such that $S' \cup \{\alpha(\mi{alice})\}$ is unsatisfiable.
	By Definition~\ref{def:cas-model}, 
	this allows us to apply $\alpha$ to this individual as expected and thus
	$\I \models \exists \mi{hasCourse}(\mi{alice})$.\EndEx
\end{example}
\rewnote{1}{First part of the paper uses a working example, which helps clarifing concepts, but then it is recovered only at the end for reasoning on CQ. Could be employed in a wider way in the paper in order to clarify the non-always-easy concepts dealt with?}
\rewnote{2}{The paper is mostly well-written and organized, but some additional examples could improve the reading and help to grasp the main concepts.}
\rewnote{2}{Example 4 illustrates a "normal" application of a default. If I understood correctly, some theories may have alternative extensions (the authors mention the Nixon Diamond in Related Work). Adding an example to show multiple extensions is crucial for a better understanding of the effects of the current semantics. Similarly, could it be the case that a consistent DL-LiteR theory (that is, when defaults are assumed as strong axioms) may become inconsistent (that is, no n-de safe extension) when we consider its non-monotonic semantics? More effort should be put in explaining the effects of the present semantics under a non-monotonic reasoning perspective.}

\comment{
\begin{example}[Nixon Diamond]
\label{ex:nixon}
	Note that different combinations of clashing assumptions can lead to different and alternative 
	justified CAS-models and thus alternative DKB-models.
	We can show this by considering the classic example of the \emph{Nixon Diamond}
	as presented in~\cite[Example 9]{BonattiFPS:15} (see also the example in~\cite[Section 7.4]{BozzatoES:18}).
	Let $\KB_{nd}$ be a DKB defined as follows:
	\begin{center}
  $\begin{array}{rl}
    \K_{nd}: & \left\{\begin{array}{l}
		           \default(\mi{Quacker} \subs \mi{Pacifist}),
		           \default(\mi{Republican} \subs \non\mi{Pacifist}),\\
		           \mi{Quacker}(\mi{nixon}),\, \mi{Republican}(\mi{nixon})  
							\end{array}\right\}
  \end{array}$	
	\end{center}
	This DKB has two possible overridings of the two defeasible axioms (having the same priority),
	which lead to two possible DKB-models $\I_1$ and $\I_2$.
	In particular, in model $\I_1$ we have a clashing assumption 
	$\chi_1 = \{\stru{\mi{Republican} \subs \non\mi{Pacifist}, \mi{nixon}}\}$
	that is justified by the clashing set\linebreak 
	$\{\mi{Republican}(\mi{nixon}), \mi{Pacifist}(\mi{nixon}) \}$: in this model
	we have that $\I_1 \models \mi{Pacifist}(nixon)$.
	Similarly, in model $\I_2$ we have the clashing assumption 
	$\chi_2 = \{\stru{\mi{Quaker} \subs \mi{Pacifist}, \mi{nixon}}\}$ with
	clashing set $\{\mi{Quaker}(\mi{nixon}), \non\mi{Pacifist}(\mi{nixon}) \}$: 
	then, in this model
  $\I_2 \models \non\mi{Pacifist}(nixon)$.	
	
	Thus, we obtain that $\K_{nd} \not\models \mi{Pacifist}(nixon)$
	and $\K_{nd} \not\models \non\mi{Pacifist}(nixon)$. 
	Similarly, the approach presented in~\cite{BonattiFPS:15} can not
	derive $\mi{Pacifist}(nixon)$ or $\non\mi{Pacifist}(nixon)$: however, 
	as we have shown in~\cite{BozzatoES:18}, differently from this approach, 
	in our semantics we can use the alternative models to enable ``reasoning by cases''.
	For example, consider the DKB $\K'_{nd}$ obtained from $\K_{nd}$ by substituting
	$\default(\mi{Republican} \subs \non\mi{Pacifist})$ with the axioms
	$\default(\mi{Republican} \subs \mi{Hawk}),$
	$\mi{Hawk} \subs \non\mi{Pacifist},$
	$\mi{Hawk} \subs \mi{Activist}, \mi{Pacifist} \subs \mi{Activist}$.
	Then, differently from~\cite{BonattiFPS:15}, 
	even with alternative models given by the possible 
	instantiation of clashing assumptions, we obtain that
	$\K'_{nd} \models \mi{Activist}(nixon)$.\EndEx
\end{example}}
%
%
We are interested in DKB-models $\I_{\CAS} = \stru{\I, \casmap}$
in which clashing assumptions $\stru{\alpha,\ee}\in \casmap$
are admitted over individuals that are both named and unnamed in
the knowledge base.
That is, we want to admit also exceptions
over 
individuals introduced by existential axioms.
However, we should have the means to limit the existence of such individuals in exceptions
in order to 
control the reasoning from such models.

Let us denote by $N_\K$ the individuals occurring in $\K$.
A condition that allows us to control the number of unnamed individuals in models
is the following.

\nop{***** old def
\begin{definition}
A CAS-interpretation $\I_{\CAS} = \stru{\I, \casmap}$ for a DKB $\K$
is \emph{$n$-bounded}\/ for $n\geq 0$, if 
at most $n \geq 0$ distinct individual names $d \in sk(N_\K) \setminus N_\K$ can appear in
clashing assumptions in $\chi$.
\end{definition}
\tenote{Here is a formulatio of this condition, we talk about *any* unnamed individuals.}
*****}

\begin{definition}[$n$-boundedness]
\label{defn:n-bounded}
A CAS-interpretation $\I_{\CAS} = \stru{\I, \casmap}$ for a DKB $\K$
is \emph{$n$-bounded}\/ for $n\geq 0$, if 
in $\casmap$ at most $n$ elements occur that are not named by $\K$, i.e., 
it holds that $|\mi{uni}_{\K}(\I_{\CAS})| \leq n$ where $\mi{uni}_{\K}(\I_{\CAS}) = \{ e_1^\I,\ldots,e_k^\I \in \NI_S \mid
\stru{\alpha,(e_1,\ldots,e_k)} \in \casmap\} \setminus \{ c^\I \mid c
\in N_\K \}$.
\end{definition}
%
%
If this condition holds, we can show that unnamed individuals appearing in 
clashing assumptions can be always linked to named individuals from the DKB.
Given a DKB $\K$, let us denote by $\K_s$ the knowledge base where 
all defeasible axioms are turned into strict
axioms.

\nop{**** old version 
\begin{lemma}
	Given a $n$-bounded CAS-interpretation $\I_{\CAS} = \stru{\I, \casmap}$,
	for every $d \in sk(N_\K) \setminus N_\K$ appearing in a clashing assumption $\stru{\alpha,\ee}\in \casmap$,
  there exists an individual name $a \in N_\K$ appearing in a positive literal in the
  clashing set $S_{\stru{\alpha,\ee}}$ 
  such that there is a role chain 
	$R_0(a, e_0), \ldots, R_m(e_{m-1}, d)$ 
	with $\{e_0, \ldots, e_{m-1}\} \subseteq sk(N_\K) \setminus N_\K$.
\end{lemma}

\tenote{Reformulated this lemma as a different condition, for one ``unnamed
exception element at a time.''}******}

\begin{lemma}
Suppose that $\I_{\CAS} = \stru{\I, \casmap}$ is a justified CAS-model of
a DKB $\K$ and that an element $e$ occurring in
$\stru{\alpha,\ee}\in \casmap$ is not named by $\K$. Then, there exists
a role chain $R_1^\I(e_0, e_1), \ldots,$  $R_m^\I(e_{m-1},e_m)$ where
$e_0 = a^\I$ for some $a \in N_\K$, $e_m=e$, 
and $e_{i+1} = f_{R_{i+1}}(e_i)$. 
\end{lemma}
That is, elements in clashing assumptions that are not named by the DKB
must be linked to it by a ``Skolem chain''.
%

\rewnote{2}{I'm not sure which is the main definition of the proposed semantics. As far as I understood, we are keeping those extensions that are n-de safe. If so, this should be emphasized in the text (something like "Main Definition" or similar).
The definition of n-de safe is based on the notion of "derivability". When you say "derivable", which logical setting do you have in mind? First Order Logic? If not, the concept of derivable should be formally defined.}
\lbnote{I am not sure if it is correct or more confusing to say that n-de safeness is the ``main definition'' of the semantics. I added that we consider FOL when referring to derivability.}

\comment{As our main definition for limiting models to
$n$-bounded CAS-interpretations, we consider the following syntactic
condition restricting unnamed individuals appearing in clashing sets.}

\nop{*********** old definition 
\begin{definition}
  A DKB $\K$ is \emph{$n$-chain exception safe} iff
   for every clashing assumption $\stru{\alpha,\ee}$ with $e \in \ee$ not
   appearing in $N_\K$,
   from $\K_s$ it can be derived a role chain of the kind
	 $R_1(a, e_1), \ldots, R_m(e_{m-1}, e_m)$ 
	 with $m \leq n$, $a \in N_\K$ and $e \in {e_1, \ldots, e_m} \subseteq sk(N_\K)$.	
\end{definition}
}

\nop{
\tenote{formulate another  generalization of exception safety; the
one above is to weak, axioms $A \isa \exists R.A$ with
a defeasible axiom $D(B\isa C)$ would already break it.}
***}

\begin{definition}[$n$-derivation exception (de) safety]
A DKB $\K$ is \emph{$n$-derivation exception (de) safe},
if $m\leq n$ 
Skolem terms $t_1,\ldots,t_m$ exist
such that for every positive assertion $D(e_1)$ resp.\ 
$R(e_1,e_2)$ from a possible clashing set $S_{\stru{\alpha,\ee}}$ for
any $D(\alpha) \in \K$ and atom $D(t'_1)$ resp.\ $R(t'_1,t'_2)$
that is derivable from $\K_s$ (\comment{in FO} under Skolemization), it holds that 
we have $t'_1 \in N_K \cup \{ t_1,\ldots, t_m\}$ resp.\ 
$t'_1,t'_2 \in N_K \cup \{ t_1,\ldots, t_m\}$.
\end{definition}
In particular, for $n=0$ we obtain that no exception on
an unnamed individual can be derived:
in this case, we say that $\K$ is \emph{exception safe}.
If $\K$ is acyclic,%
\footnote{$\K$ is acyclic, if there is no sequence of axioms $E_0 \isa
E_1$, $E_1\isa E_2$,\ldots, $E_{k-1}\isa E_k$ such that
$E_k=E_0$.} then it is $n$-de safe for some $n$ that is exponential
in the size of $\K$ in general, which drops to polynomial if
derivations are feasible in constantly many steps.

\begin{example}[Ex.~\ref{ex:2} cont'd]
Reconsider the CAS-model $\I_{\CAS_\mi{dept}} = \stru{\Ical, \chi_\mi{dept}}$ where
$\chi_\mi{dept} = \{\stru{\alpha, \mi{bob}}\}$
with $\mi{bob}^\Ical \neq \mi{alice}^\Ical$, 
and $\alpha = \mi{DeptMember} \subs \exists \mi{hasCourse}$. 
If we make $\alpha$ strict,
we cannot derive a clashing set $S = \{\mi{DeptMember}(e),$
$\neg\exists \mi{hasCourse}(e)\}$ where $e$ is an unnamed individual;
to derive $\mi{DeptMember}(e)$, it would require some axiom $\exists R^-
\isa  \mi{DeptMember}$ where some unnamed individual is introduced
by some axiom $A \isa \exists R$; however, no such former axioms
 can be derived, and thus $\K_{dept}$ is exception-safe.
\EndEx
\end{example}
However, if $\K$ is
cyclic, it may be not $n$-de safe for every $n \geq 0$.  
%
\begin{example}
\label{ex:mother}
Let us consider the DKB $\DKB = \{ \mi{Employee}
\isa\, \exists \mi{hasSupervisor},$ $\mi{\exists hasSupervisor^- \isa\, Employee},$
$\default(\exists\mi{hasSupervisor}^-\isa\, \bot)$, $\mi{Employee(alice)} \}$.%
\footnote{Here, $\bot$ is the empty concept (``falsity'') emulated by $\bot
\isa A$, $\bot \isa \non A$ for a fresh concept name $A$.}
Informally, every employee and so $\mi{alice}$ has a supervisor
that is an employee, and unless provable to the contrary, an individual
is not a supervisor. This KB has an infinite feed of Skolem terms
$f^n(\mi{alice})$, $n\geq 1$ into the defeasible axiom by the
chain $\mi{hasSupervisor}(\mi{alice},f(\mi{alice})),$
$\mi{hasSupervisor}(f(\mi{alice}),f(f(\mi{alice}))), \ldots$.  Thus, $\DKB$
is not $n$-de safe for any $n \geq 0$. It has two non-isomorphic
DKB-models: one, $\I_{\CAS}^1$, where we have an exception to
$\default(\exists\mi{hasSupervisor}^-\isa\ \bot)$ for $\mi{alice}$ (thus the
model is 0-bounded) and
$f(\mi{alice})=\mi{alice}$, and another one,  $\I_{\CAS}^2$, where we have an exception
for $f(\mi{alice})$ and $\mi{alice} \neq f(\mi{alice})$,
$f(\mi{alice}) = f(f(\mi{alice}))$ (the model is 1-bounded). No longer Skolem chain of three
different elements is possible: then two exceptions would be needed,
which then are however not provable.
If we add the assertion $\non \exists \mi{hasSupervisor}^-(\mi{alice})$ to $\DKB$
stating that $\mi{alice}$ is not a supervisor, then only the DKB-model
$\I_{\CAS}^2$ remains; adding
the assertion
$\mi{Employee(bob)}$ instead, we obtain under UNA a further DKB-model
  $\I_{\CAS}^3$ 
with an exception of $\default(\exists\mi{hasSupervisor}^-\isa\ \bot)$ for
$\mi{bob}$; an exception for both 
$\mi{alice}$ and $\mi{bob}$ is infeasible as this would not be justifiable.
%
\EndEx
\end{example}
%
A syntactic property of DKBs that is useful to be verified is related to the reachability of 
unnamed elements in derivations.


  \begin{definition}[$n$-chain safety]
A DKB $\K$ is \emph{$n$-chain safe}, if 
from $\K_s$ only role chains 
$R_1(a,t_1), \ldots, R_m(t_{m-1},t_m)$ 
where $a \in N_\K$ and 
the $t_1, \ldots, t_m$ are distinct Skolem terms can be derived 
such that $m\leq n$.
\end{definition}
This condition, in particular, is verified in the case that $\K$ is acyclic:
in this case the maximum length of chains is determined by the chains of
existential axioms in $\K$.


If a DKB $\K$ is $n$-chain safe then it is
also $m$-de safe for some $m$ that is exponentially bounded by $n$. 
On the other hand, $\K$ may be $n$-de safe but not $m$-chain safe for any
$m\geq 0$: in the latter case, recursion through axioms that do not
feed into defeasible axioms occur. For instance, if we drop in
Example~\ref{ex:mother} the defeasible axiom
$\default(\exists\mi{hasSupervisor}^-\isa\ \bot)$, then the resulting $\DKB$ is 
trivially exception safe but not $n$-chain bounded. For our purposes,
we shall call a DKB $\K$  {\em recursive}, if $\K$ is not $n$-de
bounded for any $n \geq 0$.

In case of exception safe (i.e., $0$-de safe) DKBs
we obtain the following result.

\begin{proposition}
\label{prop:pushing-eq}
  Let $\I_{\CAS} \,{=}\, \stru{\I, \casmap}$ be a CAS-model of DKB 
$\K$ and let $\K'$ result from $\K$ by pushing equality w.r.t.\ $\I$,
  i.e., replace all $a, b \,{\in}\, N_\K$ s.t.\ $a^\I\,{=}\,b^\I$  by one representative.
If $\K'$ is exception-safe,
then $\I_{\CAS}$ can be justified only if every $\stru{\alpha,\ee}\,{\in}\, \casmap$ is over $N_\K$.
\end{proposition}
%
%
We remark that the condition of exception-safety can be tested in polynomial time, by 
non-deterministically unfolding the axioms (resolution-style, or
forward in a chase).
In fact, we obtain the following result. 

\begin{proposition}
\label{prop:complexity-recognize-safe}
Deciding whether a given DKB $\DKB$ is exception safe is feasible in
\nlogspace, and whether it is $n$-de safe in \ptime, if $n$ is bounded by a polynomial in the size of $\K$. 
\end{proposition}
Syntactic classes ensuring exception safety can be singled out: 
simple examples of these can be the class of DKBs containing no
existential axioms or the class where no inverse roles and no 
defeasible role axioms appear in the DKB.

Checking chain-safety is tractable, similarly to testing exception safety.

\begin{proposition}
\label{prop:chain-bounded}
Deciding whether a given DKB $\DKB$ is $n$-chain safe, where $n\geq 0$, is feasible in
\nlogspace.
\end{proposition}
%
We remark that both checking exceptions and $n$-chain safety are
in fact \nlogspace-complete, as the hardness is inherited from the  
\nlogspace-completeness of $\DLliteR$ (which holds already in the
absence of existential axioms).


\section{Semantic Properties}
\label{sec:properties}

DKB-models have interesting semantic properties similar to 
those exhibited by CKR-models in
\cite{BozzatoES:18}. 
In this section 
we provide a review of such properties: in particular, these results
are important to show the feasibility of the reasoning approach presented
in Section~\ref{sec:translation}.

For example, we can prove that justified CAS-models have a non-monotonic behaviour
with respect to the contents of DKBs,
cf. \cite[Prop.~4, non-monotonicity]{BozzatoES:18}. 

\begin{proposition}[non-monotonicity]
Suppose $\I_{\CAS} = \stru{\I,\casmap}$ is a justified
CAS-model of a DKB $\K'$. 
Then, $\I_{\CAS}$ is not necessarily a 
	justified CAS-model of every $\K \subset \K'$.
\end{proposition}
\begin{proof}
  This property can be easily verified by considering the interpretation 
	of defeasible axioms and their justification.
	Let us suppose that $\default(A \isa B) \in \K$ (cases for other defeasible axioms can be shown similarly)
	and $\{A(c), \non B(c)\} \subseteq \K$.
	If we consider a justified CAS-model $\I_{\CAS} = \stru{\I,\casmap}$ for $\K$,
	then the defeasible axiom is not applied to the exceptional instance $c$
	in the interpretation: that is, $\stru{A \isa B,c} \in \chi$ and 
	$S = \{A(c), \non B(c)\}$ is a clashing set for such exception.
	However, if we consider $\K' = \K \setminus \{ \non B(c) \}$, then 
	in $S$ is no longer verified by models of $\K'$: thus the clashing assumption
	$\stru{A \isa B,c}$ can no longer be justified and $\I_\CAS$ is not a 
	justified model for $\K'$.
\end{proof}
%
Another property of justified CAS-models
that we can show is non-redundancy of justifications, cf.\ \cite[Prop.~6, minimality of
justification]{BozzatoES:18}. Basically, this means that
in justified models clashing assumptions are minimal, 
in the sense that no assumption can be omitted.

\begin{proposition}[non-redundancy]
\label{prop:cas-minimality}
Suppose $\I_{\CAS} = \stru{\I, \casmap}$ and 
$\I'_{\CAS} = \stru{\I', \casmap'}$ are NI-congruent justified CAS-models of a DKB $\K$,
then $\casmap' \not\subset \casmap$ holds.
\end{proposition}
\begin{proof}
	Let us consider $\stru{\alpha, \ee} \in \casmap \setminus \casmap'$.
	Then, given that $\I'_{\CAS}$ is a model for $\K$, it holds that
	$\I' \models \alpha(\ee)$ (i.e., $\ee$ is not an exceptional instance of the 
	defeasible axiom $\default(\alpha)$).
	Given that all clashing assumptions in $\chi$ are justified,
	then there exists a clashing set $S = S_{\stru{\alpha, \ee}}$ for the 
	clashing assumption $\stru{\alpha, \ee}$ such that $\I \models S$.
	Moreover, by the definition of justification,
	for every other $\I''_{\CAS} = \stru{\I'', \casmap}$ for $\K$ 
	that is NI-congruent with $\I_{\CAS}$, it holds that $\I'' \models S$.
	
	Let us consider $\I'''_{\CAS} = \stru{\I', \casmap}$.
	Given that $\I'_{\CAS} \models \K$ and $\casmap' \subseteq \casmap$,
	we have that also $\I'''_{\CAS} \models \K$.
	Moreover, since $\I'''_{\CAS}$ is NI-congruent with $\I_{\CAS}$,
	we have $\I' \models S$.
	Then, $\I' \models S \cup \{\alpha(\ee)\}$: however this contradicts the
	fact that $S$ is a clashing set for $\alpha(\ee)$. 
	This proves the fact that $\casmap' \not\subset \casmap$ must hold.
\end{proof}
%
We remark that, as a consequence of this property, 
exceptions in DKB-models are minimal (i.e., \emph{minimally justified}):
thus, in our approach this minimality property is derived from the definition of
the interpretation of defeasible axioms and it is not explicitly 
required in its definition.

As shown in Lemma~\ref{lem:horn-equiv}, $\DLliteR$ knowledge bases
can be represented as Horn theories: however, differently 
from 
$\SROIQrl$~\cite[Prop.~7, intersection property]{BozzatoES:18},
the use of Skolem functions does not allow us to properly preserve the 
intersection property of Horn theories and a revised notion of model intersection
is needed, cf.~\cite[Proof for Prop.~6.7]{BruijnEPT:11}.
%
Formally, for two $\NI$-congruent DL interpretations $\I_1$ and $\I_2$, 
we denote by $\I_1~\widetilde{\cap}_\Ncal~\I_2$ the $\NI$-congruent
``intersection'' interpretation over a set of ground terms $\Ncal$
(i.e., a set of individual names and the possible instantiations of 
Skolem functions over them) defined as follows:
\begin{itemize}
\item 
  $\Delta^{\J} = \{ [t^{\I_1}, t^{\I_2}] \;|\; t \in \Ncal\}$;
\item
  $t^{\J} = [t^{\I_1}, t^{\I_2}]$, for $t \in \Ncal$;
\item
  $[d_1, d_2] \in C^{\J}$ iff $d_1 \in C^{\I_1}$ and $d_2 \in C^{\I_2}$, for $C\in \NC$;
\item
  $([d_1, d_2], [e_1, e_2]) \in R^{\J}$ iff	
	$(d_1, e_1) \in R^{\I_1}$ and $(d_2, e_2) \in R^{\I_2}$, for $R\in \NR$;
\end{itemize} 
\comment{where we abbreviate $\J = \I_1\widetilde{\cap}_\Ncal \I_2$.}
\rewnote{2}{the notation in the first two items before last paragraph is really hard to read.  Avoid using a superscript with a superscripted and subscripted intersection operator and use an auxiliary name instead}
\lbnote{added abbreviation $\J$}
%
Note that, while for the $\NI$-congruence we have that $a^{\I_1} = a^{\I_2}$ for 
individual names $a \in \Ncal \cap \NI$, it is not necessarily true that 
$t^{\I_1} = t^{\I_2}$ for some Skolem term $t \in \Ncal$: 
in this case, by the definition above, 
in the ``intersection'' interpretation we consider \comment{$t^\J$ (i.e., $t^{\I_1\widetilde{\cap}_\Ncal\I_2}$)}
as the ``conjunction'' of the interpretations of $t$ in the two models. 
Extending this construction to CAS-interpretations, we need to 
ensure that the interpretation of the joined interpretations is coherent
on exceptions. Namely, we require the following property:
\begin{equation}
  \mbox{If $t^{\I_i} = e$ and some clashing assumption $\stru{\alpha, \ee}$ with
	$e \in \ee$ exists, then $t^{\I_1} = t^{\I_2}$.} \tag{$*$}
\label{prop:exc-conjunction}
\end{equation}
%
Then,
the following result can be shown:

\begin{proposition} 
\label{prop:model-intersection}
Let  $\I^i_{\CAS} = \stru{\I_i,\casmap}$, $i \in\{1,2\}$, be
$\NI$-congruent CAS-models of
a DKB $\K$ 
fulfilling (\ref{prop:exc-conjunction}).
Then, $\I_{\CAS} = \stru{\I,\casmap}$, where $\I = \I_1 \widetilde{\cap}_\Ncal \I_2$ 
and $\Ncal$ includes all individual names occurring in 
$\K$ and for each element $e$ occurring in $\casmap$
some $t\in \Ncal$ such that $t^{\I_1}=e (= t^{\I_2})$,
is also a CAS-model of $\K$.%
\footnote{Technically, we view here $[e,e] \in
\comment{\Delta^{\I}}$ 
as $e$; the assumption
ensures we have infinitely many standard names left.}
Furthermore, if  some $\I^i_{\CAS}$, $i\in\{1,2\}$, 
is justified and $\DKB$ is exception safe, then $\I_{\CAS}$ is justified. 
\end{proposition}
A
consequence of this result is that a least justified CAS-model exists
for exception safe DKBs 
%
relative to a \emph{name assignment}, which 
we define as any interpretation 
$\nu: \NI \rightarrow  \Delta$ of the
individual constants 
on the domain $\Delta$ (respecting SNA).
The name assignment of a CAS-interpretation
$\I_{\CAS}=\stru{\I,\casmap}$ is the one induced by $\NI^\I$.
We call a clashing assumption $\casmap$ for a DKB $\K$ 
\emph{satisfiable}\/ (resp., \emph{justified}) for a name assignment
$\nu$, if $\K$ has some
CAS-model (resp., justified CAS-model) $\I_{\CAS}$ with name assignment $\nu$. 
Then, by using the construction of ``intersection'' interpretations 
over $\CAS$-models of $\Kcal$ for a given satisfiable $\chi$, we obtain the following result.
Denote for a CAS-model $\IC_{\CAS} = \stru{\I,\casmap}$ 
by $At_\Ncal(\I_{\CAS}) = \{ A(t), R(t,t') \mid t,t' \in \Ncal, 
A \in \NC,\, R \in \NR, \I \models A(t), \I \models R(t,t') \}$
the set of all atomic concepts and roles over $\Ncal$ satisfied by
$\IC_{\CAS}$, and define $\IC_{\CAS} \subseteq_\Ncal \IC'_{\CAS}$ for
CAS-models $\IC_{\CAS}$ and $\IC'_{\CAS}$ by $At_\Ncal(\IC_{\CAS})
\subseteq At_\Ncal(\IC'_{\CAS})$.

\begin{corollary}[least model property]
\label{coroll:least-model}
If a clashing assumption $\casmap$ for an exception-safe DKB $\K$ is satisfiable for
name assignment $\nu$, then $\K$ has an  $\subseteq_\Ncal$-least (unique minimal)%
\footnote{We consider uniqueness modulo equivalence, i.e., that
$\IC_{\CAS}\subseteq_\Ncal \IC_{\CAS}'$ and $\IC'_{\CAS}\subseteq_\Ncal \IC_{\CAS}$}        
CAS-model $\hat{\I}_\K(\casmap,\nu) = \stru{\hat{\I},\casmap}$ on
$\Ncal$ that contains all Skolem terms of individual constants, i.e.,
for every CAS-model $\I_{\CAS}' = \stru{\I',\casmap}$ relative to
$\nu$, it holds that $\I_{\CAS} \subseteq_\Ncal \I'_{\CAS}$.
Furthermore, 
$\hat{\I}_\K(\casmap,\nu)$ is justified if $\casmap$ is justified.
\end{corollary}
We note that, moreover, a Skolem term $t$ over an individual
constant $c$ can occur in the least model $\hat{\I}_\K(\casmap,\nu)$ 
if and only if it has an alias in $\K$, i.e., some $c' \in N_\K$ such that
$\nu(c')=\nu(c)$ exists; thus,  $\hat{\I}_\K(\casmap,\nu)$ is fully
characterized by its restriction to $N_\K$.
We also note that we can reason independently from (in)equalities
that emerge from the name assignment $\nu$ regarding exceptions, 
since modulo $\nu$, no new Skolem terms can occur
in derived positive atoms. This
also means that (in)equalities do not affect (relative to $\nu$) the conditions for $n$-de safety.

As in the case of $\SROIQrl$ knowledge bases in~\cite{BozzatoES:18},
in order to formulate a reasoning method it is important to
show that also in $\DLliteR$ knowledge bases (and DKBs) we can 
concentrate on reasoning over the \emph{named} part of an interpretation.
Notably, in the case of $\DLliteR$ we need to extend this notion to 
consider the interpretation of Skolem individuals.

We say $\I$ is \emph{named} relative to $\Ncal\subseteq sk(\NI) \setminus\NI_S$, 
if 
(i) $C^\I\subseteq \Ncal^\I$ and
$R^\I\subseteq \Ncal^\I{\times} \Ncal^\I$ for each $C\in
 \NC$ and $R\in \NR$; 
(ii) for every Skolem function $f$ and $a \in \Ncal \cap \NI$,
$f(a)^\I \in \Ncal^\I$.
Moreover, for a $\DLliteR$ knowledge base $\K$, 
if $c^\I \neq d^\I$ for
any distinct $c,d\in \Ncal$ and $\Ncal$ includes $sk(N_\K)$ 
(i.e., all constants that occur in $\K$ and their Skolem constants), 
we call $\I$ a \emph{pseudo Herbrand interpretation} for $\K$ relative to $\Ncal$.

Let for any $N\subseteq \NI\setminus\NI_S$
the \emph{$N$-restriction of $\I$}, denoted by $\I^N$, 
be the interpretation that results from 
$\I$ by:
(i) restricting $C^\I$ to $sk(N)^\I$
for all $C\in \NC$ and $R^\I$ to $sk(N)^\I{\times} sk(N)^\I$ for every $R \in \NR$;
(ii) redirecting any role of the type $(e_1, e_2) \in R^\I$ with $e_1 \in sk(N)$
and $e_1 \notin sk(N)$ to $(e_1, g) \in R^\I$ with $g \neq e_1$ and $g \in sk(N)^\I \setminus N^\I$
(i.e., $g$ is the interpretation of some Skolem 
term on $N$).
%
Then, we can obtain the following lemma over such interpretation restrictions.

\begin{lemma}
\label{lem:named}
Suppose $\I$ is a model of a $\DLliteR$ knowledge base $\K$ and $N\subseteq
\NI \setminus\NI_S$ includes all individuals occurring in $\K$. Then, 
the $N$-restriction $\I^N$ is named w.r.t.\ $sk(N)$ and a model of $\K$.
\end{lemma}

\noindent
In the case of exception safe DKBs,
this property can be extended to CAS-interpretations
$\I_{\CAS} = \stru{\I,\casmap}$ 
of DKB $\K$.
Considering $N$ that 
includes each individual constant that occurs in $\K$,
a CAS-interpretation $\I^{N}_{\CAS} = \stru{\I^{N}, \chi^{N}}$ 
can be obtained from $\I_{\CAS}$ by 
(i) replacing $\I$ with its $N$-restriction $\I^{{N}}$,
(ii) removing each clashing assumption $\stru{\alpha,\dd}$ from $\casmap$
where $\dd$ is not over $N$, and 
(iii) interpreting each constant
symbol $c\in sk(\NI) \setminus (sk({N}) \cup \NI_S)$ 
by some arbitrary element not in $sk(N)^\I$.
%
In particular, we consider the case of $N = N_\K$
consisting of the individual constants that occur in $\K$.
Then, we obtain:
\begin{theorem}[named model focus]
\label{theo:named}
Let $\I_{\CAS}$ be a CAS-model of an exception-safe DKB $\K$ and
suppose $N_\K \subseteq N\subseteq \NI\setminus\NI_S$.
Then, also 
$\I_{\CAS}^{N}$, and in particular $\I_{\CAS}^{N_\K}$,
is a CAS-model for $\K$. 
Furthermore, $\I_{\CAS}^N$ is
justified if $\I_{\CAS}$ is justified, and every clashing assumption $\stru{\alpha,\ee}$ in $\I_{\CAS}^N$
is justified by some clashing set $S$ formulated with constants from $sk(N)$.
\end{theorem}
\begin{proof}
	Suppose that $\I_{\CAS} \models \K$, with $\I_{\CAS}=\stru{\I,\casmap}$.
	Then, by the Definition~\ref{def:cas-model} of CAS-model,
	if we consider the restriction $\I^N_\CAS = \stru{\I^N, \chi^N}$,
	from Lemma~\ref{lem:named} we directly obtain condition (i) on strict axioms in $\K$,
	i.e., for every $\alpha \in \Lcal_\Sigma$ in $\K$, $\I^N \models \alpha$.
	We can prove also the satisfaction of condition (ii) on the interpretation of
	defeasible axioms.
	Let $\dd \notin \{ \ee \mid \stru{\alpha,\ee} \in \casmap^N\}$ for
  $\default(\alpha) \in \K$. 
  If $\dd$ is over $N$, then by Lemma~\ref{lem:named} we obtain that
	$\I \models \alpha(\dd)$ and thus $\I^N \models \alpha(\dd)$.
	Otherwise, if $\dd$ is not over $N$, then as noted in the proof of previous lemma,
	in the translation to Horn clauses $\phi_\alpha(\dd)$ there must be 
	a clause $\gamma_i(\dd, \vec{x}_i)$ where some constant outside $N$ occurs
	in the antecedent. This causes $\gamma_i(\dd, \vec{x}_i)$ to evaluate to false
	for every assignment on $\vec{x}_i$.
	
	We can show that if $\I_\CAS$ is justified, then $\I^N_\CAS$ is also justified.
	Let us assume that $\I^N_\CAS$ is not justified: then 
	there exists a $\stru{\alpha,\ee} \in \casmap^N$ (and thus also $\stru{\alpha,\ee} \in \casmap$) 
	that is not justified.
	By the definition of justification, this means that 	
	for every clashing set $S = S_{\stru{\alpha,\ee}}$
  for $\stru{\alpha,\ee}$, there exists some CAS-model
  $\I^{N'}_{\CAS}=\stru{{\I^N}',\casmap^N}$ of $\K$ that is
  $\NI$-congruent to $\I^N_{\CAS}$ and ${\I^N}' \not\models S$.
	In particular, this must hold for the clashing set $S$ 
	providing the justification of $\stru{\alpha,\ee}$ 
	in $\I_\CAS$.
	Consider then the interpretation $\I'_{\CAS}=\stru{{\I}',\casmap}$,
	corresponding to changing the interpretations of symbols in $\I$
	to the interpretation of ${\I^N}'$: then, $\I'_{\CAS}$ is $\NI$-congruent
	with $\I_{\CAS}$, but $\I'_{\CAS} \not\models S$.
	This contradicts the fact that $\I_{\CAS}$ is justified:
	thus, $\I^N_\CAS$ is justified as well.
	
	Finally, the fact that every $\stru{\alpha,\ee}$ in $\I_{\CAS}^N$
is justified by some clashing set $S$ over $sk(N)$
can be verified by considering that $S$ can be expressed in (a grounding of)
Horn clauses and $\I^N_\CAS \models S$. Thus an equivalent renaming of
constants in $S$ over $sk(N)$ can be provided.
\end{proof}
%
%
The following property is useful in order to prove the correctness
of justifications: the result provides a
characterization of justification 
based on the least model 
$\hat{\I}_\K(\casmap,\nu)$ for a clashing assumption set $\casmap$ and a name
assignment $\nu$.

\begin{theorem}[justified CAS characterization]
\label{theo:justified-cas-char}
Let $\casmap$ be a satisfiable clashing assumptions set for an 
exception safe DKB
$\K$ and name assignment $\nu$. 
Then, $\casmap$ is justified iff $\stru{\alpha,\ee} \in \casmap$ implies
some clashing set $S = S_{\stru{\alpha,\ee}}$ exists such that 
  \begin{enumerate}[label=(\roman*)]
\itemsep=0pt
  \item $\hat{\I} \models \beta$, for each  positive $\beta
    \in S$, where $\hat{\I}_\K(\casmap,\nu)= (\hat{\I},\casmap)$,
and 
\item no CAS-model $\I_{\CAS} = \stru{\I, \casmap}$ with name assignment
  $\nu$ exists s.t. $\I \models \beta$ for some $\neg \beta \in S$.
\end{enumerate}
\end{theorem}
	%
	%
%
In the following sections, we concentrate on reasoning in exception safe DKBs under UNA (on elements of $\K$); 
we will discuss the possible extensions for more general DKBs.



\section{Datalog Translation for $\DLliteR$ DKB}
\label{sec:translation}

We present a datalog translation for reasoning on $\DLliteR$
DKBs which refines the 
translation provided in~\cite{BozzatoES:18}.
The translation provides a reasoning method 
for positive instance queries w.r.t. entailment on DKB-models 
for exception safe DKBs.
%
An important aspect of this translation is that, due to the form of $\DLliteR$ axioms,
no inference on disjunctive negative information 
is needed for the reasoning on derivations of clashing sets. 
Thus, 
reasoning by contradiction using 
``test environments'' is not needed and we can directly encode 
negative reasoning as rules on negative literals: with respect to the discussion in~\cite{BozzatoES:18},
we can say that $\DLliteR$ thus represents an inherently ``justification safe'' fragment
which then allows us to formulate such a direct datalog encoding.
With respect to the interpretation of right-hand side existential axioms, we follow the 
approach of~\cite{Krotzsch:10}: for every axiom
of the kind $\alpha = A \isa \exists R$, an auxiliary abstract individual $aux^\alpha$
is added in the translation to represent the class of all
$R$-successors introduced by $\alpha$.

%
We introduce a \emph{normal form} for axioms of $\DLliteR$
(in Table~\ref{tab:dlr-normalform})
which allows us to simplify the formulation of reasoning rules.
We can provide 
rules to transform any
$\DLliteR$ DKB into normal form and show that the rewritten DKB is equivalent
to the original \comment{(see Lemma~\ref{lem:normal-form} and the discussion
following in Appendix~\ref{sec:nform-appendix})}.
\rewnote{2}{"the rewritten DKB is equivalent to the original" -> when we deal with a non-monotonic semantics, we may distinguish between regular (weak) equivalence, in this case, coincidence of n-de safe extensions, and strong equivalence, that is, that both theories yield the same result no matter the additional context in which they are inserted. Is the current relation a strongly equivalent one (modulo auxiliary symbols)?}
\lbnote{This should be simply what stated on Lemma 4, i.e. they have the same DKB models. 
I added a comment in the authors reply. Any other comments on this?}
In the normalization,
we introduce new concept names $A_{\exists R}$
to simplify the management of existential formulas $\exists R$ in rules for defeasible axioms:
we assume that, 
for every role $R$, axioms $A_{\exists R} \subs \exists R$, $\exists R \subs A_{\exists R}$ are added to 
the DKB.
Note that, with respect to the previous formulation 
of normal form provided in~\cite{BozzatoES:18}, we further simplified
the case for defeasible assertions and negative inclusions, as they can be represented using 
(defeasible) class and role inclusions with auxiliary symbols.
%
\begin{table}[t!]%
\caption{Normal form for $\DKB$ axioms from $\Lcal_\Sigma$}
\label{tab:dlr-normalform}

\vspace{1ex}
\hrule\mbox{}\\[1ex]
\centerline{\small
$\begin{array}{c}
\multicolumn{1}{l}{\text{\textbf{Strict axioms:} for $A,B \in \NC$, $R,S \in \NR$, $a,b \in \NI$:}}\\[1ex]
\begin{array}{c}
  A(a) \qquad  R(a,b) \qquad A  \subs B \qquad  A \subs \non B \qquad \exists R \subs A_{\exists R}\qquad A_{\exists R} \subs \exists R\\[1ex]
  R  \subs S \qquad \mathrm{Dis}(R,S) \qquad \mathrm{Inv}(R,S) \qquad
  \mathrm{Irr}(R) 
\end{array}\\[3.5ex]
\multicolumn{1}{l}{\text{\textbf{Defeasible axioms:} for $A,B \in \NC$, $R,S \in \NR$:}}\\[1ex]
\begin{array}{c}
  \default(A \subs B) \qquad  
	\default(R \subs S) 
	\qquad \mathrm{Inv}(R,S) \qquad \mathrm{Irr}(R) 
\end{array}\\[1.5ex] 
\end{array}$}
\hrule\mbox{}
\end{table}
%
\begin{lemma}
\label{lem:normal-form}
Every DKB $\DKB$ can be transformed in linear time into an
equivalent DKB $\DKB'$ which has modulo auxiliary symbols the same
DKB-models, and such that $n$-de safety and $n$-chain safety are preserved.
\end{lemma}
%

\noindent
\textbf{Translation rules overview.}
We can now present the components of our datalog translation for 
$\DLliteR$ based DKBs.
%
As in the original formulation in~\cite{BozzatoES:14,BozzatoES:18},
which extended the encoding without defeasibility proposed in~\cite{BozzatoSerafini:13}
(inspired by the materialization calculus in~\cite{Krotzsch:10}),
the translation includes sets of \emph{input rules} (which encode DL axioms and signature in datalog),
\emph{deduction rules} (datalog rules providing instance level inference) and \emph{output rules}
(that encode, in terms of a datalog fact, the ABox assertion to be proved).
The translation is composed by the following sets of rules: 

\rewnote{1}{About the encoding presentation, the labels, added in parentheses, of the various rules are not always clear, not clear what they are useful for and seem that they do not add much, at least in the paper (while they look to be more referred in the Appendix).}
\lbnote{Added note in author response: we use them here and refer to these labels in the appendix}

\smallskip\noindent
\emph{$\DLliteR$ input and output rules:}
rules in $I_{dlr}$ encode as datalog facts the $\DLliteR$ axioms
and signature of the input DKB.
For example, in the case of existential axioms,\footnote{Note that, by the normal form above,
this kind of axioms is in the form $A_{\exists R} \subs \exists R$.} 
these are translated \comment{by rule (idlr-supex)} as 
$A \subs \exists R \mapsto \{\supEx(A,R,aux^{\alpha})\}$.
Note that this rule, in the spirit of~\cite{Krotzsch:10}, introduces an auxiliary element
$aux^\alpha$, which intuitively represents the class of
all new $R$-successors generated by the axiom $\alpha$.
Similarly, output rules in $O$ encode in datalog the
ABox assertions to be proved.
These rules are provided in Table~\ref{tab:dlr-rules-tgl}. 

\begin{table}[t!]%
\caption{$\DLliteR$ input, deduction and output rules}
\vspace{1ex}
\hrule\mbox{}\\
\textbf{$\DLliteR$ input translation $I_{dlr}(S)$}\\[.7ex]
\scalebox{.9}{
\small
$\begin{array}[t]{l@{\ \ }l}               
\mbox{(idlr-nom)} 
& a \in \NI \mapsto \{\nom(a)\}\\
\mbox{(idlr-cls)} 
& A \in \NC \mapsto \{\cls(A)\}\\
\mbox{(idlr-rol)} 
& R \in \NR \mapsto \{\rol(R)\}\\[1ex]

\mbox{(idlr-inst)} 
& A(a) \mapsto \{\insta(a,A)\} \\
\mbox{(idlr-triple)} 
& R(a,b) \mapsto \{\triplea(a,R,b)\} \\[1ex]

\mbox{(idlr-subc)} 
& A \subs B \mapsto \{\subClass(A,B)\}\\

\end{array}$
\;
$\begin{array}[t]{l@{\ \ }l}               

\mbox{(idlr-subex)} 
& \exists R \subs B \mapsto \{\subEx(R,B)\} \\

\mbox{(idlr-supex)} 
&  A \subs \exists R \mapsto \{\supEx(A,R,aux^{\alpha})\}\\[2ex]
        
\mbox{(idlr-subr)} 
& R \subs S \mapsto \{\subRole(R,S)\}\\
\mbox{(idlr-dis)} & \mathrm{Dis}(R,S)  \mapsto \{\pDis(R,S)\}\\
\mbox{(idlr-inv)} & \mathrm{Inv}(R,S) \mapsto \{\pInv(R,S)\}\\
\mbox{(idlr-irr)} & \mathrm{Irr}(R) \mapsto \{\pIrr(R)\}\\[1ex]
\end{array}$}\\[1ex]
\textbf{$\DLliteR$ deduction rules $P_{dlr}$}\\[.7ex]
\scalebox{.85}{
\small
$\begin{array}{l@{\;}r@{\ }r@{\ }l@{}}
	 \mbox{(pdlr-instd)} & \instd(x,z) & \rif & \insta(x,z).\\
	 \mbox{(pdlr-tripled)} & \tripled(x,r,y) & \rif & \triplea(x,r,y).\\[.5ex]

   
		
   \mbox{(pdlr-subc)} 
   &      \instd(x,z) & \rif & \subClass(y,z), \instd(x,y). \\
   \mbox{(pdlr-supnot}) 
   & \non\instd(x,z) & \rif & \supNot(y,z), \instd(x,y). \\
   \mbox{(pdlr-subex)} 
   & \instd(x,z) & \rif & \subEx(v,z), \tripled(x,v,x'). \\
   \mbox{(pdlr-supex)} 
   & \tripled(x,r,x') & \rif & \supEx(y,r,x'), \instd(x,y).\\[.5ex]


   \mbox{(pdlr-subr)} 
   & \tripled(x,w,x') & \rif & \subRole(v,w), \tripled(x,v,x'). \\
   \mbox{(pdlr-dis1)} & \non\tripled(x,u,y) & \rif & \pDis(u,v), \tripled(x,v,y).\\	
	 \mbox{(pdlr-dis2)} & \non\tripled(x,v,y) & \rif & \pDis(u,v), \tripled(x,u,y).\\	
   \mbox{(pdlr-inv1)} & \tripled(y,v,x) & \rif & \pInv(u,v), \tripled(x,u,y). \\
   \mbox{(pdlr-inv2)} & \tripled(y,u,x) & \rif & \pInv(u,v), \tripled(x,v,y). \\
   \mbox{(pdlr-irr)} & \non\tripled(x,u,x) & \rif & \pIrr(u), \const(x).\\[.5ex]


   \mbox{(pdlr-nsubc)} 
   &    \non\instd(x,y) & \rif & \subClass(y,z), \non\instd(x,z). \\
   
   \mbox{(pdlr-nsupnot}) 
   & \instd(x,y) & \rif & \supNot(y,z), \non\instd(x,z).\\
                          
   \mbox{(pdlr-nsubex)} 
   & \non\tripled(x,v,x') & \rif & \subEx(v,z), \const(x'), \non\instd(x,z).\\

   \mbox{(pdlr-nsupex)} 
	  & \non\instd(x,y) & \rif & \supEx(y,r,w), \const(x), 
		                          \AllNRel(x,r).\\[.5ex]	

   \mbox{(pdlr-nsubr)} 
   & \non\tripled(x,v,x') & \rif & \subRole(v,w), \non\tripled(x,w,x'). \\	 

   \mbox{(pdlr-ninv1)} & \non\tripled(y,v,x) & \rif & \pInv(u,v), \non\tripled(x,u,y). \\
   \mbox{(pdlr-ninv2)} & \non\tripled(y,u,x) & \rif & \pInv(u,v), \non\tripled(x,v,y). \\[1ex]	
	
		\mbox{(pdlr-allnrel1)} & \AllNRelStep(x,r,y) & \rif & \first(y), \non \tripled(x,r,y).\\
	  \mbox{(pdlr-allnrel2)} & \AllNRelStep(x,r,y) & \rif & \AllNRelStep(x,r,y'), \nextp(y',y),
		                                               \non \tripled(x,r,y).\\
	  \mbox{(pdlr-allnrel3)} & \AllNRel(x,r) & \rif & \lastp(y), \AllNRelStep(x,r,y).\\ 
  \end{array}$}\\[1ex]
\textbf{Output translation $O(\alpha)$}\\[.7ex]	
\scalebox{.9}{
\small
$\begin{array}{l@{\ \ }l}
\mbox{(o-concept)} & A(a) \mapsto \{A(a)\} \\
\mbox{(o-role)} & R(a,b) \mapsto \{R(a,b)\} \\[1ex]
\end{array}$}\\[.5ex]
\hrule\mbox{}
\label{tab:dlr-rules-tgl}
\vspace{-2.5ex}
\end{table}



\smallskip\noindent
\emph{$\DLliteR$ deduction rules:}
rules in $P_{dlr}$ (in Table~\ref{tab:dlr-rules-tgl}) add deduction rules for ABox reasoning.
In the case of existential axioms,
the rule \comment{(pdlr-supex)} introduces a new relation to the auxiliary individual as follows:
\begin{center}\small
   $\tripled(x,r,x')  \rif  \supEx(y,r,x'), \instd(x,y).$
\end{center}
In this translation the reasoning on negative information
is directly encoded by ``contrapositive'' versions of the rules.
For example, with respect to previous rule, we have the negative rule \mbox{\comment{(pdlr-nsupex)}}:
\begin{center}\small
	$\non\instd(x,y) \rif \supEx(y,r,w), \const(x), \AllNRel(x,r).$
\end{center}
where $\AllNRel(x,r)$ verifies that $\non \triple(x,r,y)$ holds for all $\const(y)$
by an iteration over all constants.


\smallskip\noindent
\emph{Defeasible axioms input translations}: 
the set of input rules $I_\default$ 
(shown in Table~\ref{tab:input-default-tgl}) 
provides the translation of defeasible axioms $\default(\alpha)$ in the DKB: in other words,
they are used to specify that the axiom $\alpha$ needs to be considered as
defeasible. 
For example, $\default(A \isa B)$ is translated to 
$\defsubs(A, B)$.
Note that, by the definition of the normal form, the 
existential axioms are ``compiled out'' from defeasible axioms
(i.e., defeasible existential axioms can be expressed by using 
the newly added $A_{\exists R}$ concepts).

\begin{table}[t]%
\caption{Input and deduction rules for defeasible axioms}
\label{tab:input-default-tgl}
\vspace{1ex}
\hrule\mbox{}\\[1ex]
\textbf{Input rules for defeasible axioms $I_{\default}(S)$}\\[.7ex]	
\small
\scalebox{1}{
$\begin{array}{@{}l@{~}r@{~}l@{}}

 \mbox{(id-subc)} & \default(A \subs B)  & \mapsto \{\, \defsubs(A,B).\,\} \\  

 \mbox{(id-subr)} & \default(R \subs S)  & \mapsto \{\, \defsubr(R, S).\,\} \\[.5ex]   
 \mbox{(id-inv)} & \default(\mathrm{Inv}(R,S))  & \mapsto \{\, \definv(R,S).\,\} \\   
 \mbox{(id-irr)} & \default(\mathrm{Irr}(R))  & \mapsto \{\, \defirr(R).\,\}\\[1.5ex]
\end{array}$}\\

\textbf{Deduction rules for defeasible axioms $P_{\default}$: overriding rules}\\[.7ex]
\scalebox{1}{
$\begin{array}{l@{\ \ }r@{\ \ }l}
 \mbox{(ovr-subc)} &
 \ovr(\subClass,x,y,z) \rif &
 \defsubs(y,z), \instd(x,y), \non \instd(x,z). \\
													

  \mbox{(ovr-subr)}  &
  \ovr(\subRole,x,y,r,s) \rif &
  \defsubr(r,s), \tripled(x,r,y),
  \non \tripled(x,s,y).\\
  \mbox{(ovr-inv1)}  & 
	\ovr(\pInv,x,y,r,s) \rif &
  \definv(r,s), \tripled(x,r,y),
	\non \tripled(y,s,x).\\	
  \mbox{(ovr-inv2)} & 
	\ovr(\pInv,x,y,r,s) \rif &
  \definv(r,s), \tripled(y,s,x), 
	\non \tripled(x,r,y).\\
  \mbox{(ovr-irr)}  & 
  \ovr(\pIrr,x,r) \rif &
  \defirr(r), \tripled(x,r,x).\\[1.5ex]
\end{array}$}\\

\textbf{Deduction rules for defeasible axioms $P_{\default}$: application rules}\\[.7ex]
\scalebox{1}{
\small
$\begin{array}{l@{\;}r@{\ }r@{\ }l@{}}
   \mbox{(app-subc)} 
   &   \instd(x,z) & \rif 
	 & \defsubs(y,z), \instd(x,y), \naf \ovr(\subClass,x,y,z).\\

													

   \mbox{(app-subr)} 
   & \tripled(x,w,y) & \rif & \defsubr(v,w), \tripled(x,v,y), \naf \ovr(\subRole,x,y,v,w). \\
   \mbox{(app-inv1)} 
   & \tripled(y,v,x) & \rif & \definv(u,v), \tripled(x,u,y), \naf \ovr(\pInv,x,y,u,v).\\
   \mbox{(app-inv2)} 
   & \tripled(x,u,y) & \rif & \definv(u,v), \tripled(y,v,x), \naf \ovr(\pInv,x,y,u,v).\\       

 \mbox{(app-irr)} 
   & \non\tripled(x,u,x) & \rif & \defirr(u), \const(x), \naf \ovr(\pIrr,x,u). \\[.5ex]
		

   \mbox{(app-nsubc)} 
   & \non\instd(x,y) & \rif 
	 & \defsubs(y,z), \non\instd(x,z), \naf \ovr(\subClass,x,y,z).\\

													
												
   \mbox{(app-nsubr)} 
   & \non\tripled(x,v,y) & \rif & \defsubr(v,w), \non\tripled(x,w,y), \naf \ovr(\subRole,x,y,v,w). \\
   \mbox{(app-ninv1)} 
   & \non\tripled(y,v,x) & \rif & \definv(u,v), \non\tripled(x,u,y), \naf \ovr(\pInv,x,y,u,v).\\
   \mbox{(app-ninv2)} 
   & \non\tripled(x,u,y) & \rif & \definv(u,v), \non\tripled(y,v,x), \naf \ovr(\pInv,x,y,u,v).\\[1.5ex]	
	\end{array}$}
 \hrule
\vspace{-2.5ex}
\end{table}
	
\smallskip\noindent
\emph{Overriding rules:}
rules for defeasible axioms provide the different
conditions for the correct interpretation of defeasibility: the overriding rules
define conditions (corresponding to clashing sets) for 
recognizing an exceptional instance.
For example, for axioms of the form $\default(A \isa B)$,
the translation introduces the rule \comment{(ovr-subc)}:
\begin{center}\small
  $\ovr(\subClass,x,y,z) \rif \defsubs(y,z), \instd(x,y), \non \instd(x,z).$
\end{center}
Note that in this version of the calculus, 
the reasoning on 
negative information (of the clashing sets) is directly encoded in the deduction
rules.
Overriding rules in $P_\default$ are shown in Table~\ref{tab:input-default-tgl}.

\smallskip\noindent
\emph{Defeasible application rules:}
another set of rules in $P_\default$ defines the defeasible application of such axioms:
intuitively, defeasible axioms are applied only to instances that have
not been recognized as exceptional.
For example, the rule \comment{(app-subc)} applies a defeasible
concept inclusion $\default(A \isa B)$:
\begin{center}\small
  $\instd(x,z) \rif \defsubs(y,z), \instd(x,y), \naf \ovr(\subClass,x,y,z).$
\end{center}
Defeasible application rules are provided in Table~\ref{tab:input-default-tgl}.

\smallskip\noindent
\textbf{Translation process.}
Given a DKB $\K$ in $\DLliteR$ normal form, 
a program $PK(\K)$ that encodes query answering for $\K$ is obtained as:

\smallskip

\centerline{$PK(\K) = P_{dlr} \cup P_{\default} \cup 
     I_{dlr}(\K) \cup I_{\default}(\K)$}

\smallskip

\noindent 
Moreover, $PK(\K)$ is completed with a set of supporting facts
about constants:
for every literal $\nom(c)$ or $\supEx(a,r,c)$ 
in $PK(\K)$, $\const(c)$ is added to $PK(\K)$. Then, given an arbitrary enumeration $c_0, \dots, c_n$ 
s.t. each $\const(c_i) \in PK(\K)$, the facts $\first(c_0), \lastp(c_n)$
and $\nextp(c_i, c_{i+1})$ with 
$0 \leq i < n$ 
are added to $PK(\K)$.
Query answering $\K \models \alpha$ is then obtained 
by testing whether the (instance) query, translated to
datalog by $O(\alpha)$, is a consequence of $PK(\K)$, 
i.e., whether $PK(\K) \models O(\alpha)$ holds. 

\rewnote{2}{if I understood correctly the ASP translation, it seems to assume that the elements in the domain are ordered by some linear relation represented by first(y) and next(y,y'). I wonder how this is achieved in practice: I would like more details about this. On the other hand, ASP allows nowadays the use of aggregates and conditional literals to perform quantification without the need of resorting to an artificial ordering of the domain elements.}
\comment{Note that we use a linear ordering of constants in an encoding
by means of the predicates $\first, \lastp$
and $\nextp$, which allows us to verify universal sentences over all constants
(in our case, negation on roles), by walking through them starting at
the first constant over the next one until the last constant is
reached. We note that verifying universal sentences can also 
accomplished by means of aggregates in ASP~\cite{DBLP:journals/ki/AlvianoF18}: however, we chose to 
use this simpler method in order to keep the standard interpretation of ASP programs.}


\smallskip\noindent
\textbf{Correctness.}
The presented translation procedure provides
a sound and complete materialization calculus
for instance checking on $\DLliteR$ DKBs in normal form.


As in~\cite{BozzatoES:18}, 
the proof for this result can be verified by
establishing a correspondence between minimal justified models of $\K$
and answer sets of $PK(\K)$.
Besides the simpler structure of the final program,
the proof is simplified by the direct formulation of rules for 
negative reasoning.
Another new aspect of the proof in the case of $\DLliteR$
resides in the management of existential axioms,
since there is the need to define a correspondence between the auxiliary individuals
in the translation and the interpretation of existential axioms in the semantics:
we follow the approach of 
\citeN{Krotzsch:10}, where
in building the correspondence with justified models, auxiliary constants $aux^\alpha$ are mapped to
the class of Skolem individuals for existential axioms $\alpha$.
We remark that this collective encoding of unnamed individuals is possible since,
in the case of exception safe DKBs, no exceptions can appear on such individuals:
thus, differently from named individuals (which need to single out exceptional elements
of the domain), there is no need to identify single unnamed elements of the domain.

As in~\cite{BozzatoES:18}, in our translation
we consider UNA on elements of $\K$
and \emph{named models}, i.e., interpretations restricted to $sk(N_\K)$. 
Thus, we can show the correctness result on 
the least model for $\K$ with respect to a 
set of clashing assumptions $\casmap$, that will
be denoted by $\hat{\I}(\casmap)$.
\rewnote{1}{"i.e.," vs. "i.e." before. Please unify. Same for "e.g." vs. "e.g.,"}
\lbnote{fixed to "i.e.," and "e.g.,"}

Let $\I_\CAS = \stru{\I, \casmap}$ be a justified named CAS-model. We define the set of overriding assumptions
$\OVR(\I_\CAS) = \{\, \ovr(p(\ee)) \;|\; \stru{\alpha, \ee} \in \casmap,\, I_{dlr}(\alpha) = p \,\}$.
%
Given a CAS-interpretation $\I_{\CAS}$, 
we define a corresponding 
interpretation 
$I(\I_{\CAS})$ for $PK(\K)$
by including the following atoms in it:
\begin{enumerate}[label=(\arabic*)]
\itemsep=0pt
\item 
all facts of $PK(\K)$;
\item 
 $\instd(a,A)$, if $\I \models A(a)$ and
 $\non\instd(a,A)$, if $\I \models \non A(a)$; 
\item 
  $\tripled(a,R,b)$, if  $\I \models R(a,b)$ and
  $\non\tripled(a,R,b)$, if  $\I \models \non R(a,b)$;	
\item 
  $\tripled(a,R,aux^\alpha)$, if $\I \models \exists R(a)$
	for $\alpha = A \isa \exists R$;
\item
	$\AllNRel(a,R)$
	if $\I \models \non \exists R(a)$;
\item
  each $\ovr$-literal from $\OVR(\I_{\CAS})$;
\end{enumerate}
%
%
The next proposition shows that the least models of $\K$
can be represented by the answer sets of the program $PK(\K)$.

\begin{proposition}
\label{prop:correctness}
Let  $\K$ be an exception safe DKB in $\DLliteR$ normal form. Then:
%
	\begin{enumerate}[label=(\roman*)]
	\item 
	  for every (named) justified clashing assumption $\casmap$, 
		the interpretation $S = I(\hat{\I}(\casmap))$ is an answer set of $PK(\K)$;
	\item
	   every answer set $S$ of $PK(\K)$ is of the form 
		$S = I(\hat{\I}(\casmap))$ where $\casmap$ is a (named) justified clashing assumption for $\K$.
	\end{enumerate}
\end{proposition}

\noindent
The correctness of the translation with respect to instance checking
is obtained as a direct consequence of
Proposition~\ref{prop:correctness}.
\begin{theorem}
\label{thm:encode-c-entailment}
Let  $\K$ be an exception safe DKB in $\DLliteR$ normal form, and let 
$\alpha \in \Lcal_\Sigma$ such that the output translation $O(\alpha)$ is defined. 
Then, $\K \models \alpha$ iff $PK(\K) \models O(\alpha)$.
\end{theorem}


\smallskip\noindent
\textbf{Prototype implementation.} 
A proof-of-concept implementation of the presented datalog translation for $\DLliteR$ DKBs 
has been included in 
the latest version 
of the \CKRew \emph{(CKR datalog rewriter)} 
prototype~\cite{BozzatoES:18}.
\CKRew is a Java-based command line application that accepts as input
RDF files representing
(contextualized) knowledge bases with defeasible axioms and
produces as output a single {\tt .dlv} text file with
the datalog rewriting for the input KB.
The current version of the prototype includes an option
to accept as input a single RDF file containing a $\DLliteR$ DKB
(represented as OWL axioms in the normal form of Table~\ref{tab:dlr-normalform})
and apply the datalog translation presented above. 

The latest version of \CKRew,
together with sample RDF files implementing the 
knowledge base of Example~\ref{ex:syntax},
is available on-line at: \url{http://ckrew.fbk.eu/}.


\section{Complexity of Reasoning Problems}
\label{sec:complexity}


In this section, we turn to the computational complexity of reasoning
from a DKB. As in the previous section, we shall pay special
attention to adopting the UNA on knowledge bases, in particular when
we consider lower complexity bounds. As UNA on DKBs is easy to express
in $\DLliteR$, the results will carry over to the case without assumptions.

\subsection{Satisfiability} 
\label{sec:sat}

We first consider the satisfiability problem, i.e., deciding whether a
given $\DLliteR$ DKB has some DKB-model. As it turns out, defeasible
axioms do not increase the complexity with respect to satisfiability
of $\DLliteR$, due to the following property. 

\begin{proposition}
\label{prop:DKB-existence}
Let $\DKB$ be a normalized $\DLliteR$ DKB, and let 
$\casmap' = \{ \stru{\alpha,\ee} \mid D(\alpha) \in \DKB$, $\ee$ is over 
standard names$\,\}$ be the clashing assumption with all exceptions possible. 
Then, $\DKB$ has some 
justified CAS-model $\I_{\CAS} = \stru{\I, \casmap}$ such that
$\casmap\subseteq \casmap'$ iff
$\DKB$ has some CAS-model $\I_{\CAS} = \stru{\I, \casmap'}$.
\end{proposition}
That is, a DKB $\K$ has a DKB-model iff the 
$\DLliteR$ KB consisting of the non-defeasible axioms in $\K$ has a model.
We note that in the argument for this proposition, no particular 
NI-congruence is considered. Conditions such as UNA or other equivalence relations over the individuals in $\K$
can be accommodated (using Horn axioms).


Thus, DKB-satisfiability testing with arbitrary exceptions boils down to
testing whether $\K$ is satisfiable if all defeasible axioms are
dropped, which is tractable. 

\begin{theorem}
\label{theo:DKB-sat}
Deciding whether a given arbitrary $\DLliteR$ DKB $\DKB$ has some DKB-model is
$\nlogspace$-complete in combined complexity and FO-rewritable 
in data complexity.
\end{theorem}%
\begin{proof}
\nop{******* old text 
\tenote{add formal proof}
To see this, the program $PK(\DKB)$ for $\DKB$ has in each rule at most one
literal with an intensional predicate in the body, i.e., a predicate
that is defined by proper rules. Thus, we have a linear datalog
program with bounded predicate arity, for which derivability of an atom
is feasible in nondeterministic logspace, as this can be reduced to 
a graph reachability problem in logarithmic space. The \nlogspace-hardness is inherited from the combined complexity of KB satisfiability
in $\DLliteR$, which is $\nlogspace$-complete.
*******}
We can normalize $\DKB$ efficiently in linear time (and in fact
logspace) while preserving exception safety, so we may assume $\DKB$ is of this form. We then can test
whether $\DKB$ with defeasible axioms dropped, which is an ordinary
$\DLliteR$ KB, is satisfiable; it is well-known that this is
feasible in \nlogspace\ \cite{CalvaneseGLLR07}. The \nlogspace-hardness
is inherited from the combined complexity of KB satisfiability in $\DLliteR$, which is $\nlogspace$-complete.
%
\nop{**** old text 
As regards data-complexity, it is
well-known that instance checking and similarly satisfiability testing for $\DLliteR$
are FO-rewritable \cite{CalvaneseGLLR07};
this has been shown by a reformulation algorithm, which informally
unfolds the axioms $\alpha(\vec{x})$ (i.e., performs resolution
viewing axioms as clauses), such that deriving an instance $A(a)$
reduces to presence of certain assertions in the ABox. This unfolding
can be adorned by typing each argument $x\,{\in}\,\vec{x}$ of an axiom to whether it is
an individual from the DKB  (type i), or an unnamed individual (type u); for
example, $\alpha(x) = A \isa B$ yields $\alpha_{\rm i}(x)$ and
$\alpha_{\rm u}(x)$. The typing carries over to unfolded
axioms.  In unfolding, one omits typed versions of defeasible axioms
$D(\alpha(\vec{x}))$, which w.l.o.g.\
have no existential restrictions; e.g., for $D(\alpha(x)) = D(B \isa
C)$, one omits $\alpha_{\rm i}(x)$. In this way, instance derivation 
(and similarly satisfiability testing) is reduced to
presence of certain ABox assertions again.
***}

As regards data-complexity, it is
well-known that instance checking and 
satisfiability testing for $\DLliteR$
are FO-rewritable \cite{CalvaneseGLLR07};
this has been shown by a reformulation algorithm, which informally
unfolds the axioms $\alpha(\vec{x})$ (i.e., performs resolution
viewing axioms as clauses), such that deriving an instance $A(a)$
reduces to presence of certain assertions in the ABox. We can use the
same rewriting and apply it to $\K$ with all defeasible axioms
dropped. 
\end{proof}%
\noindent
We note that while satisfiability is tractable for arbitrary
DKB-models in general, this does not necessarily hold under
restrictions on exceptions, as the construction in the proof of 
Proposition~\ref{prop:DKB-existence} depends on the enumeration; in
particular, deciding the existence of some  DKB-model with no exceptions 
involving unnamed individuals (i.e., of a 0-bounded justified DKB-model) is
intractable; this can be shown, e.g.,\ by an adaption of an \np-hardness
proof for 0-bounded justified model existence for defeasible $\ELbot$
context knowledge repositories in \cite{DBLP:conf/birthday/BozzatoES19}.
On the other hand, under a condition that ensures that some DKB-model
is 0-bounded if any DKB-model exists, we retain tractability. In particular: 

\begin{corollary}
\label{cor:DKB-sat}
Deciding whether a given 
exception-safe  DKB $\DKB$ has some
DKB-model is $\nlogspace$-complete in combined complexity and FO-rewritable 
in data complexity.
\end{corollary}
%
%
In passing, we remark that for exception-safe DKBs $\K$, 
checking whether an interpretation $\I$ is a DKB-model of $\K$ is tractable 
(as follows from the ASP encoding), as is constructing some arbitrary
DKB-model; however, we focus in the sequel here on inference.

\subsection{Entailment Checking} 
\label{sec:entailment}

As regards inference, entailment checking of axioms from the DKB-models
of an exception-safe DKB is intractable:
there can be exponentially many justified clashing assumptions for
such models, even under UNA; finding a DKB-model that violates an
axiom turns out to be difficult.

\begin{theorem}
\label{theo:DKB-entail-conp}
Given an exception-safe DKB $\DKB$  and an axiom $\alpha$,  deciding
whether $\DKB\models\alpha$ is \conp-complete; this holds also for data
complexity and instance checking, i.e., $\alpha$ is an assertion of the form $A(a)$.
\end{theorem}
\begin{proof}[Proof (Sketch)]
To refute $\DKB\models \alpha$,
we need to show that
some justified CAS-model  $\IC_{\CAS} = \stru{\I,\casmap}$ of $\DKB$ 
exists such that $\I \not\models \alpha$. Without loss of generality,
we assume that $\alpha$ is normalized.

Given $\casmap$ and a name assignment $\nu$, we can prove the refutation depending on the type of $\alpha$.
For example, if
%
$\alpha$ is an inclusion axiom $A \isa B$, then we need to
show that for some element $e$ it holds that $\I \models A(e)$ and
$\I\models \neg B(e)$. To deal with this, we first 
incorporate $\nu$ into $\K$, by pushing (in)equalities
w.r.t.\ $\nu$ (replace all equal constants by one representative, add
axioms that enforce inequalities $a\neq b$, e.g.,\ stating $A_a(a)$,
$\neg A_b(b)$ where $A_a$ and $A_b$ are fresh concept names). We then
add to $\K$ the
axioms $Aux \isa A$, $Aux \isa \non B$. 
We may then assume without loss of generality that $\I \models Aux(e)$, i.e., $\I \not\models \neg Aux(e)$.
We next add to $\K$ an assertion $A_e(a_e)$, where $A_e$ and $a_e$ are a
fresh concept and individual name, respectively; this serves to
give $e$ a name if it is outside the elements named in $\I$ by Skolem terms.
%
We then check whether $\neg Aux(a_e)$ is not derivable from the
resulting DKB $\K'$ under $\casmap$; this holds iff some $\I$ with $e$
not named by some Skolem term of $\K$ exists. 
Otherwise, $e$ must be named by some Skolem term $t$ of $\K$. We thus
check that for none such $t$, $\neg Aux(t)$ is derivable from $\K'$
under $\casmap$; the depth of $t$ can be polynomially bounded.
The checks can be done in non-deterministic logspace, and thus deciding
$\K \not\models \alpha$ under $\casmap$ is feasible in
polynomial time. 
The cases for other forms of $\alpha$ can be shown similarly
and are described in the full proof in Appendix~\ref{sec:complexity-appendix}.

  %
%

Thus, to decide $\K\not\models \alpha$, we can guess a justified
clashing assumption $\casmap$ over $N_\K$ 
together with a clashing set $S_{\stru{\alpha,\ee}}$ for each
$\stru{\alpha,\ee} \in \casmap$
for a name assignment $\nu$.
We then check  relative to $\nu$ (i) that $\casmap$ is satisfiable, (ii) that all
$S_{\stru{\alpha,\ee}}$
are derivable from $\DKB$ under $\casmap$ and $\nu$, and (iii) that 
$\DKB \not\models \alpha$. Each of the steps (i)--(iii) is feasible in polynomial time. Consequently, 
the entailment problem $\K\models \alpha$ is in \conp. 

The \conp-hardness can be shown by a reduction from 
inconsistency-tolerant reasoning from $\DLliteR$
KBs under AR-semantics \cite{DBLP:conf/rr/LemboLRRS10}. 
The result for $\DLliteR$ KBs can be easily extended to exception safe
DKBs and to data complexity (further details are provided
in the full proof in Appendix~\ref{sec:complexity-appendix}).
\end{proof}
\noindent
We observe that the \conp-hardness proof in \cite{DBLP:conf/rr/LemboLRRS10} 
used many role restrictions and inverse roles; for 
combined complexity, \conp-hardness of entailment 
in absence of any role names can be derived from  results about
propositional circumscription. 

\begin{proposition}
\label{prop:DKB-entail-conp}
Given a DKB $\DKB$, deciding whether $\DKB\models\alpha$ is \conp-hard
even if no roles occur in $\DKB$ and $\alpha$ is 
an assertion $A(a)$.
\end{proposition}
While the proof of Proposition~\ref{prop:DKB-entail-conp} establishes \conp-hardness of entailment for combined
complexity under UNA when roles are absent (and for the case
without UNA as well), this setting has
tractable data complexity: 
we can consider the axioms for individuals $a$ separately, and 
if the GCI axioms are fixed only few
axioms per individual exist. This also holds 
if role axioms but no
existential restrictions are
permitted, as we can concentrate on
the pairs $a,b$ and $b,a$ of individuals.
The question remains how much of the latter is possible while staying tractable.

\subsection{Conjunctive Query Answering} 
\label{sec:cq-answering}

A \emph{conjunctive query (CQ)} is a formula $Q(\vec{x})
= \exists\vec{y}.\gamma(\vec{x},\vec{y})$ where $\vec{x}$,
$\vec{y}$ are disjoint
lists of different variables and $\gamma(\vec{x},\vec{y})
= \gamma_1\land\cdots\land\gamma_m$ is a
conjunction of atoms $\gamma_i = \alpha_i(\vec{t}_i)$, $1\leq i
\leq m$ where $\alpha_i$ is either a
concept name or a role name and $\vec{t_i}$ is a tuple of variables
from $\vec{x} \cup\vec{y}$ and individual constants that matches the
arity of $\alpha_i$. The CQ is \emph{Boolean} (a \emph{BCQ}), if $\vec{x}$ is empty.

A CAS-interpretation $\I_\CAS = \stru{\I,\casmap}$ satisfies a
BCQ $Q$, denoted $\I_\CAS \models Q$, if a query matches, i.e., some substitution
$\vartheta: \vec{y} \rightarrow \NI_s$ exists such that $\I \models
\alpha_i(\vec{t}_i\vartheta$) for all $i=1,\ldots,m$.
A DKB $\DKB$ entails $Q$, denoted $\DKB\models Q$, if every DKB-model
of $\DKB$ entails $Q$. The \emph{(certain) answers}\/ for a general CQ
    $Q(\vec{x})$ are
then as usual the tuples $\vec{c}$ of individual names such
that $\DKB\models Q(\vec{c})$.


\begin{example}[Example~\ref{ex:2} cont'd]
\label{ex:bcq}
Consider
the CQ $Q(x) = \exists y.\mi{DeptMember}(x) \land \mi{hasCourse}
(x,y)$ on the DKB $\K_{dept}$ in Example~\ref{ex:syntax}.
In Example~\ref{ex:2}, we discussed that $\K_{dept}$ has a
justified CAS-model $\I_{\CAS} = \stru{\Ical,\casmap}$ with an
exception for $\mi{bob}$ on the axiom $\alpha = \mi{DeptMember} \subs
\exists \mi{hasCourse}$, while $\mi{alice}$ has no exception. Thus,
the query $Q$ has a match in $\I_{\CAS}$ by $\vartheta = \{ x\mapsto
\mi{alice}^\I,$ $ y \mapsto f_{\mi{hasCourse}}^\I(x)\}$, where
$f_{\mi{hasCourse}}^\I$ is the Skolem function in $\I$; in fact, in
every such justified CAS-model $\I_{\CAS}$ the query has this match.
If $\mi{alice}$ and $\mi{bob}$ are regarded different, i.e., under the
unique name assumption, no other justified CAS-model exists; thus
$\mi{alice}$ is the (only) certain answer of the query. \EndEx
\end{example}
Deciding whether a BCQ $Q$ has a query match in a
$\DLliteR$ KB is known to be \np-complete, cf.\ \cite{CalvaneseGLLR07}. 
As multiple (even exponentially many) clashing assumptions may lead
  to different DKB-models, and as for  each such assumption the query must have a match,
BCQ answering from exception-safe DKBs is at the second level of the polynomial hierarchy.

\begin{theorem}
\label{theo:CKR-CQ-answering}
Given an exception-safe DKB $\DKB$ and a
Boolean CQ $Q$, deciding whether $\DKB \models Q$ is
(i) $\Pi^p_2$-complete in combined complexity and (ii)
\conp-complete in data complexity.
\end{theorem}
\begin{proof}[Proof (Sketch)]
To start with (i), as for  membership in $\Pi^p_2$, to refute $Q$ we
can guess for a justified
CAS-model $\I_{\CAS} = \stru{\I, \casmap}$ such that
$\I_{\CAS} \not\models Q$ 
the clashing assumption $\casmap$ on $N_\DKB$ and a name
assignment $\nu$.
Since $\DKB$ is exception safe, we can decide
in \nlogspace\ whether $\casmap$ is satisfiable relative to $\nu$
using the entailment method for positive and negative assertions in 
the proof of Theorem~\ref{theo:DKB-entail-conp}; note that $\nu$ can be
pushed to $\DKB$
and indeed can give rise to a desired justified CAS-model $\I_{\CAS}$ of $\DKB$. We then can use an
\np\ oracle to check whether for some
polynomial number of Skolem terms $ST$, 
where the number depends on $Q$ and $\DKB$,
the query has a match on $N_\K \cup ST$ in a least CAS-model
$\I_{\CAS} = \stru{\I, \casmap}$ of $\K$;
to this end, each atom $A(t)$ resp.\ $R(t,t')$ in the match must be 
derived by applying the axioms (that is, by unraveling $\I_{\CAS}^{N_\K}$); 
this will ensure that a match exists
in each CAS-model $\I_{\CAS} = \stru{\I, \casmap}$ of $\K$.
If the oracle answer is no, then some $\I_{\CAS}$ such that
$\I_{\CAS}\not\models Q$ exists. 
Consequently, refuting $\DKB \models Q$ is in $\Sigma^p_2$, which proves the membership part.


The $\Pi^p_2$-hardness of (i) is shown by a reduction from a
generalization of deciding whether a graph is 3-colorable: given an
(undirected) graph $G=(V,E)$, can every color assignment to the nodes of degree 1
in $G$ (i.e., source nodes) be extended to a 3-coloring of $G$? This
problem is $\Pi^p_2$-complete (see Lemma~\ref{lem:sink-3col} in Appendix~\ref{sec:complexity-appendix}).
The construction for such reduction is provided in the
full proof in Appendix~\ref{sec:complexity-appendix}.

(ii) As for data complexity, we note that the check where 
$Q$ has no match in any $\I_{\CAS}''$ is feasible in polynomial time,
as the number of variables in the query is fixed and thus only
constantly many Skolem terms $ST$ have to be added to
$N_\K$ for a query match in a least CAS-model $\I_{\CAS} = \stru{\I,
\casmap}$ of $\K$, for which only polynomially many possibilities
exist; furthermore, the inference of atoms $A(t)$ resp.\ $R(t,t')$ is
feasible in polynomial time. Hence, the
problem is in \conp. The \conp-hardness follows from Theorem~\ref{theo:DKB-entail-conp}.
\end{proof}

\section{Reasoning on Unnamed Individuals}
\label{sec:unnamed}

%
In the sections above, we have concentrated on exception safe DKBs,
where no exceptions on unnamed individuals are possible. However, 
this is not a real limitation in principle, as unnamed individuals may
be named. 
Specifically, we note the following property of $n$-bounded CAS
models (recall from Definition~\ref{defn:n-bounded} that 
$\mi{uni}_{\K}(\I_{\CAS})$ are the domain elements in clashing assumptions not
named by individuals in $\K$):

\begin{proposition}
\label{prop:name-exceptions}
Let $\I_{\CAS} = \stru{\I, \casmap}$ be a CAS-model of a DKB $\K$ such
that $\mi{uni}_{\K}(\I_{\CAS}) = \{ e_1,\ldots, e_m \}$. 
Let $c_1,\ldots,c_m$ be fresh individual names, and $A$ be a fresh concept. Then, $\I'_{\CAS} =
\stru{\I', \casmap}$ where $c^{\I'}_i = e_i$, $i=1,\ldots,m$ and
$A^{\I'} = \{ e_1,\ldots, e_m\}$ is a CAS-model of $\K' = \K \cup \{
A(c_1), \ldots, A(c_m)\}$.
\end{proposition}
That is, we can name unnamed individuals in clashing assumptions, and
in this way turn an $n$-bounded CAS-model into a $0$-bounded CAS-model. 
In particular, if $n$ is polynomial in the size of $\K$
we can do this with polynomial overhead.

%

However, when we reason from a DKB $\DKB$ under a (named)
clashing assumption $\casmap$, the issue rises whether for a Skolem
term $f(t)$ and a clashing assumption $\stru{\alpha,e}$, where $e$ is
named by an individual $a$, say, the exception is applicable (if $f(t)
= a$) or not (if $f(t)\neq a$) in a model. This in fact complicates
reasoning, and to decide whether a given partial CAS-interpretation
$\I_{\CAS} = \stru{\Ical,\casmap}$
where all exceptions are individuals in $\K$ can be extended to some
justified CAS-model of $\DKB$ is no longer easy to accomplish. If
$\DKB$ is $n$-de safe, we have to consider such terms $f_1(t_1)$,
\ldots, $f_m(t_m)$ for $m\leq n$ and possibly collapse them with some individuals in
$\DKB$. This leads to an exponential explosion,  
even if $n = O(|\K|^k)$ is 
polynomial in the size of $\K$. As it turns out, already deciding
whether $\K$ has for $\casmap$ some CAS-model is intractable in this
setting, which can be decided by a proper guess of the (in)equality for all $f_i(t_i)$.
More precisely, the following property can be shown.

\begin{proposition}
\label{prop:n-bounded-mc}
Given an $n$-de safe DKB $\DKB$, where $n$ is polynomial in the size
of $\DKB$,  and a
clashing assumption $\casmap$ defined on $N_\K$,
deciding whether $\K$ has (i) some arbitrary CAS-model 
resp.\ (ii) some justified CAS-model of form $\I_{\CAS} = \stru{\I,\casmap}$ is 
\np-complete resp.\ $D^p$-complete%
\footnote{$D^p$ consists loosely speaking of the
``conjunction'' of independent instances $I_1$ and $I_2$ of two problems
in \np\ and \conp, respectively (e.g., SAT-UNSAT).}
in general but feasible in
polynomial time if 
$n$ is bounded by a constant.
\end{proposition}
As regards properties of justified CAS-models,
Proposition~\ref{prop:model-intersection} readily generalizes from
exception-safe to $n$-de safe DKBs if the (in)equalities of the Skolem
terms $t_1,\ldots,t_m$ with individuals are fixed; hence, a least 
justified CAS-model exists relative to such fixed (in)equalities and a
name assignment $\nu$.

The intractability in Proposition~\ref{prop:n-bounded-mc}
holds even under data complexity and when $\K$ is $k$-chain bounded
for a small constant $k$. Furthermore, we obtain as a side
result from the proof that DKB-model checking, i.e., deciding whether an
interpretation $\I$ is a DKB model of a given $\K$, is
\conp-hard and for such DKBs \conp-complete.


As a consequence of the previous result, reasoning when a few (constantly many chains) to
exceptions exist is not more expensive than if no such chains exist;
for polynomially many chains, axiom inference gets more expensive.
%
\begin{theorem}
\label{theo:comp-n-bounded}
Given an $n$-de safe DKB $\DKB$, where $n$ is bounded by a polynomial
in $|\DKB|$, (i) deciding $\DKB\models\alpha$ for an axiom
$\alpha$ and (ii) BCQ answering $\DKB\models Q$ are both
$\Pi^p_2$-complete. In case $n$ is bounded by a constant,  
(i) is $\conp$-complete while (ii) remains $\Pi^p_2$-hard.
\end{theorem}
The results in the theorem hold in fact under data complexity. 
Intuitively, the complexity of BCQ-Answering does not increase in the 
$n$-de bounded case, as checking whether a guess for a 
clashing assumption $\casmap$ that allows to refute the query $Q$ does
not add further complexity in general, since checking whether a 
query $Q$ has no match in an ABox is already \conp-hard. It does so,
however, if the query $Q$ is just an assertion $A(a)$.

\smallskip\noindent
{\bf Towards an ASP encoding.}\; A possible approach for extending
the ASP encoding in Section~\ref{sec:translation} 
in this regard, then, 
would be as follows. We may consider the different ``equality
environments'', where each environment $e$ is given by 
the condition $f_i(t_i) = a_{i_j}$ for some individual $a_{i_j}$, 
using test environments similar as in~\cite{BozzatoES:18}: for each $f_i(t_i)$, we introduce
a predicate $sk_i$ and
an argument $x_i$ in the predicates for the derivation, where
$sk_i$ and $x_i$ can take
any individual name $a_i$ or $f_i(t_i)$; we add propagation rules
that push equalities, of form
\begin{align}
\instd(x_i,z,x_1,\ldots,x_m) \rif \instd(x,z,x_1,\ldots,x_m), sk_i(x).
\label{rule-outro1}\\
\non\instd(x_i,z,x_1,\ldots,x_m) \rif \comment{\non\instd(x,z,x_1,\ldots,x_m)},
sk_i(x). \label{rule-outro2}
\end{align}
and similar for $\tripled$; furthermore, with a technique similar
as in the rules (pdlr-allnrel1) -- (pdlr-allnrel3) for $\AllNRel$ we
can check whether a derivation succeeds for all
environments. As an alternative, we may consider using (recursive) aggregates to
perform this check. Furthermore, the set of auxiliary constants is extended
with further constants that allow to build the Skolem paths $f_1(t_1)$,
\ldots, $f_m(t_m)$. For $m$ bounded by a constant, this would
lead to a fixed program, where the Skolem chains are provided as data;
the latter might be determined inside an ASP encoding as well, which
however is more involving.

\rewnote{2}{rules (2) and (3) are weird. They have the same body but complementary heads with respect to strong negation. This means that if the body is true, they yield an inconsistency. Given that this could be better achieved by adding a constraint, I wonder whether there may be an error here. Or is it perhaps the they are not considered at the same time? If so, this should be better clarified.}
\lbnote{I changed the negative rule with a negative body.}


\section{Related Work}
\label{sec:related}

The relation of the justified exception approach to nonmonotonic
description logics was discussed in \cite{BozzatoES:18}, where in
particular an in-depth comparison w.r.t.\
typicality in DLs \cite{GiordanoGOP:13},
normality \cite{DBLP:journals/jair/BonattiFS11} and
overriding \cite{BonattiFPS:15} was given. 
A distinctive feature of our approach, linked to the 
interpretation of exception candidates as different clashing assumptions,
is the possibility to ``reason by cases'' inside the alternative
justified models \comment{(as we have demonstrated over the Nixon Diamond problem in 
Example~\ref{ex:nixon})}.
%
Note that 
we do not 
consider a preference ordering across defeasible axioms,
but all alternative interpretations that justify their
clashing assumptions (cf. also~(Bozzato, Serafini and Eiter~\citeyearNP{DBLP:conf/kr/BozzatoSE18}) where a preference
is defined by the KB contextual structure).
%

\rewnote{1}{The paper also in the current form often refers to (Bozzato et al. 2018). Thus, it would be useful to precisely confront against (Bozzato et al. 2018). Some lines are devoted to it in Section 8, but this part should be improved.}
\comment{In particular, compared to~\cite{BozzatoES:18}, in this paper we work on a different language:
particularly, $\DLliteR$ allows for reasoning with unnamed individuals and their use in 
inverse roles, as detailed in the sections above. Moreover, we are not considering 
contextual aspects of the previous works, e.g., 
knowledge propagation by $\eval$ operator
and local interpretation of knowledge. Note that, with respect to the notion of defeasibility
expressed on the CKR contexts structure, we have a slightly different interpretation for DKBs: 
while in CKR we defined defeasibility over the inheritance from more general to more specific contexts,
in DKBs we consider exceptions on the ``local'' application of defeasible axioms to the elements of the
knowledge base. We remark that the results shown 
for DKBs 
could be then extended to CKRs in order to study the interaction of
$\DLliteR$ features with contextual structures and overriding preferences.}

The introduction of non-monotonic features in the $\DLlite$ family
and, more in general, to low complexity DLs has been 
the subject of many works, mostly with the goal of 
preserving the low complexity properties of the base logic in the extension.
%
For example, in~\cite{DBLP:journals/jair/BonattiFS11} an in-depth study of the 
complexity of reasoning with circumscription in $\DLliteR$ and
$\cal{EL}$ was presented: 
the idea is to verify whether syntactic restrictions of these languages 
can be useful to limit the complexity of the non-monotonic version of these languages.
The work considers defeasibility on inclusion axioms of the kind $C \isa_n D$,
which intuitively can be read as ``an instance of $C$ is \emph{normally} an instance of $D$''.
Conflicts across defeasible inclusions are solved by providing a priority on such axioms:
an option to define priority is to use the specificity of defeasible inclusions, 
that is $C_1 \isa_n D_1$ is preferred to $C_2 \isa_n D_2$ if $C_1$ is subsumed by $C_2$.
Different fragments of $\DLliteR$ and $\cal{EL}$ are considered to limit the 
complexity of the reasoning problems. In the case of $\DLliteR$,
it is shown that such syntactic restrictions allow to limit the complexity of
instance checking to $\Pi^p_2$.
For $\cal{EL}$, its extension to circumscription is 
ExpTime-hard and more restrictions are needed to limit its complexity
to the second level of the polynomial hierarchy.

Similarly, in~\cite{GiordanoGOP:11} the authors studied the complexity of the
application of their typicality approach to low complexity description logics.
They note that the introduction of the typicality operator to $\ALC$
leads to an increase in complexity of reasoning (query entailment becomes $\textsc{Co-Nexp}^{NP}$):
thus, their goal is to find (fragments of) low level description logics where 
the extension to typicality has a limited impact on the complexity of entailment.
In~\cite{GiordanoGOP:11}, an extension to typicality of the DLs $\DLlite_c$ and $\cal{EL}^\bot$
is proposed and their complexity properties are studied.
It is shown that, in the case of $\cal{EL}^\bot$, the extension with 
typical concept inclusions (called $\Ecal\Lcal^\bot T_{min}$)
is $\textsc{ExpTime}$-hard. However, by limiting to \emph{left local} KBs in $\cal{EL}^\bot$
(i.e., using a fragment of $\cal{EL}^\bot$ that restricts the form of left-side concepts in concept inclusions),
one can show that complexity can be limited to $\Pi^p_2$.
Similarly, the extension of $\DLlite_c$ can be shown to have the same 
$\Pi^p_2$ complexity upper bound. Notably, the complexity bounds for 
$\cal{EL}^\bot$ and $\DLlite_c$ match the ones proved 
in~\cite{DBLP:journals/jair/BonattiFS11}. 

A recent work in this direction
is \cite{DBLP:conf/lpnmr/PenselT17,DBLP:journals/ijar/PenselT18},
where a defeasible version of $\cal{EL}^\bot$
was obtained:
as an interesting parallel with the work presented in our paper,
the goal of such work is
to overcome issues with the approach by
\cite{DBLP:conf/jelia/CasiniS10} on the interpretation of 
defeasible properties on quantified concepts, especially in
nested expressions.
This approach is based on an extension of 
classical canonical
models of $\cal{EL}^\bot$, called \emph{typicality models},
where 
multiple representatives for each concept are used to 
identify different versions of the same concept 
under different levels of typicality.
Using typicality models the authors show that they can
obtain stronger versions of rational and relevant entailment
that do not neglect defeasible information in nested quantifications.
The authors also present a reasoning algorithm for instance checking
under the proposed semantics
by a variant of the materialisation-based approach that only uses the expressivity of $\ELbot$.
Finally, the computational complexity of the defeasible subsumption and instance checking
under the different semantics is investigated:
in particular, the definition of the materialization method 
extending the reasoning to defeasibility while keeping the expressivity of the base logic $\ELbot$
provides the evidence that complexity of reasoning need not to increase in this logic.
This approach is different from ours, which works on all models and
uses factual justifications that need to be derived: on the other hand, canonical models
are useful for characterization and implementation, thus we could investigate some
of the results in these papers in the extension of our work.

\rewnote{3}{I think it is not very clear whether the proposed system successfully models intuitive defeasible reasoning. At least some representative examples could be useful to understand how minimal justified models behave.
Some papers where a lot of representative examples for non-monotonic reasoning can be found are: (see mail)
It would be interesting to see how the system behaves w.r.t. some of these examples.}
\comment{
\begin{example}[Inheritance blocking]
	In studying the properties of rational and minimal relevant closure, 
	\cite{DBLP:journals/ijar/PenselT18} highlight that rational closure suffers from 
	the problem of \emph{inheritance blocking}, intuitively the effect where all
	properties of superclasses of an exceptional class $C$ are not inherited,
	even if they are not related to the ``exceptionality'' of $C$. 
	In the case of our semantics, we have shown in~\cite{BozzatoES:18} by means of
	the \emph{Situs inversus} example from~\cite{BonattiFPS:15} that we can deal
	with property inheritance at the instance level. 

	We show that inheritance can be also preserved when reasoning involves unnamed individuals,
	by rephrasing Example~3.2 from \cite{DBLP:journals/ijar/PenselT18}.
	Consider the DKB $\K_{org}$ defined as:
	\begin{center}
  $\begin{array}{rl}
    \K_{org}: & \left\{\begin{array}{l}
		           \mi{Boss} \subs \mi{Worker},\ \mi{Boss} \subs \non\exists \mi{hasSuperior},\  
		           \exists \mi{hasSuperior}^- \subs \mi{Boss}\\
		           \default(\mi{Worker} \subs \exists \mi{hasSuperior}),\
		           \default(\mi{Worker} \subs \mi{Productive}),\\ 
		           \default(\mi{Boss} \subs \mi{Responsible}),\ 
							 \mi{Worker}(\mi{bob}),\ \non \mi{Boss}(\mi{bob}) 
							\end{array}\right\}
  \end{array}$	
	\end{center}
	Similarly to Example~\ref{ex:mother},
	this 
	DKB admits a model $\I_\CAS$ where we have
	an exception on $\alpha = \mi{Worker} \subs \exists \mi{hasSuperior}$ 
	for the (unnamed) boss $f(\mi{bob})$ of $\mi{bob}$,
	with $f(\mi{bob}) \neq \mi{bob}$. 
	However, we have no reason to override the other properties of $\mi{Boss}$
	on $f(\mi{bob})$, thus we have 
	$\I_\CAS \models \mi{Responsible}(f(\mi{bob}))$ and 
	$\I_\CAS \models \mi{Productive}(f(\mi{bob}))$. Also, if 
	we add $\non \mi{Productive}(\mi{bob})$ to $\K_{org}$,
	the overriding of the respective axiom does not 
	influence the applicability of $\alpha$ 
	to $\mi{bob}$.
	
	As shown by~\cite{DBLP:journals/ijar/PenselT18}, rational closure by the 
	materialization-based approach of~\cite{DBLP:conf/jelia/CasiniS10} 
	fails to derive information on such existential 
	individuals; thus, e.g., it can not derive $\mi{Productive}(f(\mi{bob}))$. 
	Similar considerations on inheritance blocking can be drawn for DLs with 
	typicality as those presented in~\cite{GiordanoGOP:11}.
	\EndEx%
\end{example}
}

Another recent work about reasoning on non-monotonic versions of the ${\cal EL}$ 
family is~\cite{CasiniSM:18}.
The paper considered the logic ${\cal ELO}_\bot$ (i.e., the extension of $\ELbot$ with nominals)
and studies the problem of (non-monotonic) concept subsumption, where
the non-monotonic aspects are represented via rational closure.
The authors provided a polynomial time subsumption algorithm for ${\cal ELO}_\bot$ under rational closure that,
notably, reduces the problem to a series of classical monotonic subsumption tests in the same language.
This allows to use the customary (monotonic) ${\cal EL}$ based reasoners to implement the reasoning method.

A recent approach related to our work is~\cite{EiterLP:16},
in which
inconsistency-tolerant query answering over a set of existential rules
was studied.
The authors considered
removing errors from ontological axioms that lead to inconsistency,
with the possibility to specify a set of axioms 
that should not be touched. They introduced 
two semantics for BCQ Answering on existential rules, 
in which a maximal set of designated rules \emph{(GR semantics)}
\comment{or} rule instances \emph{(LGR semantics)} is 
kept while maintaining consistency.
This dually corresponds to the inherent minimality of clashing assumptions in
justified DKB-models, with the difference that in \cite{EiterLP:16} no
proof of a clashing set or a similar certificate is required for
removing a rule instance.
Exceptions may be harder to obtain under LGR semantics; e.g., in
Example~\ref{ex:mother}, no exception for $\mi{alice}$ to the
defeasible axiom is possible, and thus only models with an unnamed
supervisor for $\mi{alice}$ exist.
Notably, for LGR semantics minimal removal checking is polynomial
for such rule sets, while testing whether a corresponding clashing
assumption is justified is intractable
(cf.\ Proposition~\ref{prop:n-bounded-mc}).
Closer relationship 
with LGR semantics 
remains to be clarified in future work.




\section{Conclusion}    
\label{sec:conclusion}

In this paper, we considered the justified exception approach in~\cite{BozzatoES:18} 
for reasoning on $\DLliteR$ KBs with defeasible axioms.
With respect to our previous works, we had to consider 
the problem of reasoning with unnamed individuals introduced by existential
axioms, especially when they are involved in the reasoning over 
exceptions: we provided different characterizations of DKB-models
with respect to the use of unnamed individuals and their presence
in exceptions.
Considering DKBs where exceptions can appear only on named individuals,
we studied the semantic properties of DKB-models and
we analyzed the complexity of the main reasoning problems. 
We have shown that the limited language of $\DLliteR$ allows us to formulate a
direct datalog translation to reason on derivations for negative information
in instance checking.
Finally, we provided some insights in the case of reasoning with exceptions 
on existential individuals and a direction for extending the datalog translation
in this regard.

\rewnote{2}{The significance seems to lie more on the Description Logics area, where a new non-monotonic language extension is proposed, than in Logic Programming (LP) itself, since no conclusions about the use of ASP are being drawn from the LP perspective.}
\comment{While the focus of this work lies mostly in the area of
Description Logics, 
where we extend current languages to deal with non-monotonicity, we note that 
this work also shows the strength of LP and ASP technologies: in particular, 
we have that ASP provides a practical way to encode and solve tasks 
such as conjunctive query answering in other logical systems. Moreover, by comparing our
previously defined ASP encodings, we note that by the declarative nature of ASP,
the management of the new aspects of $\DLliteR$ can be encoded by an adaptation of
the program rules. This shows some flexibility of the rule-based
approach, which is in particular valuable for developing prototypical implementations.}


As discussed in the previous sections, reasoning with
unnamed individuals in exceptions can be further studied to 
obtain more insights, possibly with a
refinement or further elaboration of the ASP encoding
sketched in Section~\ref{sec:unnamed}. In particular, 
it would be interesting to explore possible variants of the
exception semantics that we considered here, from
comparison with and inspired by related work such as the one discussed
in Section~\ref{sec:related}. Using Skolem terms in exceptions, which
underlies the approach in \cite{EiterLP:16}, may be an option, but the
consequences will have to be carefully
considered, as issues with Skolemization in non-monotonic reasoning are folklore. 

The complexity results obtained for $n$-de safe DKBs imply that some
encoding in ASP is possible that uses predicates of arity bounded by a
constant, in contrast to the rules
(\ref{rule-outro1})--(\ref{rule-outro2}) sketched in
Section~\ref{sec:unnamed}, with rule bodies that have a variable
(possible large) number of literals. In other contexts, such
bounded-arity encodings proved to be useful using solvers based on
decomposition techniques \cite{DBLP:journals/tplp/BichlerMW16}. It
thus would be interesting to see whether this approach could be
fruitfully used for encoding DKB reasoning with unnamed individuals as
well.

Finally, we plan to apply 
the current results on $\DLliteR$ in the framework of 
CKR with hierarchies as
in~(Bozzato, Serafini and Eiter~\citeyearNP{DBLP:conf/kr/BozzatoSE18}),
for which the current results have to be extended to the
respective setting. Imposing preference on exceptions in the
hierarchy may however increase the complexity and thus require
language constructs that offer increased expressivity, such as
optimization (e.g.,\ by weak constraints) or disjunction in rules
heads.

\rewnote{1}{In references:

. "(DL2014)" vs. "RuleML+RR 2019" vs. "DL-2012" vs "KR 2018". Please unify the style of the conference acronyms

. "New Generation Computing" is the only journal name which is fully specified}
\lbnote{Fixed formats in references (we might leave out conferences locations for being more compact)}



\comment{
\bigskip\noindent
\textbf{Acknowledgments.}
%
We thank the reviewers for their constructive comments and suggestions to improve this paper.

\smallskip\noindent
\textbf{Competing interests declaration.} The authors declare none.
}

\nocite{BozzatoES:18}
\bibliographystyle{acmtrans}
\bibliography{bibliography}


\newpage
\appendix
\section{Description Logic $\DLliteR$}

\begin{table}[t] \small
\caption{Syntax and Semantics of $\DLliteR$ operators, where $A$ is any
atomic concept, $C$ and $D$ are any concepts, $P$ and $R$ are any
atomic roles, 
$S$ and $Q$ are any (possibly complex) roles, $a$ and $b$ are any individual constants.}

\label{tab:dl-lite-operators}
\begin{tabular}{lll}
\hline
\hline
\textbf{Concept constructors} & Syntax	& Semantics \\
\hline
atomic concept		& $A$		& $A^\I$ \\
complement		& $\neg C$	& $\Delta^\I \setminus C^\I$ \\
existential restriction	& $\exists R$	& $
                                         \{x\in\Delta^\I \,|\,
					       \exists y.\Pair{x}{y}\in R^\I
						     \}$\\
\hline
\hline
\textbf{Role constructors}	& Syntax	& Semantics \\
\hline
atomic role		& $R$		& $R^\I$\\
role complement		& $\non S$		& $\Delta^\I \times \Delta^\I \setminus S^\I$\\
inverse role		& $R^-$		& $\left\{\Pair{y}{x} \,\left|\,
                                           \Pair{x}{y}\in R^\I\right.\mkern-2mu\right\}$\\
\hline
\hline
\textbf{Axioms}	& Syntax	& Semantics \\
\hline
concept inclusion 
	& $C\isa D$	& $C^\I \subseteq D^\I$\\
role inclusion
	& $S \isa Q$	& $S^\I \subseteq Q^\I$\\[.5ex]
role disjointness	& $\mathrm{Dis}(P,R)$ & $P^\I\cap R^\I = \emptyset$ \\
role inverse	& $\mathrm{Inv}(P,R)$ & $P^\I = \left\{\Pair{y}{x} \,\left|\,
                                           \Pair{x}{y}\in R^\I\right.\mkern-2mu\right\}$ \\
reflexivity assertion   & $\mathrm{Ref}(R)$ & $\{ \Pair{x}{x}|\, x \in \Delta^\I \} \subseteq R^\I$ \\
irreflexivity assertion   & $\mathrm{Irr}(R)$ & $R^\I \cap \{ \Pair{x}{x}|\, x \in \Delta^\I \} = \emptyset$\\[.5ex]
concept assertion 	& $C(a)$	& $a^\I \in C^\I$ \\
role assertion 		& $S(a,b)$	& $\Pair{a^\I}{b^\I} \in S^\I$ \\
\hline
\hline
\end{tabular}
\end{table}

In Table~\ref{tab:dl-lite-operators} we present the syntax and semantics of operators 
included in the description logic $\DLliteR$.
Further rules for the composition of axioms in $\DLliteR$ are specified in 
Section~\ref{sec:prelims}.

\section{FO-translation for $\DLliteR$}
\label{sec:fo-translation}

\begin{table}[t]%

\small
\caption{Translation $\phi_\alpha(\vec{x})$, $\vec{x}=x_1,\ldots,x_n$, of $\DLliteR$ axioms $\alpha$ in
  $\Lcal_\Sigma$ to first-order logic.}
\label{tab:fol-translation}

\medskip
\hrule\mbox{}\\[.5ex]
\textbf{Translation $\phi_\alpha(\vec{x})$ for axioms $\alpha$}\\[1ex]	
$\begin{array}[t]{r@{\;\;}l}               
D(a) \mapsto & \gamma_D(a) \\
R(a,b) \mapsto & R(a,b) \\
\non R(a,b) \mapsto & \non R(a,b) \\[2.5ex]

C \subs D \mapsto & 
\beta_C(x_1) \im \gamma_D(x_1)\\
R \subs T \mapsto & 
R(x_1,x_2) \im T(x_1,x_2)\\[1ex]
\end{array}$\;
%
$\begin{array}[t]{r@{\;\;}l}               
\mathrm{Dis}(R,S) \mapsto & 
(R(x_1,x_2) \im \non S(x_1,x_2))\ \wedge\\
& (S(x_1,x_2) \im \non R(x_1,x_2))\\
\mathrm{Inv}(R,S) \mapsto & 
(R(x_1,x_2) \im S(x_2,x_1))\ \wedge\\
& (S(x_1,x_2) \im R(x_2,x_1)) \\
\mathrm{Ref}(R) \mapsto & 
R(x_1,x_1)\\
\mathrm{Irr}(R) \mapsto & 
\non R(x_1,x_1)\\[1ex]
\end{array}$

\medskip
\textbf{Translation $\beta_E(\vec{x})$ for (left-side) expressions $E$}\\[1ex]	
$\begin{array}{r@{\;\;}l}
A \mapsto & A(x_1)\\
\exists R \mapsto & R(x_1,x_2)\\[1ex]
\end{array}$

\medskip
\textbf{Translation $\gamma_E(\vec{x})$ for (right-side) expressions $E$}\\[1ex]	
$\begin{array}{r@{\;\;}l}
A \mapsto & A(x_1)\\
\non C_1 \mapsto & \non \beta_{C_1}(x_1)\\
\exists R \mapsto & R(x_1, f_R(x_1))\\[1ex]
\end{array}$
\hrule\mbox{}
\end{table}

In Table~\ref{tab:fol-translation} we provide a FO-translation for
axioms in $\DLliteR$.
Given a $\DLliteR$ axiom $\alpha$ in
$\Lcal_\Sigma$,
the formula $\forall \vec{x}.\phi_\alpha(\vec{x})$,
 where $\vec{x}=x_1,x_2,\ldots, x_n$ is a list of variables,
expresses $\alpha$ as a first order formula.
The translation rules for $\phi_\alpha(\vec{x})$ are recursively defined by 
set of rules $\beta_E(\vec{x},x_c)$ for left-side and $\gamma_E(\vec{x},x_c)$ for right-side expressions,
shown at the bottom of Table~\ref{tab:fol-translation}.

By the definition of this translation, Lemma~\ref{lem:horn-equiv} 
can then be proved analogously to the case of the FO-translation for 
$\SROIQrl$ provided in~\cite[Appendix A.2]{BozzatoES:18}.
Intuitively, it is possible to show that, using the provided translation, 
every $\DLliteR$ axiom can be expressed as a universal Horn sentence
$\forall\vec{x}.\phi_\alpha(\vec{x})$,
where $\vec{x}=x_1,\ldots,x_n$ is a list of free variables.
Hence, $\phi_\alpha(\vec{x})$  can be
written as $\phi_\alpha(\vec{x}) = \bigwedge_{i=1}^\ell
\forall\vec{x}_i. \gamma_i(\vec{x},\vec{x}_i)$, where each
$\gamma_i$ is a Horn clause of the form 
\begin{equation}
\label{eq:horn-clause-form}
\gamma_i(\vec{x},\vec{x_i}) = p_1(\vec{x},\vec{x}_{i,1})\land\cdots\land p_k(\vec{x},\vec{x}_{i,k}) \rightarrow p_0(\vec{x},\vec{x}_{i,0})
\end{equation}


\noindent where
	(i) each $p_i$ is a concept name or a role name with 
  %
possibly $p_0 = \bot$ (falsum); and 
%
(ii) each variable in $\vec{x},\vec{x}_{i,j}$ occurs in the antecedent
      (safety), and  $\vec{x}_i = \vec{x}_{i,0},\ldots,\vec{x}_{i,k}$.

\section{Proofs of Main Results}

\subsection{DL Knowledge Base with Justifiable Exceptions}

\begin{manualproposition}{\ref{prop:pushing-eq}}
  Let $\I_{\CAS} \,{=}\, \stru{\I, \casmap}$ be a CAS-model of DKB 
$\K$ and let $\K'$ result from $\K$ by pushing equality w.r.t.\ $\I$,
  i.e., replace all $a, b \,{\in}\, N_\K$ s.t.\ $a^\I\,{=}\,b^\I$  by one representative.
If $\K'$ is exception-safe,
then $\I_{\CAS}$ can be justified only if every $\stru{\alpha,\ee}\,{\in}\, \casmap$ is over $N_\K$.
\end{manualproposition}
\begin{proof}
Suppose $\I_{\CAS}$ is justified and some $\stru{\alpha,\ee}\in
\casmap$ is not over $N_\K$, i.e., for some $e$ in $\ee$ we
have $e \notin N_\K^\I$.  
Then, by definition of justification, some clashing set $S$ for
$\stru{\alpha,\ee}$ with $\ee$ not over  $N_\K$
is satisfied in all CAS-models  $\I'_{\CAS}$ of $\K$
that are NI-congruent with $\I_{\CAS}$. This means that $S$ can be
derived with axiom unfolding restricted by the clashing
assumptions in $\casmap$. But then $S$ can also be derived without
restrictions, and thus from the knowledge base $\K_s'$. 
%
However, this means that $\K'$ is not exception-safe, which is a contradiction. 
\end{proof}

\begin{manualproposition}{\ref{prop:complexity-recognize-safe}}
Deciding whether a given DKB $\DKB$ is exception safe is feasible in
\nlogspace, and whether it is $n$-de safe in \ptime, if $n$ is bounded by a polynomial in the size of $\K$. 
\end{manualproposition}

\begin{proof} 
Apparently, $\DKB$ is not exception safe, if some Skolem term $t_1$ 
resp.\ Skolem terms $t_1,t_2$ exists such that an atom $D(t_1)$
resp.\ $R(t_1,t_2)$ can be derived from the first-order rewriting
$\phi_{\K_s}$ of $\K_s$ such that an assertion $D(e_1)$ resp.\ $R(e_1,e_2)$
occurs in a clashing set for some possible exception
$\stru{\alpha,\vec{e}}$ to a defeasible axiom $\default(\alpha)$ in
$\K$. Such $D(e_1)$ resp.\ $R(e_1,e_2)$ can be guessed and the 
derivation of $D(t_1)$ resp.\ $R(t_1,t_2)$ be non-deterministically 
checked by applying iteratively the axioms and deriving atoms 
$\alpha_0(\vec{t}^0)$, $\alpha_1(\vec{t}^1)$, \ldots,
$\alpha_i(\vec{t}^{i}) = D(t_1)$ resp.\   $\alpha_i(\vec{t}^{i}) =
R(t_1,t_2)$; at each step, it is sufficient to store merely the
type of each argument (Skolem term, $N_K$ individual) of $\alpha_j$
and to require that $t_1$ resp.\ some of $t_1,t_2$ is a Skolem
term. This is feasible in non-deterministic logspace. 

As for $n$-de safety where $n>0$, one can first similarly check
whether a derivation with a cycle is possible, such that 
atoms $D(t_1)$ and $D(t'_1)$ resp.\ $R(t_1,t_2)$ and $R(t'_1,t'_2)$
can be derived where $t_1$ is a subterm of $t'_1$ resp.\ $t_1$ is a
subterm of $t'_1$ or $t_2$ is a subterm of $t'_2$. This can be done
with additional book-keeping in nondeterministic logspace. 

If such a cycle exists, then $\DKB$ is not $n$-de safe for any $n\geq
0$. Otherwise, we can systematically enumerate the terms $t_1$
resp.\ $t_1,t_2$ in a lexicographic fashion. To this end, we determine
for an individual $a \in N_\K$ all assertions $R_i(a,f_{R_i}(a))$ that hold for
it; each gives rise to a child $f_{R_i}(a)$, and by repeated
application (where again all assertions $R_i(f_{R_i}(a),f_{R_j}(a))$
are determined) we obtain a tree whose depth is linearly bounded. We can
traverse this tree in a depth first manner where, before expanding, we
ask at the current node  whether some atom $D(t_1)$
resp.\ $R(t_1,t_2)$ as above is reachable
(which then contains some new Skolem term not seen so far). This test
is, like computing all assertions $R(t,f_R(t))$ for  $t$ feasible in
nondeterministic logspace. In this way, the number of nodes 
explored until $n+1$ different terms are found is polynomial in $n$,
and the effort for each node is polynomial; as there are linearly many 
starting nodes, the overall effort is polynomial in $n$. Thus if $n$ is 
polynomially bounded in the size of $\K$, the overall effort is
polynomial in the size of $\K$ as well.
\end{proof}

\begin{manualproposition}{\ref{prop:chain-bounded}}
Deciding whether a given DKB $\DKB$ is $n$-chain safe, where $n\geq 0$, is feasible in
\nlogspace.
\end{manualproposition}
\begin{proof}
This test can be made by an algorithm similar to the one checking
exception-safety. It nondeterministically builds a chain
$R_1(a,t_1), \ldots, R_m(t_{m-1},t_m)$ starting from the assertions in
$\K_s$, where
it records at each point just the predicate name $R_i$ and the type of
the arguments. It increases a counter whenever by applying some axiom
a new Skolem term $f_{R_i}(t)$ is introduced; if the counter exceeds
$n$, then $\K$ is not $n$-chain safe. Since logarithmic workspace is
sufficient for the book-keeping, the result follows.
\end{proof}

\subsection{Semantic Properties}

\begin{manualproposition}{\ref{prop:model-intersection}} 
Let  $\I^i_{\CAS} = \stru{\I_i,\casmap}$, $i \in\{1,2\}$, be
$\NI$-congruent CAS-models of
a DKB $\K$ 
fulfilling (\ref{prop:exc-conjunction}).
Then, $\I_{\CAS} = \stru{\I,\casmap}$, where $\I = \I_1 \widetilde{\cap}_\Ncal \I_2$ 
and $\Ncal$ includes all individual names occurring in 
$\K$ and for each element $e$ occurring in $\casmap$
some $t\in \Ncal$ such that $t^{\I_1}=e (= t^{\I_2})$,
is also a CAS-model of $\K$.
Furthermore, if  some $\I^i_{\CAS}$, $i\in\{1,2\}$, 
is justified and $\DKB$ is exception safe, then $\I_{\CAS}$ is justified. 
\end{manualproposition}
\begin{proof}
	By Lemma~\ref{lem:horn-equiv}, we have that the FO-translation $\phi_\K$ of a DKB $\K$
	is equivalent to a conjunction of Horn clauses. 	
	Since $\I^i_{\CAS} = \stru{\I_i,\casmap}$, $i \in\{1,2\}$ are CAS-models of $\K$,
we can consider their ``intersection'' model $\I = \I_1 \widetilde{\cap}_\Ncal \I_2$
over the set $\Ncal$ corresponding to $N_\K$ extended with its grounding on Skolem functions.
	We can prove that $\I_{\CAS}$ is indeed a model for $\Kcal$:
	let us suppose that in $\I$ we have that some 
	Horn clause $\gamma(\vec{x}) = a_1(\vec{x}), \ldots, a_n(\vec{x}) \rightarrow b(\vec{x})$ from $\phi_\K$ 
	is violated for some variable assignment $\sigma$ for $\vec{x} = x_1, \ldots, x_k$.
	Since all elements of $\I_1 \widetilde{\cap}_\Ncal \I_2$ are named by some term in $\Ncal$,
the assignment is of the form $\sigma(x_i) = t_i$ with $t_i \in \Ncal$ for
	$i \in \{1, \ldots, k\}$. Since each $t_i$ is interpreted in $\I$ by
	$[t^{\I_1}, t^{\I_2}]$, by the interpretation of concepts and roles in 
	$\I$ it follows that each $a_i\sigma$ is true in both $\I_1$
        and $\I_2$. Furthermore, as $\gamma(\vec{x})$ is violated for
        $\sigma$, the axiom $\alpha$ in $\K$ that led to
        $\gamma(\vec{x})$ does not have an exception for $\sigma$ in $\I_{\CAS}$;
        hence, from the property (\ref{prop:exc-conjunction}) it follows
        that  $\gamma(\vec{x})$ has neither for the assignment
        $\sigma_1(x_i) = t_i^{\I_1}$, $i\in \{1, \ldots, k\}$, an
        exception in $\I^1_{\CAS}$ nor for the assignment
        $\sigma_2(x_i) = t_i^{\I_2}$, $i\in \{1, \ldots, k\}$, in
        $\I^2_{\CAS}$; thus $b(\vec{x})$ is true for $\sigma_i$ in
        $\I_i$, $i=1,2$. By construction, this means $b(\vec{x})$ is
        true for $\sigma$ in $\I$, and thus $\gamma(\vec{x})$ is
         satisfied for $\sigma$ in $\I$, which is a contradiction.
	Hence $\I_{\CAS} = \stru{\I_1 \widetilde{\cap}_\Ncal \I_2,\casmap}$ is also a CAS-model for $\K$.
	
	With respect to justification, let us assume without loss of generality that 
	$\I^1_{\CAS}$ is justified.
	Hence, if $\stru{\alpha, \ee} \in \chi$, then we have that there exists a
	clashing set $S_{\stru{\alpha, \ee}}$ for this clashing assumption
	such that for every NI-congruent $\I'_{\CAS}$ it holds that $\I'_{\CAS} \models S_{\stru{\alpha, \ee}}$.
	Moreover, considering $\Kcal$ to be exception safe, 
	all exceptions in $\chi$ are named by individuals. Then, for the justified model
	$\I^1_{\CAS}$, for every Skolem term $t \in \Ncal$, we must have
	$t^{\I_1} \neq c^{\I_1}$ if $c \in \NI$ appears in an exception.
	Thus, since $\I_{\CAS}$ is NI-congruent with $\I^1_{\CAS}$
	and property~(\ref{prop:exc-conjunction}) holds, 
	it follows that also the intersection model $\I_{\CAS}$
	is justified.
\end{proof}

\begin{mcorollary}{\ref{coroll:least-model} (least model property)}
If a clashing assumption $\casmap$ for an exception-safe DKB $\K$ is satisfiable for
name assignment $\nu$, then $\K$ has an  $\subseteq_\Ncal$-least (unique minimal)
CAS-model $\hat{\I}_\K(\casmap,\nu) = \stru{\hat{\I},\casmap}$ on
$\Ncal$ that contains all Skolem terms of individual constants, i.e.,
for every CAS-model $\I_{\CAS}' = \stru{\I',\casmap}$ relative to
$\nu$, it holds that $\I_{\CAS} \subseteq_\Ncal \I'_{\CAS}$.
Furthermore, 
$\hat{\I}_\K(\casmap,\nu)$ is justified if $\casmap$ is justified.
\end{mcorollary}
\begin{proof}
Given that $\casmap$ is satisfiable, 
consider any CAS-model	$\I_{\CAS} = \stru{\I,\casmap}$ for $\K$ with name assignment $\nu$.
Then, some model $\I'_{\CAS} = \stru{\I',\casmap}$ such that $\I'_{\CAS} \subseteq_\Ncal \I_{\CAS}$ 
exists that is founded, i.e., has the following property:
\begin{itemize}
    \item[] If $R(e_1,e_2)$ is true in $\I'_{\CAS}$, 
	then a sequence $\alpha_1(\ee_1) = R(e_1,e_2), \alpha_2(\ee_2),$$\ldots,$$ \alpha_k(\ee_k)$ 
	of atoms with domain elements such that:
	\begin{enumerate}[label=(\arabic*)]
	\item 
	  $\alpha_i(\ee_i)$ is obtained from $\alpha_{i+1}(\ee_{i+1})$ by applying an
		axiom (resp. the rule for it) where $\alpha_i(\ee_i)$ is the head and
    $\alpha_{i+1}(\ee_{i+1})$ is the body, and both are satisfied;
	\item
	  some $A(t_1)$ or $S(t_1,t_2)$ in $\K$ exists such that
    $\alpha_k(\ee_k) = A(t_1^{\I'})$  resp. $\alpha_k(\ee_k) = S(t_1^{\I'},t_2^{\I'})$.
	\end{enumerate}
\end{itemize}
If this were not the case for some $R(e_1,e_2)$ in $\I_{\CAS}$, i.e., its truth 
is not ``founded'' in the facts in $\K$, we could set $R(e_1,e_2)$ and, recursively along
all possible such chains, all $\alpha_i(\ee_i)$ to false and obtain a
smaller model; eliminating all such $R(e_1,e_2)$ (at once or
alternatively in an iterative process) yields a founded model
$\I'_{\CAS}$.

Now if $R(e_1,e_2)$ is true in $\I'_{\CAS}$, starting from $A(t_1)$ resp. $S(t_1,t_2)$, we can apply the
same axioms (rules) as in (1) w.r.t. $\K_s$, using Skolemization and obtain $R(t,t')$.
	Since $\K$ is exception safe, if the first (resp.\ second) argument of $R$ can appear in a clashing
  assumption, $t$ (resp. $t'$) must not be a Skolem term, but a constant. 
This is analogous for $A(e_1)$ where $A$ can occur in a clashing assumption.
	
Let now $\I^1_{\CAS} = \stru{\I_1,\casmap}$ and
        $\I^2_{\CAS}=\stru{\I_2,\casmap}$ be two founded models of
        $\K$ relative to $\nu$. Considering condition~(\ref{prop:exc-conjunction}),
	if	$t^{\I_1} = e$ holds where $e$ occurs in a clashing assumption, then $t = c$ 
	must hold for some constant $c$ (in fact, for some $c \in \K$). 
	As $\I^2_\CAS$ interprets constants in the same way, we thus have $t^{\I_1} = t^{\I_2}$ so the intersection property holds.
        Thus by Proposition~\ref{prop:model-intersection}, $\I_{\CAS} = \stru{\I_1
        \widetilde{\cap}_\Ncal \I_2,\casmap}$ is also a model of $K$.
       We claim that $\I^1_{\CAS}\subseteq_\Ncal \I_{\CAS}$ holds, i.e.,         
       for each atom $\alpha \in At_\Ncal(\I^1_\CAS)$, $\I^1 \models
       \alpha$ implies $\I \models \alpha$. 
       Towards a contradiction, suppose $\I^1\models \alpha$ but
       $\I'\not\models \alpha$, and consider $\alpha = R(t,t')$.
       As $\I^1_{\CAS}$ is founded,
       $R(t,t')$ must result, modulo $\nu$, from a derivation chain 
       of an atom $R(e_1,e_2)$ showing foundedness. The bottom of
       this chain, $\alpha_k(\vec{t}) = A(t_1)$ resp.\ $\alpha_k = S(t_1,t_2)$ is a fact in $\K$,
       and thus $\I' \models \alpha_k(\vec{t})$. By an inductive argument, we
       obtain that the axiom that has been applied to derive
       $\alpha_{i-1}(\ee_{i-1})$ from $\alpha_i(\ee_i)$ in the
       chain from $\alpha_k(\ee_k)$ to $\alpha_1 = R(e_1,e_2)$ in $\I^1$ can be 
       applied in $\I'$ to derive the same $\alpha_{i-1}(\vec{t}_{i-1})$
       from $\alpha_{i}(\vec{t}_{i})$ as in $\I'$, as the axiom is
       also applicable to $\alpha_i(\vec{t}_i)$ in $\I^2$. Thus,
       $\I'\models \alpha_1$, i.e., $\I'\models R(t,t')$, which is a contradiction.
       The case of $\alpha = A(t_1)$ is analogous. Thus, $\I^1_{\CAS}\subseteq_\Ncal \I_{\CAS}$
       holds, which means by transitivity that
        $\I^1_{\CAS}\subseteq_\Ncal \I^2_{\CAS}$. As $\I^2_{\CAS}$ can
        be an arbitrary founded CAS-model, it follows that 
        $\I^1_{\CAS}$ is a $\subseteq_\Ncal$-least founded CAS-model;
        and as for every CAS-model $\IC_{\CAS}$ some founded CAS-model
        $\I^2_{\CAS}$ exists such that $\IC^2_{\CAS} \subseteq_\Ncal \IC_{\CAS}$,
        $\I^1_{\CAS}$ is a $\subseteq_\Ncal$-least CAS-model
        $\hat{\I}_\K(\casmap,\nu)$ of $\K$, as as claimed.
        By symmetry, also has $\IC^2_{\CAS}$ this property.
        Justification of $\hat{\I}_\K(\casmap,\nu)$ 
				is immediate from Proposition~\ref{prop:model-intersection}.
\end{proof}

\begin{manuallemma}{\ref{lem:named}}
Suppose $\I$ is a model of a $\DLliteR$ knowledge base $\K$ and $N\subseteq
\NI \setminus\NI_S$ includes all individuals occurring in $\K$. Then, 
the $N$-restriction $\I^N$ is named w.r.t.\ $sk(N)$ and a model of $\K$.
\end{manuallemma}
\begin{proof}
	By Lemma~\ref{lem:horn-equiv}, we have that 
	$\phi_\K$
	represents the contents of $\K$ as a first order formula.
	As discussed in~\ref{sec:fo-translation}, 
	this formula can be written as a conjunction 
	$\phi_\alpha(\vec{x}) = \bigwedge_{i=1}^\ell\forall\vec{x}_i. \gamma_i(\vec{x},\vec{x}_i)$, 
	where each $\gamma_i(\vec{x},\vec{x}_i)$ is a Horn clause.
	By construction, $\I^N$ is named relative to $sk(N)$. 
	Moreover, in any assignment $\theta: \vec{x}_i \mapsto \NI_s$,
 $\gamma_i \theta$ evaluates to false in $\I^N$	whenever there is some
 $\theta(x) \notin sk(N)$ with $x \in \vec{x}_i$.
 It follows that, if $\gamma_i$ is verified under $\I$ and $\theta$,
 then it is also verified by $\I^N$ and $\theta$.
 This implies that if $\I \models \phi_\K$ then $\I^N \models \phi_\K$,
 and thus $\I^N$ is a named model for $\K$.
\end{proof}

\begin{manualtheorem}{\ref{theo:justified-cas-char} (justified CAS characterization)}
Let $\casmap$ be a satisfiable clashing assumptions set for an 
exception safe DKB
$\K$ and name assignment $\nu$. 
Then, $\casmap$ is justified iff $\stru{\alpha,\ee} \in \casmap$ implies
some clashing set $S = S_{\stru{\alpha,\ee}}$ exists such that 
  \begin{enumerate}[label=(\roman*)]
\itemsep=0pt
  \item $\hat{\I} \models \beta$, for each  positive $\beta
    \in S$, where $\hat{\I}_\K(\casmap,\nu)= (\hat{\I},\casmap)$,
and 
\item no CAS-model $\I_{\CAS} = \stru{\I, \casmap}$ with name assignment
  $\nu$ exists s.t. $\I \models \beta$ for some $\neg \beta \in S$.
\end{enumerate}
\end{manualtheorem}
\begin{proof}
	Since $\chi$ is satisfiable, then for Corollary~\ref{coroll:least-model} 
	there exists a least CAS-model $\hat{\I}_\K(\casmap,\nu)= (\hat{\I},\casmap)$.
	
	Then, we can show that the justification is characterized by the conditions
	on the validity of clashing sets.
	Let us suppose that $\chi$ is justified. Then, $\hat{\I}_\K(\casmap,\nu)$
	is justified: this implies that, by definition, for every $\stru{\alpha, \ee} \in \chi$
	there exists a clashing set $S = S_{\stru{\alpha,\ee}}$ such that,
	for every CAS-model $\I'_{\CAS} = \stru{\I', \casmap}$ that is $\NI$-congruent
	with $\hat{\I}_\K(\casmap,\nu)$, we have that $\I' \models S$.
	This directly verifies items (i) and (ii).
	
	On the other hand, let us suppose that for every 
	$\stru{\alpha,\ee} \in \casmap$ there exists some
	clashing set $S = S_{\stru{\alpha,\ee}}$ such that
	items (i) and (ii) hold.
	Considering a CAS-model $\I'_{\CAS} = \stru{\I', \casmap}$ that is $\NI$-congruent
	with $\hat{\I}_\K(\casmap,\nu)$, it holds that from (i)
	$\I' \models \beta$ for every positive $\beta \in S$.
	Moreover, condition (ii) implies that 
	$\I' \models \non\beta$ for every negative assertion $\non\beta \in S$.
	Thus, $\stru{\alpha,\ee} \in \casmap$ is justified in $\chi$:
	hence, $\hat{\I}_\K(\casmap,\nu)$ is justified.
\end{proof}

\subsection{Datalog Translation for $\DLliteR$ DKB}

\subsubsection{Normal Form}
\label{sec:nform-appendix}

\begin{manuallemma}{\ref{lem:normal-form}}
Every DKB $\DKB$ can be transformed in linear time into an
equivalent DKB $\DKB'$ which has modulo auxiliary symbols the same
DKB-models, and such that $n$-de safety and $n$-chain safety are preserved.
\end{manuallemma}
\begin{proof}
Using the new concept names $A_{\exists R}$, all strict inclusion 
axioms $C\isa D$ can be easily expressed in the given form. Negative
concept assertions $\non C(a)$ can be expressed by $A_a(a)$ and
$A_a\isa \non C$, where $A_a$ is a fresh concept name. Similarly,
negative role assertions $\non R(a,b)$ can be expressed by
$R'(a,b)$ and $\mathrm{Dis}(R,R')$ where $R'$ is a fresh
role name. 

As regards defeasible axioms, we can express $\default(\non C(a))$ by 
$A_a(a)$ and $\default(A_a\isa \non C)$, where $A_a$ is as above, and furthermore
$\default(A \isa \non B)$ by 
$\default(A \isa A')$ and $B \isa \non A'$, where $A'$ is a fresh concept
name. Finally, we can express  $\default(\non R(a,b))$ by
$R'(a,b)$, $\default(R'\isa S)$, and $\mathrm{Dis}(R,S)$, where
$R'$ and $S$ are fresh role names. It can be seen that 
for each justified CAS-model $\I_{\CAS} = \stru{\I, \casmap}$ of the
original DKB, some justified CAS-model $\I'_{\CAS} = \stru{\I', \casmap'}$
for the DKB obtained by a rewriting step exists, where $\casmap'$
is obtained by modifying exceptions in $\casmap$ according to the
rewriting in the obvious way. In the same way, one obtains from each justified CAS-model $\I'_{\CAS}
= \stru{\I', \casmap'}$ of the rewritten DKB a justified CAS-model of the
original DKB. The rewriting does not use existential axioms, and thus
$n$-de safety and $n$-chain safety are preserved.
\end{proof}

\comment{
\noindent
With respect to the equivalence result defined in Lemma~\ref{lem:normal-form},
one interesting aspect would be \emph{strong
equivalence}~\cite{DBLP:journals/tocl/LifschitzPV01}
between the original DKB $\DKB$ and the transformed DKB $\DKB'$.
In the non-monotonic setting, two theories $\Gamma$ and $\Gamma'$ are strongly
equivalent if, for any other theory $H$, $\Gamma \cup H$ and $\Gamma' \cup H$
have the same models.
For our purposes, we would need to carefully define and analyze
this notion, taking auxiliary symbols into account, cf.\
\cite{DBLP:journals/tplp/Woltran08}. The very limited form of changes by
the normal form transformation and the syntax of $\DLliteR$ suggest that
a relativized equivalence may prevail, possible under adaptations.
}

\subsubsection{Translation Correctness}

Let $\I_\CAS = \stru{\I, \casmap}$ be a justified named CAS-model. We define the set of overriding assumptions
$\OVR(\I_\CAS) = \{\, \ovr(p(\ee)) \;|\; \stru{\alpha, \ee} \in \casmap,\, I_{dlr}(\alpha) = p \,\}$.
Given a CAS-interpretation $\I_{\CAS}$, 
we can define a corresponding 
interpretation 
$S = I(\I_{\CAS})$ for $PK(\K)$: the construction
is similar to the one in~\cite{BozzatoES:18}, 
by extending it to negative literals and 
providing an interpretation for existential individuals:

\begin{enumerate}[label=(\arabic*).]
\itemsep=0pt
\item 
$l \in S$, if $l \in PK(\K)$;
\item 
 $\instd(a,A) \in S$, if $\I \models A(a)$ and
 $\non\instd(a,A) \in S$, if $\I \models \non A(a)$; 
\item 
  $\tripled(a,R,b) \in S$, if  $\I \models R(a,b)$ and
  $\non\tripled(a,R,b) \in S$, if  $\I \models \non R(a,b)$;	
\item 
  $\tripled(a,R,aux^\alpha) \in S$, if $\I \models \exists R(a)$
	for $\alpha = A \isa \exists R$;
\item
	$\AllNRel(a,R) \in S$
	if $\I \models \non \exists R(a)$;
\item
  $\ovr(p(\ee)) \in S$, if $\ovr(p(\ee)) \in \OVR(\I_{\CAS})$;
\end{enumerate}

\begin{manualproposition}{\ref{prop:correctness}}
Let  $\K$ be an exception safe DKB in $\DLliteR$ normal form. Then:
%
	\begin{enumerate}[label=(\roman*).]
	\item 
	  for every (named) justified clashing assumption $\casmap$, 
		the interpretation $S = I(\hat{\I}(\casmap))$ is an answer set of $PK(\K)$;
	\item
	   every answer set $S$ of $PK(\K)$ is of the form 
		$S = I(\hat{\I}(\casmap))$ where $\casmap$ is a (named) justified clashing assumption for $\K$.
	\end{enumerate}
\end{manualproposition}
\begin{proof} 

	We consider $S = I(\hat{\I}(\casmap))$ built as above and 
	reason over the reduct 
	$G_S(PK(\K))$ of $PK(\K)$ with respect to $S$.
	By definition, the reduct $G_S(PK(\K))$ is the set of rules
	resulting from the ground instances of rules of $PK(\K)$
	after the removal of (i)
	every rule $r$ such that $S \models l$ for some NAF literal $\naf l \in \Body(r)$; 
	and	(ii) the NAF part (i.e., $\ovr$ literals) from the bodies of the remaining rules.
	Basically, $G_S(PK(\K))$ contains all ground rules from $PK(\K)$
	that are not falsified by some NAF literal in $S$:
	in particular, this excludes application rules for the axiom instances
	that are recognized as overridden.

	\medskip\noindent		
	Item (i) can be proved by showing that given a
	justified $\chi$, $S$ is an answer set for 
	$G_S(PK(\K))$ (and thus $PK(\K)$).
	The proof follows the same reasoning of the one in~\cite[Lemma 6]{BozzatoES:18},
	where 
	the fact that $I(\hat{\I}(\casmap))$ satisfies rules
	of the form (pdlr-supex) in $PK(\K)$
	is verified by the condition (4) on existential formulas
	in the construction of the model above.
	
	We first show that $S \models G_{S}(PK(\K))$: for every
  rule instance $r \in G_{S}(PK(\K))$ we show that $S \models r$ holds by 
  examining the possible rule forms that occur in the reduction.
	In the following we show some of the most relevant cases, 
	while the other cases can be proven with similar reasoning.
	Assuming that $S \models \Body(r)$ for a rule instance $r$ in 
	$\grd(PK(\K))$, we show that $\Head(r)\in S$.
	
	\begin{itemize}
  \item
	   \textbf{(pdlr-instd):}
		 then $\insta(a,A) \in S$ and, by definition of the
		 translation, $A(a) \in \K$. This implies that $\hat{\I} \models A(a)$
		 and thus $\instd(a,A)$ is added to $S$.
  \item
	   \textbf{(pdlr-subc):}
     then $\{\subClass(A,B), \instd(a,A)\} \subseteq S$.
		 By definition of the translation we have $A \subs B \in \K$.
		 Then, for the construction of $S = I(\hat{\I}(\casmap))$, 
		 $\hat{\I} \models A(a)$. 
		 This implies that $\hat{\I} \models B(a)$ 
		 and $\instd(a,B)$ is added to $S$.		
  \item
	   \textbf{(pdlr-supex):}
		 then $\{\supEx(A,R,\mi{aux}^{\alpha}), \instd(a,A)\} \subseteq S$ 
		 (with $\mi{aux}^{\alpha}$ a new constant relative to the considered 
		 existential axiom $\alpha$).
		 By definition of the input translation, we have that
		 $A \subs \exists R \in \K$. Moreover, by the construction of $S$,
		 we have $\hat{\I} \models A(a)$. 
		 This implies that $\hat{\I} \models \exists R(a)$: thus, by the conditions defining the 
		 construction of $S$,
     $\tripled(a,R,\mi{aux}^{\alpha})$ is added to $S$.
  \item
	   \textbf{(pdlr-nsupex):}
		 then $\{\supEx(A,R,\mi{aux}^{\alpha}), \const(a), \AllNRel(a,R)\} \subseteq S$.
		 By definition of the translation, we have that 
		 $A \subs \exists R \in \K$ and $a$ is a constant (i.e., either an individual name
		 appearing in $\K$ or any other auxiliary existential constant $\mi{aux}^{\beta}$).
		 Moreover, by the construction of $S$, since $\AllNRel(a,R) \in S$ 
		 it holds that  $\hat{\I}\models \non \exists R(a)$.
		 This implies that $\hat{\I} \models \non A(a)$ and then
		 $\non\instd(a,A)$ is added to $S$.
  \item
	   \textbf{(ovr-subc):} 
	   then $\{\defsubs(A,B), \instd(a,A), \non\instd(a,B)\} \subseteq S$.
	   By definition of the translation, we have that 
		 $\default(A \subs B) \in \K$ and, by the construction of $S$,
		 $\hat{\I} \models A(a)$ and $\hat{\I} \models \non B(a)$.
		 Thus, $\hat{\I}$ satisfies the clashing set $\{A(a), \non B(a)\}$
		 for the clashing assumption $\stru{A \subs B, a}$. 
		 This implies that $\stru{A \subs B, a} \in \casmap$
		 and by construction $\ovr(\subClass,a,A,B)$ is added to $S$.						
  \item
	   \textbf{(app-subc):} 
     then $\{\defsubs(A,B), \instd(a,A)\} \subseteq S$.
	   As $r\in G_{S}(PK(\K))$,
	   we have that $\ovr(\subClass, a,A,B) \notin \OVR(\hat{\I}(\casmap))$ and hence
	   $\stru{A \subs B, a} \notin \casmap$.		
	   By definition, $A \subs B \in \K$ and, by the construction of $S$,  
		 $\hat{\I} \models A(a)$.
	   Thus, for the definition of $\CAS$-model and the semantics,
		 $\instd(a,B)$ is added to $S$.
  \end{itemize}
  To show that $S$ is indeed an answer set for $G_S(PK(\K))$, we have
	to prove its minimality with respect to the rules in the reduction.
	We can show that no model $S'\subseteq S$ of $G_{S}(PK(\K))$ such that $S'\neq S$ can
  exist: as $\hat{\I}(\casmap)$ is the least model of $\K$
  w.r.t.\ $\casmap$, $S'$ can not be a proper subset of $S$ on any of the
  facts from the input translations, nor on the derivable instance level facts 
	($\instd$, $\tripled$ or their negation). 
	Thus, $S'$ needs to contain all atoms on $\ovr$ from $S$, as for every corresponding
  clashing assumption $\stru{\alpha,\ee} \in \casmap$ the body of
  some overriding rule in $PK(\K)$ that encodes a
  clashing set for $\stru{\alpha, \ee}$ will be satisfied.
	Thus, $S'=S$ must hold.
		
	\medskip\noindent
	For item (ii), we can show that from any answer set $S$
	we can build a justified model $\I_S$ for $\Kcal$
	such that $S = I(\hat{\I}(\casmap))$ holds.
	The model can be defined similarly to the original
	proof in~\cite{BozzatoES:18}, but we need to consider auxiliary individuals in the
	domain of $\I_S$.
	Considering an answer set $S$ for $PK(\K)$, 
	we can build a model $\I_S = \stru{\I_S, \chi_S}$ as follows:
	\begin{itemize}
  \item 
    $\Delta^{\I_S} = \{ c \;|\; c \in \NIs \} \cup \{\mi{aux}^\alpha \;|\; \alpha = A \isa \exists R \in \K \}$;  
  \item
	  $a^{\I_S} = a$, for every $a \in \NIs$ and auxiliary constant of the kind $\mi{aux}^\alpha$;
	\item
	  $A^{\I_S} = \{d \in \Delta^{\I_S} \mid S \models \instd(d, A) \}$, for every $A \in \NC$;
	\item
		$R^{\I_S} = \{(d,d') \in \Delta^{\I_S} \times \Delta^{\I_S} \,|\, S \models \tripled(d, R, d') \}$ for $R \in \NR$;
  \end{itemize}
	Finally, $\casmap_S = \{\stru{\alpha, \ee} \mid I_{dlr}(\alpha) = p, \ovr(p(\ee)) \in S \}$.
	We need to show that $\I_S$ is a least justified CAS-model for $\K$, that is:
	\begin{enumerate}[label=(\alph*)]
  \item
   for every $\alpha \in \Lcal_\Sigma$ in $\Kcal$, $\I_S \models \alpha$;
  \item
   for every  $\default(\alpha) \in \Kcal$ (where $\alpha \in \Lcal_\Sigma$),
   with $|\vec{x}|$-tuple $\vec{d}$ of elements 
   in $\NIs$ such that $\vec{d} \notin \{ \ee \mid \stru{\alpha,\ee} \in \casmap \}$, 
   we have $\I_S \models \phi_\alpha(\vec{d})$.
  \end{enumerate}
	The claim can then be proven by considering the 
	effect of deduction rules for existential axioms in $G_S(PK(\K))$:
	auxiliary individuals provide the domain elements in $\I_S$
	needed to verify this kind of axioms.
	
	In particular, condition (a) can be shown 
	by cases considering the form of
	all of the (strict) axioms $\beta \in \Lcal_\Sigma$ that can occur in $\K$.
	We show in the following some of the cases (the others are similar):
	\begin{itemize}
	\item 
	  Let $\beta = A(a) \in \K$, then, by rule (pdlr-instd),
	  $S \models \instd(a,A)$. 
	  This directly implies that $a^{\I_S} \in A^{\I_S}$.
	\item 
	  Let $\beta = A \subs B \in \K$, then
	  $S \models \subClass(A,B)$. 
		If $d \in A^{\I_S}$,
	  then by definition $S \models \instd(d,A)$:
	  by rule (pdlr-subc) we obtain that $S \models \instd(d,B)$
	  and thus $d \in B^{\I_S}$.
    On the other hand, let us assume that $e \in \non B^{\I_S}$ with $S \models \non\instd(e,B)$:
		then, by the negative rule (pdlr-nsubc) we have that $S \models \non\instd(e,A)$.
		This implies that $e \in \non A^{\I_S}$, since otherwise 
		we would have $\instd(e,A) \in S$ and $S$ would be inconsistent.
		Note that if we assume $e \in \non B^{\I_S}$ but $S \not\models \non\instd(e,B)$,
		we also obtain that $e \in \non A^{\I_S}$, otherwise by (pdlr-instd) 
		and the definition of $\I_S$ we would derive that	$e \in \non B^{\I_S}$.	
	\item 
		Let $\beta = A \subs \exists R \in \K$, then
	  $S \models \supEx(A,R,\mi{aux}^\beta)$. 
		If $d \in A^{\I_S}$, then by definition $S \models \instd(d,A)$.
		By rule (pdlr-supex) we obtain that $S \models \tripled(d,R,\mi{aux}^\beta)$.
		By the definition of $\I_S$, this means that there exists an $e \in \Delta^{\I_S}$
		such that $(d,e) \in R^{\I_S}$: this implies that $\I_S \models \exists R(d)$.
		On the other hand, if we consider an $e \in \Delta^{\I_S}$ such that $e \in \non \exists R^{\I_S}$
		and $S \models \non \tripled(e, R, c)$ for each $\const(c)$ in $S$ (i.e., for each $c \in \Delta^{\I_S}$).
		Then, by the definition of $\I_S$ and the
		rules (pdlr-allnrel1) -- (pdlr-allnrel3) it holds that $S \models \AllNRel(e,R)$.
		By the negative rule (pdlr-nsupex), we have that 
		$S \models \non \instd(e,A)$. As above,	we have that $e \in \non A^{\I_S}$,
		as otherwise we would have $\instd(e,A) \in S$ 
		and $S$ would be inconsistent.
	\end{itemize}
	For condition (b), let us assume that $\default(\beta) \in \K$ with $\beta \in \Lcal_\Sigma$.
	  Let $\beta = A \subs B$. 
		Then, by definition of the translation,
	  we have that $S \models \defsubs(A,B)$.
	  Let us suppose that $b^{\I_S} \in A^{\I_S}$: then 
	  $S \models \instd(b, A)$.
	  Supposing that $\Pair{A \subs B}{b} \notin \casmap_{S}$, then
	  by definition $\ovr(\subClass, b, A, B) \notin \OVR(\hat{\I}(\casmap))$.
	  By the definition of the reduction, the corresponding instantiation
	  of rule (app-subc) has not been removed from $G_{S}(PK(\K))$.
	  This implies that $S \models \instd(b, B)$ and thus 
		$b^{\I_S} \in B^{\I_S}$.
		The contrapositive case for $b^{\I_S} \in \non B^{\I_S}$ and the 
		cases for 
		the other forms of defeasible axioms
		can be proved similarly.
	
	Thus, $\I_S$ is a CAS-model of $\K$: moreover, we can show that
  $\I_S=\hat{\I}(\casmap_S)$ holds. 
  Assuming that $\I\subset
  \I_S$ is a CAS-model of $\K$ with clashing assumption $\casmap_S$, we
  can construct an interpretation $S'\subset S$ such that $S'\models
  G_S(PK(\K))$, by removing (at least) one instance-level fact
  $\instd(d,A)$ or $\tripled(d,R,d')$ from $S$. 
  However, this would contradict that $S$ is an answer set of $PK(\K)$: 
  thus, it holds that $\I_S=\hat{\I}(\casmap_S)$.

	The justification of $\casmap_S$
	follows by verifying that the new formulation of 
	overriding rules correctly encodes the 
	possible clashing sets for the input defeasible axioms.
	Formally, since any 
	$\stru{\alpha, \ee} \in \casmap_S$ is due to 
	$\ovr(p(\ee)) \in S$ and $\ovr(p(\ee))$ is derived from the reduct 
	$G_S(PK(\K))$, it follows that $S$ must satisfy 
	some overriding rule $r$ for $p(\ee)$.
	This means that $\I_S$ must
	satisfy the clashing set $S_{\stru{\alpha, \ee}}$ for $\stru{\alpha, \ee}$ encoded 
	by the rule $r$.
	By the property defined in Theorem~\ref{theo:justified-cas-char}
  it follows that the clashing assumption $\stru{\alpha, \ee}$ is
  justified: thus, $\casmap_S$ is justified.
\end{proof}





\subsection{Complexity of Reasoning Problems}
\label{sec:complexity-appendix}

    


\subsubsection{Satisfiability}

\begin{lemma}
\label{lem:sink-3col}
Given a graph $G=(V,E)$, deciding whether
every color assignment to the nodes of degree 1 in $G$ is can be extended to
a 3-coloring of $G$ is $\Pi^p_2$-complete.
\end{lemma}
\begin{proof}
The problem is in $\Pi^p_2$, since a guess for a coloring $\kappa$ of
the degree 1 nodes of $E$ that can not be extended to a 3-coloring of
$G$ can be checked with a call to a \conp-oracle.

The $\Pi^p_2$-hardness is shown by a reduction from evaluating  QBFs
of the form $\Phi = \forall X \exists Y\, E$, where $E = \bigwedge_i C_i$ is
in CNF and each clause $C_i$ has size 3; to this end, a reduction of
3SAT to 3-colorability can be easily generalized, e.g., the one in
\url{https://www.cs.princeton.edu/courses/archive/spring07/cos226/lectures/23Reductions.pdf}.
In this and similar reductions, a graph $G$ is constructed from $E$ that is 3-colorable iff $E$ is
satisfiable. This graph has nodes $v_p$ and $v_{\neg p}$ that
correspond to the literals $p$, $\neg p$ of the variables $p$
occurring in $E$, such that color assignments to $v_p$ resp.\ $v_{\neg
p}$ correspond to truth assignments to the literals; furthermore, it
has a distinguished node $B$, such that $v_p$, $v_{\neg p}$, and $B$ form a triangle for each $p$.

We may assume that $G$ has no node of degree $1$ (else we would involve that node
with two fresh nodes in a triangle). To encode $X$, we connect
to each $v_x$ for $x\in X$ a fresh node $v'_x$, which becomes a
source. Furthermore,  we connect fresh nodes $z_1$ and $z_2$ to
$B$ and to each other (forming a triangle), and connect further
fresh nodes $z'_1$ and $z'_2$ to $z_1$ and $z_2$, respectively. The
purpose of the latter gadget is to ensure that in each 3-coloring of
the resulting graph $G'$, the node $B$ must have the
color of either $z'_1$ or $z'_2$. This will then allow to easily reduce checking
whether for an assignment $\sigma$ to $X$, $E(\sigma(X),Y)$ is
satisfiable to an 3-coloring extension test. 

The claim is that $\Phi$ evaluates to true iff every color assignment $\kappa$ to the nodes $v'_p$ is extendible to a 3-coloring
of $G'$.

($\Leftarrow$) Suppose every coloring to  the $v'_p$ is extendible to
a 3-coloring of $G'$. Then, if we color $z'_1$ and $z'_2$ with $b$, also
$B$ is colored with $b$. If we have a truth assignment
$\sigma$ to $X$ and set $\kappa(v_p) = g$ if $\sigma(p) =$ true and 
$\kappa(v_p) = r$ if $\sigma(p) =$ false, the extending 3-coloring must
encode a truth assignment $\mu$ such that $E(\sigma(X),\mu(Y))$
evaluates to true, by the properties of the graph $G$.

($\Rightarrow$) Conversely, suppose $\Phi$ evaluates to true, and
consider and coloring $\kappa$ of the degree-1 nodes. We can extend
$\kappa$ to a 3-coloring of $G'$ as follows. Suppose the node $z'_1$ has
the color $b$; then we can color $z_1$, $z_2$ and $B$ such
that the latter also has the color $b$. Now for each $x\in X$ we color
$v_x$ with $g$ and $v_{\neg x}$ with $r$  if $n'_x$ has color $r$, and 
$v_x$ with $r$ and $v_{\neg x}$ with $g$  otherwise; thus, the
coloring of the nodes $v_x$, $v_{\neg x}$ reflects a truth assignment
to $X$. As $\Phi$ evaluates to true, we can color all other $v_p$ and
$v_{\neg p}$ nodes, as well as the auxiliary nodes such that we
obtain a 3-coloring of $G'$.  
\end{proof}

\begin{manualproposition}{\ref{prop:DKB-existence}}
Let $\DKB$ be a normalized $\DLliteR$ DKB, and let 
$\casmap' = \{ \stru{\alpha,\ee} \mid D(\alpha) \in \DKB$, $\ee$ is over 
standard names$\,\}$ be the clashing assumption with all exceptions possible. 
Then, $\DKB$ has some 
justified CAS-model $\I_{\CAS} = \stru{\I, \casmap}$ such that
$\casmap\subseteq \casmap'$ iff
$\DKB$ has some CAS-model $\I_{\CAS} = \stru{\I, \casmap'}$.
\end{manualproposition}
\begin{proof}
($\Rightarrow$) Every 
justified CAS-model $\I_{\CAS} = \stru{\I, \casmap}$ of $\K$ 
is a CAS-model $\I_{\CAS} = \stru{\I, \casmap}$ of $\K$, and  
since $\casmap\subseteq \casmap'$, also $\I'_{\CAS} = \stru{\I, \casmap'}$
is a CAS-model of $\DKB$ as every exception $\stru{\alpha,\ee}$ in
$\I_{\CAS}$ is also made in $I'_{\CAS}$ (and possibly more exceptions are made). 

($\Leftarrow$) Suppose that $\DKB$ has some CAS-model of the form
$\I'_{\CAS} = \stru{\I, \casmap'}$. We then can construct some
justified CAS-model of $\DKB$ by trying to remove, one by one, the 
clashing assumptions $\stru{\alpha,\ee}$ in $\casmap'$. To this end, 
let $\Delta_0$ consist of the FO-translation $\forall \vec{x}
\phi_\alpha(\vec{x})$ of all non-defeasible axioms in $\K$, plus all
instances $\phi_\alpha(\ee)$ of defeasible axioms $D(\alpha)$ in $\K$
such that $\stru{\alpha,\ee}\notin\casmap_0$. Furthermore, let $ex_1$,
$ex_2$,\ldots, $ex_i = \stru{\alpha_i,\ee_i}$, \ldots\  be a (possibly infinite) enumeration
of $\casmap_0$, and let $\casmap_0 = \casmap$. We then build the sequences
theories $\Delta_i$ and clashing sets $\casmap_i$, $i\geq 1$ inductively as follows:
$$
(\Delta_{i+1}, \casmap_{i+1}) = \left\{
\begin{array}{@{}ll@{}}
(\Delta_i, \casmap_i) & \textrm{ if } \Delta_i \cup \{
    \phi_{\alpha_i}(\ee_i)\} \textrm{ is unsatisfiable}\\
(\Delta_i \cup \{\phi_{\alpha_i}(\ee_i)\}, \casmap_i \setminus \{ \stru{\alpha_i,\ee_i}\}  & \textrm{ otherwise }
\end{array}
\right.
$$
We obtain then that $\Delta = \bigcup_{i\geq 0} \Delta_i$ is satisfiable iff
$\Delta_0$ is satisfiable, and in the latter case it satisfies
$\casmap = \bigcap_{i\geq 0} \casmap_i$; furthermore, if
$\stru{\alpha_{i+1},\ee_{i+1}}$ was not removed, then $\Delta_i
\models \neg \phi_{\alpha_i}(\ee_i)$ and thus $\Delta\models \neg
\phi_{\alpha_i}(\ee_i)$. The formula  $\neg \phi_{\alpha_i}(\ee_i)$
amounts for axioms $\alpha_i$ of the defeasible (normal) form $A \isa
(\neg) B$, $R \isa S$, $\mathrm{Dis}(R,S)$, and $\mathrm{Irr}(R),$
to the minimal clashing set in Table~\ref{tab:clashingsets}; for
$\mathrm{Inv}(R,S)$, it amounts to $R(e_1,e_2)\land\neg S(e_2,e_1)\,\lor\,
\neg R(e_1,e_2)\land S(e_2,e_1)$, which is logically equivalent to 
$(R(e_1,e_2)\lor S(e_2,e_1))\land (\neg R(e_1,e_2)\lor \neg S(e_2,e_1))$.
As $\Delta$ is Horn and $\Delta \models R(e_1,e_2)\lor S(e_2,e_1)$, it
follows that either $\Delta \models R(e_1,e_2)$ or $\Delta \models
S(e_2,e_1)$ holds; thus either $\Delta \models R(e_1,e_2)\land \neg
S(e_2,e_1)$ or $\Delta \models \neg R(e_1,e_2)\land S(e_2,e_1)$
holds, i.e., one of the two clashing sets is indeed derived. 

Consequently, every model $\I$ of $\Delta$ gives rise to a justified CAS-model
$\I_{\CAS} = \stru{\I,\casmap}$ of $\K$.
\end{proof}

\subsubsection{Entailment Checking}

\begin{manualtheorem}{\ref{theo:DKB-entail-conp}}
Given an exception-safe DKB $\DKB$  and an axiom $\alpha$,  deciding
$\DKB\models\alpha$ is \conp-complete; this holds also for data
complexity and instance checking, i.e., $\alpha$ is of the form $A(a)$
for some assertion $A(a)$.
\end{manualtheorem}
\begin{proof}
To  refute $\DKB\models \alpha$,
we need to show that
some justified CAS-model  $\I_{\CAS} = \stru{\I,\casmap}$ of $\DKB$ 
exists such that $\I \not\models \alpha$. Without loss of generality,
we assume that $\alpha$ is normalized.

Depending on the type of $\alpha$, given $\casmap$ and a name
assignment $\nu$, a we push the latter first to $\K$ and then proceed as follows.

\begin{enumerate}[label=(\alph*)]
 \item If $\alpha$ is an assertion $A(a)$, $R(a,b)$, $\non A(a)$, or
   $\non R(a,b)$ we can check the existence of $\I_{\CAS}$ by the use
   of Theorem~\ref{theo:named}, and we can assume that $\I$ is named
   relative to $N$ with $N = sk(N_\DKB)$. That is, we can use the
   materialization calculus relative to $\casmap$, and verify given a
   justified $\casmap$ that $\alpha$ can not be derived.

\item If $\alpha$ is an inclusion axiom $A \isa B$, then we need to
show that for some element $e$ it holds that $\I \models A(e)$ and
$\I\models \neg B(e)$. To deal with this, we first add to $\K$ the
axioms $Aux \isa A$, $Aux \isa \non B$; this does not 
compromise exception safety nor that $\casmap$ is justified. We may then
assume without loss of generality that $\I \models Aux(e)$, i.e., $\I \not\models \neg Aux(e)$.

We next add to $\K$ an assertion $A_e(a_e)$, where $A_e$ and $a_e$ are a
fresh concept and individual name, respectively; this serves to
give $e$ a name if it is outside the elements named in $\I$ by Skolem terms.
This addition again does neither 
compromise exception safety nor that $\casmap$ is justified, and $\I$
can be adjusted to it.

We then check whether $\neg Aux(a_e)$ is not derivable from the
resulting DKB $\K'$ under $\casmap$; this holds iff some $\I$ with $e$
not named by some Skolem term of the  $\K$ exists. 

Otherwise, $e$ must be named by some Skolem term $t$ of $\K$. We thus
check that for none such $t$, $\neg Aux(t)$ is derivable from $\K'$
under $\casmap$; the depth of $t$ can be polynomially bounded.

The checks can be done in nondeterministic logspace, and thus deciding
$\K \not\models \alpha$ under $\casmap$ is feasible in
polynomial time. 

\item if $\alpha$ is a role inclusion axiom $R\isa S$, then we need to
  show that for some elements $e,e'$ it holds that $\I \models
  R(e,e')$ and $\I\models \neg S(e,e')$. Similarly as for $A\isa B$
  above, we can use an auxiliary role $Aux$ and axioms $Aux \isa
  R$, $Aux \isa \neg S$, proceed with assuming that $\I\not\models
  \neg Aux(e,e')$, introduce $a_{e},a_{e'}$ for $e,e'$ in $\K$, and 
  test then that from $\K'$ under $\casmap$, $\neg Aux(t,t')$ is not
  derivable, where $t,t'$ range over the Skolem terms of depth bounded
  by a polynomial and $a_e$
  resp.\ $a_{e'}$, where $e=e'$ must be respected. Again, this allows
   us to show that $\K\not\models \alpha$ under $\casmap$ in
  polynomial time.
  
\item in all other cases, we can proceed similarly as in a) and b),
as to refute $\alpha$, we need to show that for some element $e$ 
(resp.\ elements $e,e'$), $\I$ satisfies some literals $\alpha_i(e)$, $1=1,\ldots,k$.
(resp.\ $\alpha_i(e,e')$, $1=1,\ldots,k$). 
In particular, for $A \isa \neg B$: $A(e)$, $B(e)$; for
$\mathrm{Dis}(R,S)$: $R(e,e')$, $S(e,e')$; for 
$\mathrm{Irr}(R)$: $R(e,e')$; for $\mathrm{Inv}(R,S)$: either
$R(e,e')$, $\neg S(e,e')$ or  $\neg R(e,e')$, $S(e,e')$. Along the
same lines as above we can introduce an auxiliary concept resp.\ role
$Aux$, individual names
$a_e$, $a_{e'}$ etc.\ and decide $\K\not\models \alpha$ under
$\casmap$ in polynomial time.
\end{enumerate}
Thus, to decide $\K\not\models \alpha$, we can guess a justified clashing assumption $\casmap$ over $N_\K$ 
together with a clashing set $S_{\stru{\alpha,\ee}}$ for each $\stru{\alpha,\ee} \in \casmap$ 
and check (i) that $\casmap$ is satisfiable, (ii) that all
$S_{\stru{\alpha,\ee}}$
are derivable from $\DKB$ under $\casmap$, and (iii) that 
$\DKB \not\models \alpha$. Each of the steps (i)--(iii) is feasible in polynomial time. Consequently, 
the entailment problem $\K\models \alpha$ is in \conp. 

\smallskip

\noindent{\em \conp-hardness.} The \conp-hardness can be shown by a reduction from 
inconsistency-tolerant reasoning from $\DLliteR$
KBs under AR-semantics \cite{DBLP:conf/rr/LemboLRRS10}. Given 
a $\DLliteR$ KB $\Kcal = \Acal\cup \Tcal$ with ABox $\Acal$ and TBox
$\Tcal$,
a repair is a maximal subset $\Acal'\subseteq \Acal$ such
that $\Kcal' = \Acal'\cup\Tcal$ is satisfiable; an assertion $\alpha$
is AR-entailed by $\Kcal$, if $\Kcal'\models \alpha$ for every repair
$\Kcal'$ of $\Kcal$. As shown by Lembo et al.{},
deciding AR-entailment is \conp-hard; this continues to hold under UNA and if all
assertions involve only concept resp.\ role names.

Let $\hat{\DKB} \,{=}\, \Tcal \,{\cup}\, \{ D(\alpha) \mid \alpha \,{\in}\, \Acal\}$, i.e., 
all assertions from $\Kcal$ are defeasible. 
As easily seen, under this assumption the
maximal repairs $\Acal'$ correspond to the justified clashing assumptions
by $\casmap = \{ \stru{\alpha,\ee}\mid$ $ \alpha(\ee) \in
\Acal\setminus\Acal'\}$. Thus, $\Kcal$ AR-entails $\alpha$ iff
$\hat{\DKB} \models \alpha$,
proving \conp-hardness. 
Furthermore, in order to establish
the result for exception safety without looking further into the
structure of $\K$, we may apply the normal form transformation of
Lemma~\ref{lem:normal-form}; as defeasible
assertions $\default(A(a))$ are transformed to $A'(a)$ and $\default(A'
\isa A)$ and since $A'$ does not occur on the right hand side of any axiom,
no Skolem terms can be derived that feeds into the positive literal of
the clashing set $\{ A'(a), \neg A(a)\}$; similarly,
$\default(R(a,b))$ is translated to $R'(a,b)$ and $\default(R' \isa
R)$ and similarly exception safety is warranted.

%

As Lembo et al.{} proved the \conp-hardness under data-complexity,
with the normal form transformation (which then requires merely the
addition of the assertions $A'(a)$ resp.\ $R'(a,b)$, or a renaming of
the symbols) the claimed result for data complexity follows.
\end{proof}

\begin{manualproposition}{\ref{prop:DKB-entail-conp}}
Given a DKB $\DKB$, deciding where $\DKB\models\alpha$ is \conp-hard
even if no roles occur in $\DKB$ and $\alpha$ is 
an assertion $A(a)$.
\end{manualproposition}
\begin{proof} 
\citeN[Theorem 16]{DBLP:journals/jcss/CadoliL94} showed that 
given a positive propositional 2CNF $F$ over variables $V$, 
deciding whether an atom $z$ is a circumscriptive consequence of 
$F$ is \conp-hard if all variables except $z$ are
minimized. In circumscription, the latter is denoted as
$\mathit{CIRC}(F;P;Z)\models z$ where $P=V\setminus Z$ and $Z = \{
z\}$, where $\mathit{CIRC}(F;P;Z)$ is defined as a QBF with free
variables $V$; semantically, 
$\mathit{CIRC}(F;P;Z)$ captures
the $P;Z$-minimal models of $F$, which are the models $M$ of $F$ for which no model $M'$ of $F$ 
exists such that (a) $M \cap Q = M'\cap Q$, where $Q = V- (P\cup Z)$,
and (b) $M'\cap P \subset M\cap P$.

We reduce the inference $\mathit{CIRC}(F;P,Q;Z)\models z$
to entailment $\DKB \models A(a)$, where the variables in $V$ are used
as concept names, as follows.
\begin{itemize}
   \item For each clause $c = x \lor y$ in $F$, we add to $\DKB$ an
axiom $x \isa \non y$ if $z\neq x,y$ and an axiom $x \isa z$
(resp.\ $y\isa z$) if $z\,{=}\,y$ (resp.\ $x\,{=}\,z$). Informally, we
flip in this representation the polarity of all variables except $z$,
in order to obtain $\DLliteR$ axioms.

\item Furthermore, for each
variable $x\neq z$, we add 
a defeasible assertion
$D(x(a))$, where $a$
is a fixed individual.
\end{itemize}

This construction effects that the justified DKB-models of
$\DKB$ correspond to the models of
$\mathit{CIRC}(F;P,\emptyset;\{z\})$, where the minimality of exceptions
in justified DKB-models emulates the minimality of circumscription
models.
Formally, $\DKB \models  z(a)$ iff $\mathit{CIRC}(F;P,\emptyset;\{z\})\models z$. 

($\Leftarrow$) Suppose $\DKB \not\models z(a)$; then, some justified
model $\IC = \stru{\I,\casmap}$ of $\DKB$ exists such that
$\IC\not\models z(a)$.  Let $M = \{ v \in V \mid v\neq z, \IC
\not\models v(a) \}$; we claim that $M$ is a $P;Z$-minimal model of
$F$. Suppose this is not the case.  Then, some model $M'$ of $F$ exists
such that (a) and (b) hold. Then, some variable $v \in (M \cap P)
\setminus (M' \cap P)$ exists, which means that $F \cup \{ \neg x \mid
x\neq z, x \in V- M \} \not\models v$. As $\IC\not\models v(a)$, an
exception to $D(v(a))$ was made in $\casmap$. However, by switching 
$v(a)$ in $\IC$ to true, we obtain an NI-congruent CAS-interpretation
$\IC'$ that satisfies $\DKB$ relative to $\casmap$. This means 
that the exception to $D(v(a))$ is not justified, which is a
contradiction.

($\Rightarrow$) Suppose that $\mathit{CIRC}(F;P,\emptyset;\{z\})\not\models z$, i.e., 
some $P;Z$-minimal model $M$ of $F$ exists such that $M\not\models z$.
We define $\IC = \stru{\I,\casmap}$ where $\IC \models v(a)$ iff $v
\in V\setminus M$ and $\casmap$ contains all exceptions for $D(v(a))$
where $v \in M$. Similarly as in the if-case, it is argued that $\IC$ is a
justified model of $\DKB$. As $\IC \not\models z(a)$, it follows that
$\DKB \not\models z(a)$. This concludes the proof of the claim.

Similarly as in the proof of Theorem~\ref{theo:DKB-entail-conp},  
the defeasible assertions $D(x(a))$ can be moved to defeasible axioms $D(c \isa x)$ with a single assertion $c(a)$.
\end{proof}

\subsubsection{Conjunctive Query Answering}

\begin{manualtheorem}{\ref{theo:CKR-CQ-answering}}
Given an exception-safe DKB $\DKB$ and a
Boolean CQ $Q$, deciding whether $\DKB \models Q$ is
(i) $\Pi^p_2$-complete in combined complexity and (ii)
\conp-complete in data complexity.
\end{manualtheorem}
\begin{proof}
To start with (i), as for  membership in $\Pi^p_2$, to refute $Q$ we
can guess for a justified
CAS-model $\I_{\CAS} = \stru{\I, \casmap}$ such that
$\I_{\CAS} \not\models Q$ 
the clashing assumption $\casmap$ on $N_\DKB$ and a name
assignment $\nu$, which we can push to the knowledge base.
Since $\DKB$ is exception safe, we can decide
in \nlogspace\ whether $\casmap$ is satisfiable
relative to $\nu$ (which can be pushed to $\K$)
and can indeed give rise to a desired justified CAS-model $\I_{\CAS}$ of $\DKB$. We then can use an
\np\ oracle to check whether for some
polynomial number of Skolem terms $ST$, 
where the number depends on $Q$ and $\DKB$,
the query has a match on $N_\K \cup ST$ in 
the least CAS-model $\hat{\I}_\K(\casmap,\nu)$ of
$\K$;
to this end, each atom $A(t)$ resp.\ $R(t,t')$ in the match must be 
derived by applying the axioms (that is, by unraveling $\I_{\CAS}^{N_\K}$); 
this will ensure that a match exists
in each CAS-model $\I_{\CAS} = \stru{\I, \casmap}$ of $\K$.
If the oracle answer is no, then some $\I_{\CAS}$ such that
$\I_{\CAS}\not\models Q$ exists. 
Consequently, refuting $\DKB \models Q$ is in $\Sigma^p_2$, which proves the membership part.


The $\Pi^p_2$-hardness of (i) is shown by a reduction from a
generalization of deciding whether a graph is 3-colorable: given an
(undirected) graph $G=(V,E)$, can every color assignment to the nodes of degree 1
in $G$ (i.e., source nodes) be extended to a 3-coloring of $G$? This
problem is $\Pi^p_2$-complete (see Lemma~\ref{lem:sink-3col}).

We construct a DKB $\DKB$ as follows. We use roles $R$, $R_r$, $R_g$,
$R_b$, and $E$, and as individual names $r,g,b$ and each $v\in V$, 
where we assume that names are unique (this can be easily enforced by
adding further auxiliary axioms). Informally, $R$ and the $R_c$ serve
to encode color assignments to nodes and $E$ to represent
the edges of the graph. We add to $\DKB$ the following axioms:

\begin{itemize}
 \item defeasible axioms $D(R_r(v,r))$,    $D(R_g(v,g))$,    $D(R_b(v,b))$, for each
node $v$ of degree $\neq 1$;
    
\item $R_r \isa R$, $R_g \isa R$, $R_b \isa R$;

\item $\exists R_r \isa \neg \exists R_g$, $\exists R_r \isa \neg \exists R_b$,
  $\exists R_g \isa \neg \exists R_b$;
\end{itemize}
and the assertions 
 \begin{itemize}
  \item $R(v,r)$ $R(v,g)$, $R(v,b)$ for each non-source $v$;
  \item  $E(r,g)$, $E(r,b)$, $E(g,b)$,  $E(g,r)$, $E(b,r)$, $E(b,g)$.
\end{itemize}
Intuitively, we must make for each source node
$v$ an exception to two of the three axioms  $D(R_r(v,r))$,    $D(R_g(v,g))$,
$D(R_b(v,b))$, and in this way assign a color to $v$.
E.g., for assigning red $(r)$ the exceptions are $R_g(v,g))$ and 
$R_b(v,b)$ which have the minimal clashing sets $\{ \neg
R_g(v,g)\}$ and $\{ \neg R_b(v,b)\}$, respectively; for the
other assignments this is analogous.
Every choice $\kappa$ of a coloring for the sources 
in $G$ thus gives rise to a natural justified clashing
assumption $\CAS_\kappa$.

The Boolean query that we construct  is 
$$
Q = \exists \vec{y} \bigwedge_{v\in V} \,R(v,y_v) \land \bigwedge_{e =
    (v,v')\in E} \,E(y_v, y_{v'}).
$$ 
Informally, the graph $G$ is encoded in $Q$, where the variables
$y_v$ range over the colors of the nodes $v$; with $R(v,y_v)$ we pick
a color for a match where for sources only the color chosen
by $\kappa$ is available, while for the other nodes all three colors
$r$, $g$, $b$ are available. The $E$-atoms enforce that adjacent nodes
must have different color.

It is then not difficult to verify that $\DKB \models Q$
 holds under UNA iff 
for every coloring $\kappa$ of the sources, we have $\I_{\CAS_\kappa} \models Q$, i.e.,
the coloring $\kappa$ of the sources can be extended to a
3-coloring of the whole graph $G$. As $\DKB$ is clearly constructable from
$G$ in polynomial time, this proves $\Pi^p_2$-hardness.

(ii) As for data complexity, we note that the check where 
$Q$ has no match in any $\I_{\CAS}''$ is feasible in polynomial time,
as the number of variables in the query is fixed and thus only
constantly many Skolem terms $ST$ have to be added to
$N_\K$ for a query match in 
the least CAS-model $\hat{\I}_\K(\casmap,\nu)$ of
$\K$,
for which only polynomially many possibilities
exist; furthermore, the inference of atoms $A(t)$ resp.\ $R(t,t')$ is
feasible in polynomial time. Hence, the
problem is in \conp. The \conp-hardness follows from Theorem~\ref{theo:DKB-entail-conp}.
\end{proof}

\subsection{Complexity of Reasoning Problems with Unnamed Individuals}
\label{sec:unnamed-appendix}

    
\begin{manualproposition}{\ref{prop:n-bounded-mc}}
Given an $n$-de safe DKB $\DKB$, where $n$ is polynomial in the size
of $\DKB$,  and a
clashing assumption $\casmap$ defined on $N_\K$,
deciding whether $\K$ has (i) some arbitrary CAS-model 
resp.\ (ii) some justified CAS-model of form $\I_{\CAS} = \stru{\I,\casmap}$ is 
\np-complete resp.\ $D^p$-complete in general but feasible in polynomial time if 
$n$ is bounded by a constant.
\end{manualproposition}

\begin{proof} We can compute 
the (polynomially many) Skolem terms $t_i$, $i=1,\ldots,m \leq n$ that
feed into clashing assumptions for $\DKB$ using the algorithm in the 
proof of Proposition~\ref{prop:complexity-recognize-safe} for deciding
$n$-de safety in polynomial time. 

As for (i), in order to show that some  
CAS-model $\I_{\CAS}= \stru{\I,\casmap}$ of $\K$ exists, we
need to show that no inconsistency can be derived from $\K$ under
$\casmap$ relative to some name assigmment $\nu$ (which can be
pushed to $\K$).

We guess for each $i=1,\ldots,m$ whether $t_i = a_j$
holds for one or none of the individuals $a_1,\ldots,a_n$ that name exceptions in
$\casmap$.  Relative to this guess, we then decide in
polynomial time whether $\K$ is satisfiable.

To this end, we slightly modify a common algorithm that decides
unsatisfiability in absence of defeasible axioms by
nondeterministically deriving opposite assertions $D(t),\neg D(t)$
resp.\ $R(t,t')$, $\neg(t,t')$ from at most two assertions  in $\K$
in polynomially many steps using logarithmic workspace; these
assertions are either over $N_\K$, the same Skolem term $t$ or
$t'=f_R(t)$ for some role $R$, where outside $N_\K$ we need not store $t$. In the
extension, we keep track of which terms $\bar{t},\bar{t}'$ in the
assertion $(\neg)D(\bar{t}_h)$ resp. $(\neg)R(\bar{t}_h,\bar{t}'_h)$
derived in step $h$ are 
among the subterms of some $t_1,\ldots,t_m$; and if $\bar{t} = t_i$
resp.\ $\bar{t}' = t_i$ and $t_i=a_j$ is in the guess, then 
$D(a_j)$, $R(a_j,\bar{t}')$ resp.\ $R(\bar{t},a_j)$ can be derived.

As there are only polynomially many subterms of $t_1,\ldots,t_m$
(otherwise $\K$ would be recursive, thus not $n$-de bounded), the
bookkeeping for respecting the subterms is feasible in logarithmic
work space (each subterm $t$ may have an identifier $id(t)$, and a table 
computed before hand holds $(id(t),R,id(f_R(t)))$). 

As the algorithm works in polynomial time, we can decide in this way
also whether $\K$ is satisfiable under $\casmap$ relative to a name
assignment $\nu$.

As for (ii), we must in addition to (i) check that for each clashing
assumption $\stru{\alpha,\ee}$ in $\casmap$, some clashing set $S_{\stru{\alpha,\ee}}$ can be derived. We 
utilize here a similar guess and check algorithm as in
(i) to decide whether a given assertion $\alpha$ is not derivable from $\K$
under $\casmap$; i.e., we guess $t_i=a_j$ for all terms $t_i$ and then check in
nondeterministic logspace that $\alpha$ can not be derived.
Hence, deciding that for some $\stru{\alpha,\ee}$ no clashing
set $S_{\stru{\alpha,\ee}}$ is derivable is in \np, which implies that the
additional check is in \conp. Consequently, in case (ii) 
membership in $D^p$ follows. 

If  $n$ is bounded by a constant, then in
the algorithm above the guess for the equalities $t_i = a_j$  can be
eliminated, by cycling through all (polynomially) many
possibilities, which results in \ptime\ membership.

For the hardness proofs, for (i) we reduce deciding 3-colorability of
a graph $G=(V,E)$ to deciding CAS-model existence; we
provide for this a construction that can be reused for (ii).

We construct $\DKB$ as follows. For each edge $e_i =
(v_{i,1},v_{i,2})$ in $E = \{ e_1,\ldots, e_m\}$, we introduce two
individuals $v_{i,1}$ and $v_{i,2}$, and for each $v_{i,j}$ we
introduce three further individuals $col_{i,j}$, $c1_{i,j}$ and
$c2_{i,j}$. Informally, the latter three individuals will serve to
take the three colors red, green, and blue by roles $\mi{R_r}$,
$\mi{R_g}$, and $\mi{R_b}$ such that the color assigned to 
$col_{i,j}$ will be the color assigned to the occurrence of the node
$v_{i,j}$ in the edge $e_i$.

The alignment of
colors assigned to $v_{i,j}$ and $v_{i',j'}$ that represent the same
node $v_k$ in $V = \{v_1,\ldots, v_n\}$ will be ensured with the help of
an auxiliary node $check_{v_k}$. To this end, the nodes $col_{i,j}$ and
$col_{i',j'}$ will send their assigned colors to this node using roles 
$\mi{RCheck}$, $\mi{GCheck}$ and $\mi{BCheck}$, which tests
for their equality. That $v_{i,1}$ and  $v_{i,2}$ are colored
differently will be checked with the help of auxiliary roles 
$\mi{RNeighbor}$, $\mi{GNeighbor}$,
$\mi{BNeighbor}$. 

Finally, we use an individual $esc$ that allows us to model a state in
which no $v_{i,j}$ has a color assigned. This state however, requires
an exception to an axiom.

In the construction, we use a domain predicate, expressed by a
concept $\mi{Dom}$ that will be asserted for all individuals in $\K$
and enforced to be false for all other elements; roles between
individuals can only be in the domain, thus do not involve unnamed individuals.
Furthermore, we shall restrict roles between individuals by negative
role assertions.

The DKB $\K$ consists of the following assertions, axioms and
defeasible axioms:

\begin{enumerate}
    \item\label{e:1}  $V(v_{i,j})$ for all $v_{i,j}$  and $\mi{Dom}(a)$ for each $a=
        v_{i,j}, col_{i,j}$, $c1_{i,j}$, $c2_{i,j}, e$,
    $\mi{check}_{v_k}$, $esc$, $esc'$;

    \item\label{e:2} $V \isa \exists R$,  $V \isa \exists G$,  $V \isa \exists B$;

    \item\label{e:3}  $\exists R^- \isa \non \exists G^-$, $\exists G^- \isa \non
        \exists B^-$, $\exists R^- \isa \non \exists B^-$;

   \item\label{e:4} $\exists R^- \isa \exists \mi{RNeighbor}$,  $\exists
        G^- \isa \exists \mi{GNeighbor}$, 
         $\exists B^- \isa \exists \mi{BNeighbor}$;
   \item\label{e:5}
    $\exists R^- \isa \non\exists \mi{RNeighbor}^-$,  $\exists G^- \isa \non\exists \mi{GNeighbor}^-$, 
         $\exists B^- \isa \non\exists \mi{BNeighbor}^-$; 
  \item\label{e:6} $\exists R^- \isa \exists \mi{RCheck}$,  $\exists G^- \isa \exists \mi{GCheck}$,  $\exists B^- \isa \exists \mi{BCheck}$;
 \item\label{e:7} $\exists\mi{RCheck}^- \isa \non\exists\mi{GCheck}^-$, $\exists\mi{GCheck}^- \isa
        \non\exists\mi{BCheck}^-$,  $\exists\mi{RCheck}^- \isa \non\exists\mi{BCheck}^-$
   \item\label{e:8} $\exists R^- \isa \non E$;         
   
  \item\label{e:9} for each role $X$, we add the axiom $\exists X^- \isa
        \mi{Dom}$ and limit its range by adding $\neg X(a,b)$ for
        each pair $(a,b)$ of individuals from above that is not allowed as follows:
        
        \begin{itemize}
         \item for $R,G,B$  we allow $(v_{i,j},col_{i,j})$,
                $(v_{i,j},c1_{i,j})$, $(v_{i,j},c2_{i,j})$,  and for
                $R$ in addition $(v_{i,j},esc)$; 
         \item for $\mi{RCheck}$, $\mi{GCheck}$, $\mi{BCheck}$, we
                allow $(col_{i,j},check_k)$,  $(c1_{i,j},c1_{i,j})$,  $(c2_{i,j},c2_{i,j})$,  
                were $v_{i,j} = v_k$, and $(esc,esc)$;
         \item for  $\mi{RNeighbor}$, $\mi{GNeighbor}$, $\mi{BNeighbor}$, we
                allow $(col_{i,j},col_{i',j'})$, $(col_{i',j'},col_{i,j})$, 
                where  $v_{i,j} = v_k$, $v_{i',j'} = v_k'$, and
                $(v_k,v_k') \in E$; furthermore $(c1_{i,j},c2_{i,j})$
    and  $(c2_{i,j},c1_{i,j})$, for all $v_{i,j}$, and
    $(esc,esc')$; 
 \end{itemize}        
 \item  $\default(Dom \isa \bot)$ and $\default(E(esc))$.
 \end{enumerate}

We define the clashing assumption $\casmap$ to have an
exception of $\default(Dom \isa \bot)$ for all individuals $a$ where have
asserted $Dom(a)$ above.

We note that $\K$ has for $\casmap' = \casmap \cup \{
    \stru{E(esc),()} \}$ (i.e., when making also an exception to $\default(E(esc))$)    
a CAS-model $\I'_{\CAS} =
\stru{\I',\casmap'}$: if  in $\Ical'$  the atomic concept instances are those mentioned in (1),
and the role instances atoms are: $R(v_{i,j},esc)$,
$\mi{RCheck}(esc,esc')$, $\mi{RNeighbor}(esc,esc')$; 
$G(v_{i,j},c1_{i,j})$, $\mi{GCheck}(c1_{i,j},c2_{i,j})$; and
$B(v_{i,j},c2_{i,j})$, $\mi{VCheck}(c2_{i,j},c1_{i,j})$
for all  $v_{i,j}$; 
then by defining
the Skolem functions appropriately 
we can obtain a CAS-model which is named by $N_\K$.

On the other hand, it turns out that $\K$ has some CAS-model 
$\I_{\CAS} = \stru{\I,\casmap}$ (i.e., when making no exception to
$\default(E(esc))$) under UNA iff $G$ is 3-colorable.

To see this, if $G$ is 3-colorable, then we can reassign the roles 
$R(v_{i,j},esc)$, $G(v_{i,j},c1_{i,j})$, and  $B(v_{i,j},c2_{i,j})$ in
$\I'$ to $C(v_{i,j},col_{i,j})$,  $C_1(v_{i,j},c1_{i,j})$, and  $C_2(v_{i,j},c2_{i,j})$ to  
where $C$ is the color of the node $v_k$ such that $v_{i,j} = v_k$
in the 3-coloring, and $C_1$ and $C_2$ are the other two colors. This
then requires to set up, in a determined way, 
for $col_{i,j}$ the role instances
$\mi{CCheck}(col_{i,j},check_{i,j})$ and $\mi{CNeighbor}(col_{i,j},col_{i',j'})$,
where $v_{i,j}=v_k$ and $v_{i',j'}=v_{k'}$ and $(v_k,v'_k) \in E$,
for $c1_{i,j}$ the role instances $\mi{C_1Check}(c1_{i,j},c1_{i,j})$, 
$\mi{C_1Neighbor}(c1_{i,j},c2_{i,j})$, and for 
for $c2_{i,j}$ analogously the role instances $\mi{C_2Check}(c2_{i,j},c2_{i,j})$, 
$\mi{C_2Neighbor}(c2_{i,j},c1_{i,j})$. Finally, we make $E(esc)$ true
(further roles $\mi{RCheck}(esc,esc')$, $\mi{RNeighbor}(esc,esc')$
could be removed). The resulting CAS-interpretation $\I_{\CAS}$ is then a model of $\K$.

Conversely, if $\K$ has some CAS-model $\I_{\CAS} = \stru{\I,\casmap}$,
then $E(esc)$ is satisfied in it. 
Hence no role $R(v_{i,j},esc)$ is satisfied in $\IC_{\CAS}$, which means
that for each $v_{i,j}$, for some color $C$, $C\in\{R, G, B\}$
the role $C(v_{i,j},col_{i.j})$ holds in $\I$. As then the role 
$CCheck(col_{i,j},check_k)$ also holds, by the axioms in (\ref{e:7}) all $v_{i',j'}$
that correspond to the node $v_k$ will have the same color $C$.
Furthermore, the axiom $\exists C^- \isa \non\exists \mi{CNeighbor}^-$ is satisfied,
which implies that $v_{1,2}$ has a different color $C'\neq C$. Thus, 
we obtain from $\I_{\CAS}$ a 3-coloring of the graph $G$. This proves
the \np-hardness of case (i). 

We show the $D^p$-hardness of case (ii) by a reduction from 3COL-3UNCOL, i.e.,
given graphs $G_1$ and $G_2$, decide whether $G_1$ is 3-colorable and 
$G_2$ is not 3-colorable. We observe that the DKB $\K$
defined for the graph $G$ above has some justified CAS-model of the form 
$\I_{\CAS} = \stru{\I,\casmap \cup \{ \stru{E(esc),()} \}}$ iff $G$
is not 3-colorable.  

We thus take for $G_1$ and $G_2$ two copies $\K_1$ and
$\K_2$, respectively of the construction as above (using disjoint vocabularies), and set $\K = \K_1\cup\K_2$
and $\casmap = \casmap_1 \setminus \{ \stru{E_1(esc_1),()} \} \cup
\casmap_2$. Then, $\K$ some justified CAS-model of form $\I_{\CAS} =
\stru{\I,\casmap}$ iff $G_1$ is 3-colorable and $G_2$ is not 3-colorable.

As easily seen, $\K$ is acyclic and 
its TBox is the same for each graph $G$. Hence, $\K$ is $k$-chain bounded for some
constant $k$ 
and thus also $n$-de bounded
for some $n$ polynomial in $|\K'|$. This proves the
result under the stated restrictions, which moreover also holds under
data complexity.
\end{proof}
%
We remark that from the proof of
  Proposition~\ref{prop:n-bounded-mc}, we obtain that DKB-model
  checking, i.e., decide whether an interpretation $\I$ is a DKB-model
  of an DKB $\DKB$ is \conp-hard, as the CAS-model $\I'_{\CAS}$ for
  the DKB $\K$ constructed for the graph $G$ is justified iff $G$ is
  not 3-colorable. On the other hand, for $n$-de safe $\K$ where $n$
  is bounded by a polynomial in the size of $\K$, the problem is in
  \conp\ since it reduces to checking whether the clashing assumption
  $\casmap$ that contains all instances of axioms of $\DKB$ over
  $N_\K$ that are violated by $\I$, is justified.  That is, for such
  DKBs, the model checking problem is \conp-complete.  
   
\begin{manualtheorem}{\ref{theo:comp-n-bounded}}
Given an $n$-de safe DKB $\DKB$, where $n$ is bounded by a polynomial
in $|\DKB|$, (i) deciding $\DKB\models\alpha$ for an axiom
$\alpha$ and (ii) BCQ answering $\DKB\models Q$ are both
$\Pi^p_2$-complete. In case $n$ is bounded by a constant,  
(i) is $\conp$-complete while (ii) remains $\Pi^p_2$-hard.
\end{manualtheorem}

\begin{proof} Regarding the $\Pi^p_2$-membership results, 
to show in (i) that $\K\not\models \alpha$, we can
similarly proceed as in Theorem~\ref{theo:DKB-entail-conp} and 
guess a clashing assumption $\casmap$ for $\K$ on $N_\K$ and
a name assignment $\nu$
then check by Proposition~\ref{prop:n-bounded-mc} with an \np-oracle 
that some justified CAS-model $\I_{\CAS} = \stru{\I,\casmap}$ exists
relative to $\nu$.
If so, we check whether $\K\not\models \alpha$ relative to $\casmap$
and $\nu$ using an \np\ oracle, where we proceed depending on the type of
$\alpha$ as follows:

\begin{itemize}
    \item  If $\alpha$ is a positive or negative assertion, then we
     use the guess and check algorithm described in the proof of case
      (ii) of Proposition~\ref{prop:n-bounded-mc}.  
    \item In the other cases, we proceed similarly 
    as in the proof of Theorem~\ref{theo:DKB-entail-conp}: we introduce an auxiliary
    concept resp.\ role $Aux$, fresh individual names $a_e$ resp.\ $a_{e'}$ and check that,
   relative to $\casmap$
   and a guess for the equalities $t_i=a_j$ of the Skolem terms $t_1,\ldots,t_m$ that
   feed into clashing assumptions for $\DKB$ (which can be computed in
   polynomial time, cf.\ proof of
   Proposition~\ref{prop:complexity-recognize-safe}) to the
   individuals $a_1,\ldots,a_n$ that name exceptions in $\casmap$, 
   we can not derive $\neg Aux(t)$ resp.\ $\neg Aux(t,t')$ for some
   terms $t,t'$ that range over the Skolem terms of polynomially bounded
   depth and $a_e$ resp.\ $a_{e'}$. The algorithm in the
   proof of item (i) of Proposition~\ref{prop:n-bounded-mc} can be
   adjusted to this end, so that it runs in polynomial time.
\end{itemize}

Hence, deciding $\K\not\models \alpha$ is in $\Sigma^p_2$,
and thus deciding $\K\models \alpha$ is in $\Pi^p_2$.

In (ii), to show $\K\not\models Q$ we likewise guess a clashing
assumption $\casmap$ on $N_\K$ and a name assignment $\nu$, and
we  check using an \np\ oracle that some
justified CAS-model of form $\I_{\CAS} = \stru{\I,\casmap}$ relative
to $\nu$ exists.

We then compute the
(polynomially many) Skolem terms $t_i$, $i=1,\ldots,m$ that feed into
clashing assumptions, which is feasible in polynomial time
(cf.\ proof of Proposition~\ref{prop:complexity-recognize-safe}). 
We then guess for each $t_i$ whether $t_i = a_j$ holds for a single
$a_j$ or none, where $a_1,\ldots,a_n$ are all the
individuals that name exceptions in $\casmap$. We then can use 
an $\np$ oracle to guess polynomially many Skolem terms $ST$ that are
connected to $N_\K$, including the subterms of all $t_i$, where the
number depends on $Q$ and $\DKB$, and a match of the query $Q$ on
$N_\K \cup ST$, for which we test
the derivability of each atom $A(t)$ resp.\ $R(t,t')$ in the match.
By the least model property, this will ensure that this partial model
given by the match can
indeed be extended to a model
            
It follows that deciding
$\K\not\models Q$ is in $\Sigma^p_2$, which means deciding $\K\models Q$ is in $\Pi^p_2$.

As for (i) in case $n$ is bounded by a constant $k$, we can 
by Proposition~\ref{prop:n-bounded-mc} eliminate the
\np\ oracle and obtain membership in \conp.

To show the $\Pi^p_2$-hardness for (i), we extend the encoding of graph non-3-colorability
in the proof of Proposition~\ref{prop:n-bounded-mc} in order to
encode the constrained 3-colorability problem of Lemma~\ref{lem:sink-3col}.

We first note that the DKB $\K$ constructed for the
graph $G=(V,E)$ allows under UNA for a justified CAS-model of form $\I_{\CAS} =
\stru{\I,\casmap}$ such that $\stru{E(esc),()} \in \casmap'$ iff $G$
is not 3-colorable. To see this, recall that the constructed $\casmap$ is justified iff $G$ is not 3-colorable.
By construction of $\K$, each justified CAS-model $\I'_{\CAS} =
\stru{\I',\casmap'}$ of $\K$ such that $\stru{E(esc),()} \in \casmap'$ must
satisfy $\casmap \subseteq \casmap'$; thus by non-redundancy of
clashing assumptions (Proposition~\ref{prop:cas-minimality}), it
follows that $\casmap = \casmap'$. Furthermore, we have that $\K
\models \non E(esc)$ if $G$ is not 3-colorable and $\K \models E(esc)$ otherwise.

Suppose now that $v_{d_1},\ldots,v_{d_m}$ are the nodes in $G$ of
degree 1. We use additional concepts $S$, $F_R$, $F_G$, $F_B$ and add the following
assertions, axioms and defeasible axioms:
\begin{enumerate}
 \item $S(check_{d_j})$, for all $j=1,\ldots,m$
 \item $\default(S \isa F_R)$, $\default(S \isa F_G)$, $\default(S \isa F_B)$ 
 \item $F_R\isa \non F_G$, $F_G\isa \non F_B$, $F_R\isa \non F_B$
 \item $F_C \isa \non\exists C_1Check^-$,  $F_C \isa \non\exists
                                C_2Check^-$ where $C \in \{R,G,B\} $ and
                                $C_1$, $C_2$ are the remaining colors
            $\{R,G,B\}\setminus \{C\}$.
\end{enumerate}                
Intuitively, (1)-(3) allow us to select one of the colors for each node $v_{d_j}$
of degree 1. This selection must be in alignment with possible color
checks via $\mi{CCheck}$ roles issued by nodes $check_{i',j'}$  that
correspond to $v_{d_j}$, i.e., if color $C$ is selected, then only
incoming $\mi{CCheck}$ arcs are possible for $check_{d_j}$. Furthermore, by
non-redundancy of clashing assumptions some color for $v_{d_j}$ must
be selected, as $\non F_C(check_{i',j'})$ can not be proven for all 
three colors $C$ simultaneously from the axioms (3) and (4) because
$\mi{CCheck}^-$ can be true at $check_{i,j}$ for at most one color $C$. 
Thus, each clashing assumption of a justified CAS-model of the 
the resulting DKB $\K'$ must encode a color assignment $\rho$ to the nodes
$v_{d_1}$, \ldots, $v_{d_m}$.

For an arbitrary such $\rho$, we obtain a CAS-model $\I_{\CAS}^\rho =
\stru{\I^\rho,\casmap^\rho}$ of
$\K'$ from the candidate justified CAS-model 
$\I_{\CAS}=\stru{\I,\casmap}$ described in the proof of 
Proposition~\ref{prop:n-bounded-mc}
by making, for each $v_{d_j}$,
$check_{d_j}$ an instance of $D$ and of $F_C$ where $C=\rho(v_{d_j})$
is the color of $v_{d_j}$ and adding $\stru{D\isa F_{C_1},check_{d_j}}$,
$\stru{D\isa F_{C_2},check_{d_j}}$  to $\casmap$ for the remaining
colors $C_1$ and $C_2$. It then holds that $\I_{\CAS}^\rho$ is a
justified CAS-model of $\K'$ iff the coloring $\rho$ is not extendable
to a 3-coloring of the full graph $G$.  

By construction of $\K'$ and the non-redundancy of clashing assumptions, it
follows that every justified CAS-model $\I'_{\CAS} =
\stru{\I',\casmap'}$ of $\K'$ such that $\stru{E(esc),()} \in \casmap'$ must
under UNA be of the form $\casmap' = \casmap^\rho$  for some coloring $\rho$.
It follows that $\K' \not\models E(esc)$ iff some coloring
$\rho$ is not extendable to a 3-coloring of the full graph $G$; hence
deciding $\K \models E(esc)$ is $\Pi^p_2$-hard.

Like $\K$ in the proof of Proposition~\ref{prop:n-bounded-mc}, also DKB $\K'$ is acyclic and 
its TBox is the same for each graph $G$, and thus along the same lines
the result holds under the stated restrictions and under
data complexity.

The hardness results for the other cases follow from the results on
exception-safe DKBs.
\end{proof}

\label{lastpage}
\end{document}